\PassOptionsToPackage{pagebackref=true,hyperindex}{hyperref}

\documentclass[mathfont=newtx,accepted]{uai2026} 

\usepackage[british]{babel}

\usepackage[round]{natbib} 
\usepackage{mathtools} 
\usepackage{booktabs} 
\usepackage{tikz} 

\usepackage{microtype}          
\usepackage{pifont}             
\usepackage[normalem]{ulem}     
\usepackage{xspace}             
\usepackage{enumitem}           

\usepackage{xcolor}

\definecolor{teal}{rgb}{0.0, 0.5, 0.5}
\definecolor{amethyst}{rgb}{0.6, 0.4, 0.8}
\definecolor{thulianpink}{rgb}{0.87, 0.44, 0.63}
\definecolor{tiffanyblue}{rgb}{0.04, 0.73, 0.71}
\definecolor{pear}{rgb}{0.82, 0.89, 0.19}
\definecolor{applegreen}{rgb}{0.55, 0.71, 0.0}
\definecolor{burgundy}{rgb}{0.5, 0.0, 0.13}
\definecolor{persianindigo}{rgb}{0.2, 0.07, 0.48}

\hypersetup{
  colorlinks=true,
  linkcolor=thulianpink,
  filecolor=pear,
  urlcolor=tiffanyblue,
  citecolor=amethyst,
}

\renewcommand*{\backref}[1]{}
\renewcommand*{\backrefalt}[4]{%
  \ifcase #1%
  \or Cited on page~#2.%
  \else Cited on pages~#2.%
  \fi%
}

\usepackage{amsmath,amssymb,amsfonts,amsthm}
\usepackage{mathtools}
\usepackage{nicefrac}           
\usepackage{thmtools,thm-restate}

\newcommand{\Z}{\mathbf{Z}}
\newcommand{\z}{\mathbf{z}}
\newcommand{\x}{\mathbf{x}}
\newcommand{\hatz}{\hat{\mathbf{z}}}
\newcommand{\R}{\mathbb{R}}
\newcommand{\M}{\mathcal{M}}

\DeclareMathOperator{\Cov}{Cov}
\DeclareMathOperator{\Corr}{Corr}

\newcommand{\Var}{\sigma^2}           
\DeclareMathOperator{\sgn}{sgn}
\newcommand{\mcc}{\ensuremath{\mathrm{MCC}}}
\newcommand{\rsq}{\ensuremath{R^2}}
\newcommand{\dci}{\ensuremath{\mathrm{DCI}}}

\usepackage{amsmath,amsfonts,bm}

\def\1{\bm{1}}

\def\rvb{{\mathbf{b}}}

\def\rvx{{\mathbf{x}}}

\def\rvz{{\mathbf{z}}}

\def\rmA{{\mathbf{A}}}

\def\rmD{{\mathbf{D}}}

\def\rmI{{\mathbf{I}}}

\def\rmP{{\mathbf{P}}}

\def\rmU{{\mathbf{U}}}
\def\rmV{{\mathbf{V}}}

\DeclareMathAlphabet{\mathsfit}{\encodingdefault}{\sfdefault}{m}{sl}
\SetMathAlphabet{\mathsfit}{bold}{\encodingdefault}{\sfdefault}{bx}{n}

\def\sR{{\mathbb{R}}}

\newcommand{\cN}{\mathcal{N}} 

\usepackage{graphicx}
\usepackage{svg}
\usepackage{wrapfig}
\usepackage{float}
\usepackage{caption}
\usepackage{subcaption}

\usepackage{booktabs}
\usepackage{multirow}
\usepackage{tabularx}
\usepackage{colortbl}
\usepackage{makecell}
\usepackage{tablefootnote}
\usepackage{arydshln}
\usepackage{longtable}
\usepackage{ragged2e}

\newcolumntype{H}{>{\centering\arraybackslash}p{0.34cm}}
\newcolumntype{B}{|>{\centering\arraybackslash}p{0.34cm}}
\newcolumntype{G}[1]{>{\columncolor{applegreen!5}\raggedright\arraybackslash\textsc}p{#1}}

\usepackage{algorithm}
\usepackage{algorithmic}

\let\classAND\AND
\let\AND\relax

\let\AND\classAND
\AtBeginEnvironment{algorithmic}{\let\AND\algoAND}

\usepackage{natbib}
\usepackage[capitalize]{cleveref}
\crefname{section}{\S}{\S}
\crefname{subsection}{\S}{\S}
\crefname{subsubsection}{\S}{\S}
\crefname{figure}{Fig.}{Figs.}
\crefname{prop}{Prop.}{Props.}
\crefname{proposition}{Prop.}{Props.}
\crefname{appendix}{Appx.}{Appxs.}
\crefname{algorithm}{Alg.}{Algs.}
\crefname{theorem}{Thm.}{Thms.}
\crefname{conjecture}{Conj.}{Conjs.}
\crefname{researchquestion}{Q.}{Qs.}
\crefname{definition}{Defn.}{Defns.}
\crefname{corollary}{Cor.}{Cors.}
\crefname{lem}{Lem.}{Lems.}
\crefname{table}{Tab.}{Tabs.}
\crefname{assum}{Assum.}{Assums.}
\crefname{example}{Ex.}{Exs.}
\crefname{property}{Property}{Properties} 

\usepackage{fontawesome5}

\usepackage{tikz}
\usetikzlibrary{arrows.meta,positioning,fit,calc}

\usepackage[most]{tcolorbox}
\tcbuselibrary{skins,breakable}
\newtcolorbox{conceptbox}[1]{%
  enhanced jigsaw,
  breakable,
  colback=applegreen!5,
  colframe=applegreen!50!black,
  coltitle=black,
  colbacktitle=applegreen!14,
  title=\textsc{#1},
  boxrule=0.5pt,
  arc=2pt,
  left=4pt, right=4pt, top=2pt, bottom=2pt,
  titlerule=0pt
}

\usepackage{epigraph}

\usepackage{todonotes}
\usepackage{marginnote}

\newcommand{\cmark}{\ding{51}}
\newcommand{\xmark}{\ding{55}}

\newcommand{\authorcomment}[3]{%
  \noindent
  {\color{#1}%
    \textbf{\fbox{#2}}\;
    \textit{#3}%
  }%
}
\newcommand{\sj}[1]{%
  \authorcomment{amethyst}{Shruti}{#1}%
}

\newcommand{\patrik}[1]{%
  \authorcomment{cyan}{Patrik}{#1}%
}

\newcommand{\pb}[1]{%
  \authorcomment{orange}{Phil}{#1}%
}

\colorlet{crl}{persianindigo}
\colorlet{beyond}{thulianpink}

\colorlet{oo}{teal}
\colorlet{mo}{amethyst}
\colorlet{om}{burgundy}

\makeatletter
\newcommand{\dgp@crlsym}[1]{\ifcase#1\or\perp\or\rho\fi}
\newcommand{\dgp@beyondsym}[1]{\ifcase#1\or\or\or f\or F\fi}
\newcommand{\dcrl}[1]{\textcolor{crl}{$\mathbf{D}_{\dgp@crlsym{#1}}$}}
\newcommand{\dbeyond}[1]{\textcolor{beyond}{$\mathbf{D}_{\dgp@beyondsym{#1}}$}}
\makeatother

\newcommand{\ecrl}[2]{\textcolor{crl}{\textbf{E}}\textcolor{#2}{\textbf{#1}}}
\newcommand{\ebeyond}[2]{\textcolor{beyond}{\textbf{E}}\textcolor{#2}{\textbf{#1}}}

\newcommand{\rand}[1]{\textcolor{gray}{\textbf{E#1}}}

\newcommand{\parless}[1]{\noindent\textbf{#1}}

\newcommand{\cG}{\mathcal{G}}

\newcommand{\circuit}[1]{\\[1pt]%
  {\small\textcolor{gray!70!black}{%
  \hspace*{1em}$\triangleright$\enspace #1}}}

\newtheorem{property}{Property}
\newtheorem{definition}{Definition}
\newtheorem{proposition}{Proposition}

\title{Who Guards the Guardians? \\ The Challenges of Evaluating Identifiability of Learned Representations}

\author[1]{Shruti Joshi}
\author[1]{Th\'eo Saulus}
\author[2]{Wieland Brendel}
\author[1]{Philippe Brouillard}
\author[1]{Dhanya Sridhar}
\author[2]{Patrik Reizinger}
\affil[1]{%
    Mila - Qu\'ebec AI Institute \& Universit\'e de Montr\'eal
}
\affil[2]{%
    Max-Planck-Institute for Intelligent Systems,  ELLIS Institute T\"ubingen, University of T\"ubingen
}
  
\begin{document}
\addtocontents{toc}{\protect\setcounter{tocdepth}{-1}}
\maketitle

\begin{abstract}

Identifiability in representation learning is commonly evaluated using standard metrics (e.g., \mcc, \rsq, \dci) on synthetic benchmarks with known ground-truth factors. These metrics are assumed to reflect recovery up to the equivalence class guaranteed by identifiability theory. We show that this assumption holds only under specific structural conditions: each metric implicitly encodes assumptions about both the data-generating process (DGP) and the encoder. When these assumptions are violated, metrics become misspecified and can produce systematic false positives and false negatives. Such failures occur both within classical identifiability regimes and in post-hoc settings where identifiability is most needed. We introduce a taxonomy separating DGP assumptions from encoder geometry, use it to characterize the validity domains of existing metrics, and release an evaluation suite for reproducible stress testing and comparison.

\end{abstract}


\begin{figure*}
    \includegraphics[width=\linewidth]{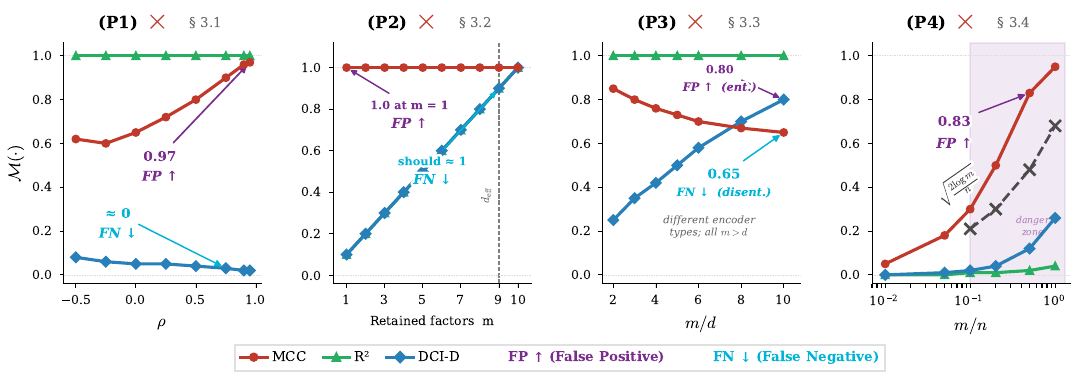}
    \caption{\textbf{Every identifiability metric fails under at least one common evaluation setting.}
We test four desiderata (\cref{prop:rho-invariance,prop:deff-sensitivity,prop:oc-invariance,prop:false-positive}) using controlled synthetic encoders that isolate metric behaviour from optimisation artefacts.
\textbf{(P1)}~Latent correlation: \mcc{} conflates correlation with identifiability (FP\,$\uparrow$ to $0.97$); \dci{}-D penalises it (FN\,$\downarrow$).
\textbf{(P2)}~Factor dropping: \dci{}-D reports perfect disentanglement even when $9$ of $10$ factors are lost.
\textbf{(P3)}~Overcompleteness: \mcc{} inflates for entangled encoders; \dci{}-D deflates for disentangled ones.
\textbf{(P4)}~Null encoder: all metrics inflate as $m/n$ grows, with \mcc{} scaling in the order of $\sqrt{2\log m / n}$.
Only \rsq{} is robust across (P1), (P3), and (P4), but shares the (P2) limitation.
No single metric is trustworthy across all settings.}
    \label{fig:four-panel-summary}
\end{figure*}

\section{Introduction}
Learning representations that are interpretable, modular, and controllable is a long-standing goal across machine learning.
Identifiability formalises this objective: a representation achieves these properties when it recovers the ground-truth generative factors uniquely, up to a specified equivalence class \citep{comon1994independent, hyvarinen1999nonlinear}.
Strong identifiability guarantees now exist for nonlinear representation learners under auxiliary information \citep{hyvarinen2019nonlinear, khemakhem2020variational}, temporal structure \citep{hyvarinen2016unsupervised}, mechanism
sparsity \citep{lachapelle2022disentanglement}, or for restricted classes of models \citep{khemakhem2020icebeemidentifiableconditionalenergybased, marconato2024all}.
Causal representation learning (CRL)~\citep{scholkopf2021toward} builds on these foundations by additionally requiring that the identified factors admit a causal semantics---typically as variables in a structural causal model with predictable responses under interventions and distribution shifts \citep{arjovsky2020invariantriskminimization, peters2016causal}. 
These results have wide-reaching implications and are increasingly adopted in fields such as mechanistic interpretability \citep{elhage2022toy}, where identifiability of learned features is now recognised as a prerequisite for reliable interpretation \citep{song2025position, joshi2025identifiablesteeringsparseautoencoding}, and in the analysis of pretrained representations more broadly \citep{roeder2021linear}.

In practice, these theoretical guarantees are validated empirically. Given ground-truth factors $\z \sim p(\z) \in \R^d$  and learned representation codes $\hat{\rvz} \in \mathbb{R}^m$, a metric $\mathcal{M}(\rvz, \hat{\rvz}) \to [0, 1]$ returns a scalar interpreted as the degree of identifiability.
The standard protocol is to compute $\mathcal{M}$ on a synthetic benchmark with known $\rvz$, and interpret a high score as evidence that the encoder has recovered the true factors up to a specified equivalence class, e.g., permutation and rescaling.

However, this puts all faith into the metrics---\emph{``Who guards the guardians?''}\footnote{``Quis custodiet ipsos custodes?''---Juvenal, \emph{Satires}~VI.} Each metric encodes structural assumptions about the latent factor distribution $p(\rvz)$, the relationship between ground-truth and learned representation dimensionalities ($d$ and $m$), the sample size $n$, and the equivalence class targeted. Yet these assumptions are typically left implicit: papers routinely report a single metric score---\mcc{} \citep{khemakhem2020icebeemidentifiableconditionalenergybased}, \rsq{}, or \dci{}-D \citep{eastwood2018framework}---as evidence of identifiability, without verifying if the evaluation setting is consistent with the metric's validity domain. Prior work has observed that metrics can disagree on method rankings and are sensitive to factors such as nonlinearity strength and hyperparameter choice \citep{sepliarskaia2019not, carbonneau2022measuring}, and that specific metrics produce false positives when latent factors are statistically related \citep{yao2025third}. However, these remain empirical observations tied to particular settings; no prior work characterises \emph{when} and \emph{why} failures arise, nor whether they reflect systematic misspecification predictable from each metric's design. A theorem may guarantee recovery despite correlated factors or only up to an affine transform, whereas, e.g., using \mcc{} targets axis-aligned recovery of independent factors---a strictly stronger assumption whose violation produces systematically wrong scores due to a structural mismatch between what the metric measures and what the experiment intends to measure.
This leads to the question:
\begin{quote}
\emph{Can the structural conditions under which a metric faithfully
measures identifiability be characterised, and can these conditions
be used to predict when false positives and false negatives will arise?}
\end{quote}
We show that the answer is yes: each metric's failure modes follow predictably from its encoded assumptions.

\parless{Structural misspecification.} When a metric's encoded assumptions do not match the latent factor structure ($p(\rvz)$) or the properties of the encoder producing $\hat{\rvz}$, we say the metric is \emph{misspecified} for that evaluation setting. Unlike finite-sample noise, misspecification is a population-level property that would persist even when the number of samples $n \to \infty$, producing \emph{false positives} (high scores despite lack of identifiability) or \emph{false negatives} (low scores despite identifiability up to the desired equivalence class). 
To predict when and how misspecification arises, we organise assumptions along two orthogonal axes: (i)~\emph{latent factor structure}---whether ground-truth factors are independent, correlated, or linked by functional constraints that reduce the effective dimensionality below $d$; and (ii)~\emph{encoder properties}---the equivalence class, the dimensionality ratio $m/d$, and if factor information is distributed across coordinates.

\parless{Main contributions} We introduce a two-axis taxonomy (\cref{sec:taxonomy-main}) separating assumptions about latent factor structure from encoder properties, with formal desiderata for identifiability metrics (\cref{prop:rho-invariance,prop:deff-sensitivity,prop:oc-invariance,prop:false-positive}). Through controlled synthetic experiments that isolate metric behaviour from optimisation artefacts, we show that no existing metric satisfies all desiderata and characterise precisely how each fails. We derive closed-form analyses showing that (i)~\mcc{} approaches $1$ when latent factors are highly correlated, even when the encoder remains entangled (\cref{subsec:correlation}), and (ii) the expected \mcc{} under an encoder producing random representations independent of the ground truth is governed by the representation-to-sample ratio ($m/n$) (\cref{subsec:false-positive}). \dci{}-D is similarly inflated for entangled encoders when $m > d$ (\cref{subsec:overcomplete}). We also find that a fundamental limitation of all metrics is that they cannot distinguish lossless compression from lossy omission of latent factors when there exist multi-factor dependencies among them (\cref{subsec:undercomplete}). Detailed discussion of related work appears in \cref{sec:related-work}.

\section{A taxonomy for metric (mis)specification}
\label{sec:taxonomy-main}
Identifiable representation learning posits a two-step data generating process.

\parless{Formal setup.} Ground truth factors $\rvz$ are sampled first, and an  observation $\rvx:=g(\rvz)$ is then generated via an unknown map $g:\sR^{d}\!\to\!\sR^{n}$ \citep{hyvarinen1999nonlinear}.
A learned encoder $f : \sR^n \rightarrow \sR^m$ produces $\hat{\rvz} := f(\rvx)$. It \emph{identifies} the generative factors up to a restricted equivalence class, typically axis-aligned transformations such as permutation and componentwise rescaling, under which the representation $\hat{\rvz}$ is identified (also called disentangled) \citep{schmidhuber1992learning, dicarlo2007untangling, bengio2013representation, higgins2018towards}. We adopt the standard notion of identifiability \citep{hyvarinen1999nonlinear}.
\begin{definition}[Identifiability up to $\cG$]
\label{def:identifiability} For $\cG$, a class of transformations acting on $\R^d$ and some $h \in \cG$, where $h : \R^d \to \R^d$, the encoder~$f$ \emph{identifies} the latent factors up to~$\cG$ if $f \circ g = h$.
\end{definition}


Three standard equivalence classes are: (i) \emph{Permutation and rescaling} ($\cG_{\mathrm{perm}}$): $h(\z) = \rmP \rmD \z$ where $\rmP$ is a permutation matrix and $\rmD$ a diagonal scaling matrix, (ii) \emph{Affine} ($\cG_{\mathrm{aff}}$):
  $h(\z) = \rmA\z + \rvb$ with $\rmA$ invertible, and (iii) \emph{Elementwise nonlinear} ($\cG_{\mathrm{nl}}$):
  $h(\z) = (h_1(z_{\pi(1)}), \ldots, h_d(z_{\pi(d)}))$ where each $h_j$ is a smooth invertible function.
All three assume $m = d$.
When $m \neq d$ \citep{eastwood2023dciesextendeddisentanglementframework, chen2025causalverse}, \cref{def:identifiability} does not apply directly, which we extend to partial and overcomplete recovery in \cref{sec:encoder-taxonomy}.

Each metric $\M$ implicitly targets one of these three equivalence classes, and using a metric outside its target class produces systematically wrong scores. Based on a systematic review of the causal representation learning and nonlinear ICA literature (\cref{apx:review}), we study the three most commonly used metrics: \mcc{} (in two variants: \mcc{}-P based on Pearson correlation,
and \mcc{}-S based on Spearman rank correlation), \rsq, and \dci{}-D (the disentanglement component). \mcc{} computes an optimal one-to-one matching of codes to factors via pairwise correlations, hence targeting elementwise identifiability: (i) in $\cG_{\mathrm{perm}}$ (both \mcc{}-P and \mcc{}-S), (ii) and in $\cG_{\mathrm{nl}}$ (\mcc{}-S). \rsq is often used by training a linear probe from $\hat{\rvz}$ to $\rvz$ and measures explained variance, hence used for evaluating linear identifiability under $\cG_{\mathrm{aff}}$---it cannot distinguish between $\cG_{\mathrm{aff}}$ and $\cG_{\mathrm{perm}}$ \dci{}-D trains a probe (linear or nonlinear, e.g., gradient boosted trees (GBT) \citep{natekin2013gradient}) to predict each ground-truth factor from the learned codes, then measures how concentrated the resulting feature importances are: a score of~$1$ means each code is important for predicting at most one factor. Unlike \mcc{}, \dci{}-D does not require
one-to-one code--factor alignment and can handle $m \neq d$, but it remains sensitive to how the probe distributes importance across coefficients---correlated or entangled codes spread importance across multiple factors, deflating the score even when all information is preserved (\cref{sec:results}).

To predict metric failures, we consider two orthogonal axes: the \emph{latent factor structure}---whether factors are independent, correlated, or linked by deterministic constraints---and the \emph{encoder geometry}---the equivalence class, dimension ratio $m/d$, and how factor information is distributed across codes.
We define each axis in turn. We use a simple physical system as a running example throughout this section to illustrate how the DGP types and encoder geometries defined below arise naturally in practice.


\begin{conceptbox}{Running Example: A Circuit with a resistor.}
\label{box:circuit}
\textbf{Setting.}\enspace Current flows through a resistor, causing it to dissipate heat and exchange energy with its environment. Sensors (ammeter, voltmeter, thermometer, thermal camera) record the circuit state, producing observations~$\x$.

\medskip
\textbf{Factors.}\enspace Four physical quantities influence measurements: $T$ (ambient temperature), $R$ (resistance),  $I$ (current), and $V$ (voltage).
All four are independently measurable, but not independently variable.
Two physical laws constrain them: $R = R_0\bigl(1 + \alpha\,(T - T_0)\bigr)$ and $V = IR$,
where $R_0$ is the resistance at reference temperature~$T_0$ and $\alpha$ is the material's temperature coefficient of resistance.
So, the factor set $(T, R, I, V)$ has four entries but two degrees of freedom.

\medskip
\textbf{Why this matters.}\enspace An unsupervised learner has no access to Ohm's law.
If it discovers a feature tracking~$V$, that feature is correct even though $V$ is, in principle, determined by $I$ and~$R$. 
Recovering all four factors reflects the DGP at a particular level of description; a different granularity, say, retaining only $(T, I)$, is equally valid.
Which level is appropriate depends on the downstream task and cannot be determined from the representation alone.
\end{conceptbox}

\subsection{Factor dependencies reduce effective dimensionality}
\label{sec:dgp-taxonomy}

Standard disentanglement benchmarks sample each latent factor independently (\dcrl{1}) \citep{dsprites17, 3dshapes18, gondal2019transfer}. However, identifiability theorems do not always assume this \citep{lachapelle2022disentanglement, hyvarinen2019nonlinear, morioka2023causal, khemakhem2020variational, khemakhem2020icebeemidentifiableconditionalenergybased, ahuja2022weakly} and permit statistical dependence (\dcrl{2}), e.g., through confounding or noisy causal mechanisms ($z_2 = f(z_1) + \varepsilon$, $\varepsilon \not\equiv 0$). In both cases, every factor retains a unique degree of freedom, so $d_{\mathrm{eff}} = d$. We argue that a third, orthogonal generalisation is equally important: factors may be linked by \emph{deterministic functional constraints} that reduce the effective dimensionality of the factor set below $d$. Such constraints arise naturally from definitional redundancies (e.g., encoding position on both linear and logarithmic scales) and physical laws (see the running example box). This is generic in unsupervised settings where the target factors are not known a priori.



\parless{Setup.} Two DGP types standard in the literature that define the regime where identifiability theory operates.

\begin{itemize}[leftmargin=*,noitemsep]
\item \textbf{\dcrl{1} --- Independent factors.}\enspace
  Factors vary independently; each contributes a unique degree of
  freedom.
  This is the implicit assumption behind most metrics.
 \circuit{} $T$ and $I$ are set by independent exogenous sources.
\item \textbf{\dcrl{2} --- Correlated factors.}\enspace
  Factors are statistically dependent but each retains a unique degree
  of freedom; no factor is a deterministic function of the others.
 \circuit{} A thermostat induces a correlation between $T$ and $I$.

\end{itemize}


\parless{Extended setup for unknown abstraction level.} The standard settings above assume that each factor contributes independent information. In practice, however, factors may be linked by deterministic relationships that reduce the effective dimensionality below~$d$. This can happen both when we know the ground truth latent factor set, and when they are not known a priori. In such a case, as in the circuit example, a learner with no knowledge of Ohm's law might reasonably include both resistance $R$ and temperature~$T$ as separate factors, unaware that $R = R_0(1 + \alpha\, T)$. In either case, some factors carry no independent information, and metrics that treat every factor as a free degree of freedom will be misspecified.

\begin{definition}[Effective dimensionality]
\label{def:effective-dim}
For latent factors $\z \in \R^d$ subject to $k$ independent smooth constraints $c_1(\z) = 0, \ldots, c_k(\z) = 0$, the \emph{effective dimensionality} is $d_{\mathrm{eff}} = d - k$, i.e., the number of factors that can vary freely.
Under \dcrl{1} and \dcrl{2}, $d_{\mathrm{eff}} = d$.
Under \dbeyond{3}/\dbeyond{4}, $d_{\mathrm{eff}} < d$.
\end{definition}
\begin{itemize}[leftmargin=*,noitemsep]

\item \textbf{\dbeyond{3} --- Single-factor constraint.}\enspace
  One factor is a deterministic function of exactly one other, $d_{\mathrm{eff}} = 1$.
  \circuit{} $R = R_0(1 + \alpha\, T)$.

\item \textbf{\dbeyond{4} --- Multi-factor constraint.}\enspace
  A deterministic relationship involves multiple factors, reducing $d_{\mathrm{eff}}$ further.
  \circuit{} $V = IR$: voltage is determined jointly by current and resistance, so $(T, R, I, V)$ has $d_{\mathrm{eff}} = 2$.

\end{itemize}
\dcrl{1} and \dcrl{2} address statistical relationships among independently varying
factors ($d_{\mathrm{eff}} = d$). \dbeyond{3} and \dbeyond{4} address settings where deterministic constraints reduce $d_{\mathrm{eff}} < d$, also addressing what happens when the abstraction level is unknown. Although single- and multi-factor dependencies might not seem qualitatively different, metrics behave differently (c.f. \cref{subsec:undercomplete}). Formal description in \cref{apx:dgp-taxonomy}.

\subsection{Encoder Structure and Dimension Mismatch}
\label{sec:encoder-taxonomy}
Identifiability theory typically assumes $m = d$: the encoder's output dimension matches the number of latent factors.
In practice---particularly when
disentangling representations from pretrained models where $d$ is unknown---the common regime is $m > d$, often $m \gg d$.
We organise encoders along two axes: the \emph{equivalence class} up to which factors are identified, and the \emph{dimension ratio} $m/d$.

\parless{Matched dimension.} The equivalence classes from \cref{def:identifiability} define three
encoder types at $m = d$. 

\begin{itemize}[leftmargin=*,noitemsep]

\item \textbf{\ecrl{1}{oo} --- Elementwise linear.}\enspace
  Recovery up to $\cG_{\mathrm{perm}}$
  (\cref{def:identifiability}~(i)).
  This is the strongest form of identifiability that can be guaranteed.
  Ideally, every metric should score~$1$ here; any that does not has an intrinsic calibration defect.
  \circuit{$(2T,\; -R,\; 0.5I,\; 3V)$.}

  \item \textbf{\ecrl{2}{oo} --- Elementwise nonlinear.}\enspace
  Recovery up to $\cG_{\mathrm{nl}}$
  (\cref{def:identifiability}~(iii)).
  Each code is a smooth
  invertible function of exactly one factor.
  The parameter~$\alpha$
  (\cref{tab:experiment-parameters}) controls the degree of
  nonlinearity; $\alpha = 0$ reduces to \ecrl{1}{oo}.
  \circuit{$(\tanh T,\; R^3,\; \sqrt[3]{I},\; \sinh V)$.}

\item \textbf{\ecrl{3}{om} --- Linearly entangled.}\enspace
  Recovery up to $\cG_{\mathrm{aff}}$
  (\cref{def:identifiability}~(ii)).
  All factor information is
  preserved, but distributed across coordinates via rotation or
  shearing.
  The degree of entanglement is controlled by the condition number $\kappa$ of~$\rmA$
  (\cref{tab:experiment-parameters}): $\kappa = 1$ reduces to \ecrl{1}{oo}.
  \circuit{$A\,(T,R,I,V)^\top$: every code mixes all factors.}
\end{itemize}

\parless{Dimension mismatch breaks coordinate-wise evaluation.}
In practice, $m$ may differ from~$d$, motivating a more general notion of identifiability.

\begin{definition}[Identifiability under dimension mismatch]
\label{def:identifiability-general}
Let $S \subseteq \{1,\ldots,d\}$ and let $\cG$ follow from
\cref{def:identifiability}.  The encoder~$f$ identifies
the factors~$S$ up to~$\cG$ if there exist
$T \subseteq \{1,\ldots,m\}$ with $|T| = |S|$ and $h \in \cG$ such that,
\vspace{-.5em}
\[
  \pi_T \circ f \circ g \;=\; \pi_S \circ h,
\vspace{-.2em}
\]
where $\pi_T \colon \R^m \to \R^{|T|}$ and
$\pi_S \colon \R^d \to \R^{|S|}$ select the coordinates
indexed by~$T$ and~$S$ resp. Since $h$ may permute coordinates, $S$ identifies which factors are recovered but not which codes
carry them.  When $m = d$ and $T = S = \{1,\ldots,d\}$, both projections become identity, and this reduces to \cref{def:identifiability}.
\end{definition}

\noindent
In words: among the $m$ learned codes in~$\hat{\rvz}$, there exist $|S|$ of them (indexed by~$T$) that together recover the factors in~$S$ up to the allowed transformation class~$\cG$; the remaining $m - |S|$ codes are ignored. The choice of~$h$ in \cref{def:identifiability-general}
inherits from \cref{def:identifiability}: for identifiability up to $\cG_{\mathrm{perm}}$, the composition $\pi_T \circ f \circ g$ recovers each factor up to permutation and rescaling; for $\cG_{\mathrm{nl}}$, up to a smooth monotonic nonlinear function; and under $\cG_{\mathrm{aff}}$, $\pi_T \circ f \circ g$ may return an invertible linear
mix of the factors in~$S$ rather than individual factors.
Next, we can define dimensionality-mismatched encoders.
\begin{itemize}[leftmargin=*,noitemsep, topsep=2pt]
    \item \emph{\ebeyond{4}{oo}} ---\textbf{Undercomplete}.
    The encoder outputs fewer dimensions than there are
    ground-truth factors ($m < d$), so $|S| < d$: some factors are unrecoverable regardless of $\pi_T$. While this is lossy in the standard sense, it may be a valid lossless compression of information in the case of redundant ground truth latent factors. 
    E.g., under \dbeyond{3} or
    \dbeyond{4}, the ground-truth factors contain deterministic redundancies, so an encoder that recovers all  $d_{\mathrm{eff}}$ independently varying factors already captures the full information of $\rvz$
    (\cref{def:effective-dim}).  \cref{def:identifiability-general} reports only \emph{which} factors appear in~$S$; judging whether $|S| \geq d_{\mathrm{eff}}$ constitutes lossless recovery requires additionally knowing the constraint structure of the DGP.
    No current metric makes this distinction: all treat $|S| < d$ uniformly, whether the omitted factors are redundant or independently informative.
    \circuit{$(T,\; I)$: $|S| = 2 = d_{\mathrm{eff}}$.}

    \item \emph{\ebeyond{8}{mo}} --- \textbf{Distributed}.\enspace
        A type of \emph{overcomplete} ($m > d$) code.
        Ground-truth factors are recoverable only through a
        \textcolor{amethyst}{many-to-one} map, multiple codes
        jointly encode a single factor, and $r$ must aggregate across them.
        Coordinate-wise metrics implicitly assume
        each factor is encoded by a single code.
        \circuit{$(a_1, a_2, T, I, V)$ where
        $V = \sqrt{a_1^2 + a_2^2}$ is fully
        determined by  $(a_1, a_2)$, yet neither alone
        predicts~$V$.}
    
\end{itemize}

Additional overcomplete geometries---linear duplication
(\ebeyond{5}{mo}), nonlinear duplication (\ebeyond{6}{mo}), and linear superposition (\ebeyond{7}{om})---are constructed and evaluated in \cref{sec:results}.   

\noindent
\textbf{\rand{9} --- Control baseline.}\enspace
$\hatz \sim \mathrm{Uniform}([0,1]^m)$, independent of~$\x$.
Every metric should return~${\approx}\,0$.
\noindent
With these definitions in hand, we can formally characterise metric failures.
Formal constructions in \cref{apx:encoder-taxonomy}.


\section{Metrics as Measurement Instruments}\label{sec:results}
We study the structural sensitivity of identifiability metrics through controlled synthetic experiments. In each experiment, we sample ground-truth factors $\rvz \in \sR^d$ according to a DGP type (\dcrl{1}--\dbeyond{4}) and construct representations $\hat{\rvz} = T(\rvz)$ via a transformation matching the encoder type (\ecrl{1}{oo}--\ebeyond{10}{oo}). \textit{The representation encoder is not learned}. This design isolates metric misspecification from optimisation artefacts: every failure we observe is a property of the metric, not of training. Unless otherwise noted, we report results for $n = 1000$ samples, $d = 5$ ground-truth factors, and average over 5 seeds; confidence bands show $95\%$ intervals. For metrics requiring a trained predictor (\dci, \rsq),
data is split into ($80/20$) training  and test sets. Full experimental details and parameter definitions are in \cref{apx:expts}\footnote{We will release a unified implementation of all metrics with improved robustness, and our metric evaluation suite upon acceptance.}.

\begin{table}[t]
\centering
\caption{\textbf{$m/d$ (overcompleteness ratio), $d/n$ (sample ratio), and $m/n$ (representation-to-sample ratio). }}
\label{tab:experiment-parameters}
\vspace{-.5em}
\footnotesize
\setlength{\tabcolsep}{2pt}
\renewcommand{\arraystretch}{1.0}
\begin{tabular}{@{}cp{5.2cm}c@{}}
\toprule
\textbf{Sym.} & \textbf{Meaning} & \textbf{Range} \\
\midrule
\multicolumn{3}{@{}l}{\textsc{Scaling parameters}}\\[2pt]
$n$ & \# i.i.d.\ paired samples & $50$--$10\mathrm{k}$\\
$m$ & Dim.\ of $\hat{\rvz}\!\in\!\mathbb{R}^m$ & $1$--$200$\\
$d$ & \# ground-truth factors $\rvz\!\in\!\mathbb{R}^d$ & $2$--$20$\\
\midrule
\multicolumn{3}{@{}l}{\textsc{Complexity parameters}}\\[2pt]
$\rho$ & Pairwise correlation (\dcrl{2}); off-diagonal entries of $\Sigma$ & $(-1,1)$\\[3pt]
$\alpha$ & Nonlinearity strength (\ecrl{2}{oo}); $\alpha{=}0$: linear, $\alpha{=}1$: fully nonlinear\newline {\scriptsize$\hat{z}_j\!=\!(1{-}\alpha)s_j z_{\pi(j)}{+}\alpha\,h_j(z_{\pi(j)})$} & $[0,1]$\\[3pt]
$\kappa$ & Condition \# of mixing matrix (\ecrl{2}{om}, \ebeyond{7}{om}); $\kappa{=}1$: orthogonal, $\kappa{=}50$: ill-cond.\newline {\scriptsize$A\!=\!U\mathrm{diag}(\mathrm{linspace}(1,\kappa^{-1}\!,d))V^\top$} & $1$--$50$\\
\bottomrule
\vspace{-1em}
\end{tabular}
\end{table}

\parless{Metrics evaluated.} We evaluate the metrics introduced in \cref{sec:taxonomy-main}, grouped into four families: \emph{correlation-based} (\mcc{}-P/S, \mcc{}-RDC  \citep{lopez2013randomized}), regression-based (\dci{}-D, \rsq{}), and (mutual information) MI-based (MIG \citep{chen2018isolating}, InfoMEC \citep{hsu2023disentanglement}), and conditional independence testing based (T-MEX \citep{yao2025third}). In main text, we focus on the commonly used metrics spanning the first two families: \mcc{}, \dci{}-D, and \rsq{}.

\begin{figure*}[t]
    \centering
    \includegraphics[width=\linewidth]{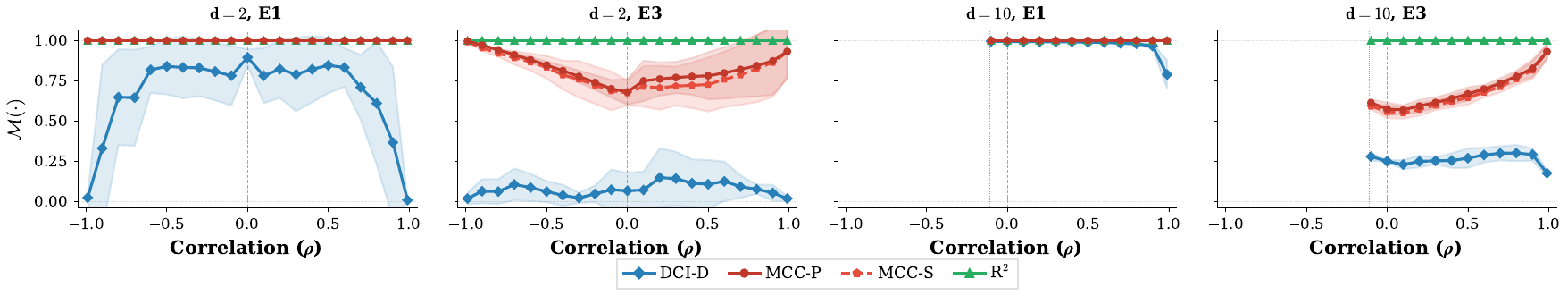}
    \vspace{-.5em}
    \caption{
    \textbf{\mcc{} conflates correlation with
    identifiability.} Under \ecrl{3}{om}, \mcc{} increases with $\rho$ and approaches the score of the perfectly disentangled encoder \ecrl{1}{oo} at high correlation, despite the encoder remaining entangled. \dci{} better separates \ecrl{1}{oo} from \ecrl{3}{om} but collapses to near-zero scores making it hard to distinguish from a non-identifiable encoder. The bias sharpens with increasing~$d$. See \cref{fig:apx-exp03-apx} for all metrics.}
    \label{fig:sign_asymmetry}
\end{figure*}

\parless{Sanity checks.} We first ask: \emph{do metric scores remain stable when the encoder perfectly recovers each factor (\ecrl{1}{oo}), but the DGP varies from independent to correlated to functionally redundant?} Any metric faithful to the equivalence class should return $\approx 1$ across \dcrl{1}--\dbeyond{4} under \ecrl{1}{oo}, since the encoder--factor relationship is identical in all cases. We find that \mcc{}-P, \mcc{}-S, and \rsq{} have outputs $\approx 1$, while \dci{}-D exhibits a systematic dip under \dbeyond{3}, particularly at small $d$ (\cref{fig:apx-exp01-main}; the dip diminishes as $d$ grows from 5 to 20 but does not vanish). The dip arises because the redundant factor creates collinearity in the regression probe, inflating the importance mass assigned to the dependent factor and reducing the disentanglement score. This persists as $n$ is increased (\cref{fig:apx-exp10-main}).
A second test is to assess sensitivity to \emph{encoder} nonlinearity rather than DGP structure; \cref{fig:apx-exp02-main} shows that flat curve with \mcc{}-S and \dci{}-D, as expected.

\subsection{Correlated and entangled latent factors lead to both false positives and false negatives}\label{subsec:correlation}

We first study how latent-factor correlation (\dcrl{2}) interacts with metric scores under encoders \ecrl{1}{oo} (perfectly disentangled) and \ecrl{3}{om} (linearly entangled).
The encoder is held fixed; only the pairwise correlation $\rho \in (-1, 1)$ among ground-truth factors varies.
Any change in the metric score is therefore a pure artifact of the latent covariance structure.

\begin{property}[Invariance to latent correlation]
\label{prop:rho-invariance}
For $\z \in \sR^d$  with pairwise correlations
$\Corr(Z_i, Z_j) = \rho_{ij}$, fix encoder $f$.
A metric $\M$ is \emph{invariant to the latent correlation structure}
if, for every encoder~$f$, $\M\bigl(\hat{\rvz},\, \rvz\bigr)$ does not depend on $\bigl(\rho_{ij}\bigr)_{i \neq j}$ and only depends on $f$.
\end{property}

Violation of \cref{prop:rho-invariance} means $\M$ conflates representation quality with the covariance structure of the DGP.

\parless{Setup: \dcrl{2} + \ecrl{1}{oo}/\ecrl{3}{om}.} Consider $d$ ground-truth factors with $\rvz \sim \mathcal{N}(\mathbf{0}, \Sigma), \Sigma_{ii} = 1, \Sigma_{ij} = \rho \text{ for } i \neq j$.
Note that the equicorrelation matrices are positive semidefinite only for $\rho \ge -1/(d{-}1)$; at $d = 10$ this gives $\rho \gtrsim -0.11$, so strongly negative correlations are infeasible at moderate~$d$.
\ecrl{1}{oo} is realised as $\hat z_j = s_j z_j$ with $s_j > 0$.
For \ecrl{3}{om}, the encoder is a full-rank linear map
$\hat{\rvz} = \rmA \rvz + \rvb$ with $\rmA = \rmU \, \mathrm{diag}(\mathrm{linspace}(1, \kappa^{-1}, d))
\, \rmV^\top$, where $\rmU, \rmV$ are random orthogonal matrices and $\kappa \ge 1$ controls the condition number (degree of entanglement).

\parless{Theoretical analysis.}
We derive a closed-form expression for \mcc{}-P under \dcrl{2}~+~\ecrl{3}{om} (\cref{apx:theory-correlation-e3}), yielding:

\begin{proposition}[\mcc{} produces false positives under correlation]
\label{prop:mcc-correlation}
Under \dcrl{2}~+~\ecrl{3}{om}, \mcc{}-P depends explicitly on~$\rho$, violating \cref{prop:rho-invariance}. Moreover, at both extremes $\rho \to +1$ and $\rho \to -1$, $\mcc{}(\rho) \to 1$, despite an entangled encoder.
\end{proposition}

\cref{prop:mcc-correlation} predicts not merely a sensitivity issue w.r.t. $\rho$, but a failure where the metric saturates at $1$ even for an entangled encoder identified only up to $\cG_{\mathrm{aff}}$. Whereas \mcc{} is designed to distinguish such encoders from ones identified up to $\cG_{\mathrm{perm}}$, making an entangled representation indistinguishable from a disentangled one. Under correlated factors and non-axis-aligned encoders, \mcc{} systematically overestimates identifiability, the gap between \ecrl{1}{oo} and \ecrl{3}{om} narrows as $\rho$ increases (\cref{fig:sign_asymmetry}). We observe that the gap and the bias sharpens with growing~$d$. \cref{fig:exp04-rho-kappa} studies the interaction between $\rho$ and $\kappa$ at $d = 10$, confirming variation of each metric's values with $\rho$ rather than $\kappa$.

\begin{tcolorbox}[
  colback=gray!34,
  colframe=gray!34, 
  boxrule=0pt,
  arc=0pt,
  left=2pt,right=2pt,top=2pt,bottom=2pt
]
\textbf{Takeaway.} \mcc{} cannot reliably compare representations learned from correlated data, scoring near $1$ at high~$\rho$ (false positive) for an entangled encoder. \dci{}-D is overly sensitive to~$\kappa$, scoring near zero for any non-trivial entanglement(false negative; \cref{fig:exp04-rho-kappa}).
\end{tcolorbox}

\subsection{Metrics cannot detect multi-factor redundancy}
\label{subsec:undercomplete}

We now study what happens when the encoder outputs fewer dimensions than the number of ground-truth factors ($m < d$). We construct \ebeyond{4}{oo} by selecting $m$ factors and applying elementwise rescaling, so the retained factors are \emph{perfectly} identified. The central question is: \emph{can metrics distinguish an encoder that drops a redundant factor (lossless) from one that drops an informative factor (lossy)?}

\begin{property}[Faithfulness to effective dimensionality]
\label{prop:deff-sensitivity}
Let $\z \in \sR^d$ have effective dimensionality $d_{\mathrm{eff}} \le d$ (\cref{def:effective-dim}).
$\M$ is \emph{faithful to the effective dimensionality} if $\M = 1$ whenever the encoder recovers all $d_{\mathrm{eff}}$ independently varying factors (even if $m < d$), and $\M < 1$ whenever the encoder fails to recover at least one independently varying factor.
\end{property}

\parless{Setup: \dcrl{1}/\dbeyond{3} + \ebeyond{4}{oo}.}
Under \dcrl{1}, all $d$ factors are independent, so every omission is lossy. Under \dbeyond{3}, one factor is a deterministic function of another ($z_2 = z_1^3$, so $d_{\mathrm{eff}} = d - 1$); dropping $z_2$ is
lossless. Under \dbeyond{4}, one factor depends on two others ($z_k = g(z_i, z_j)$, $d_{\mathrm{eff}} = 9$); dropping $z_k$ is again lossless. In all cases: $\hat z_j = s_j z_j$ for $j \in S$, $|S| = m$.

\begin{figure*}[t]
    \centering
    \includegraphics[width=0.85\linewidth]{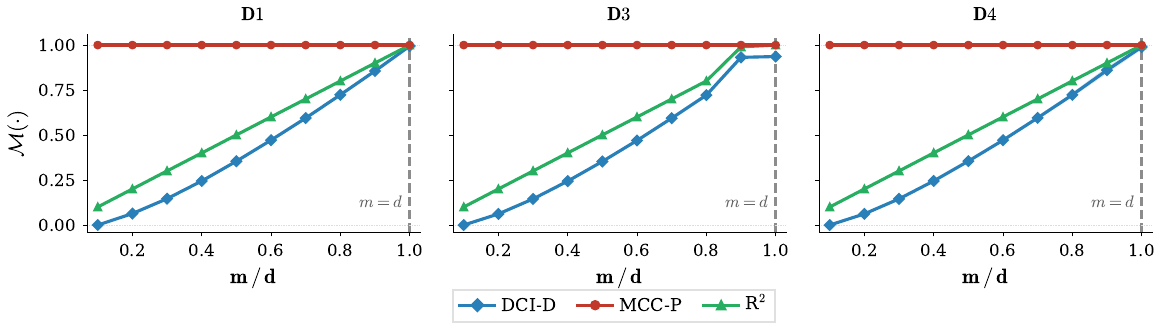}
    \caption{\textbf{Regression-based metrics detect single-factor redundancy.} \emph{Left} (\dcrl{1}): every dropped factor is informative; \rsq{} follows $m/d$ and \dci{}-D declines steadily. \mcc{}-P/S report $1.0$ even at $m{=}1$ (false positive).
    \emph{Right} (\dbeyond{3}): at $m{=}9  d_{\mathrm{eff}}$, the dropped factor is redundant ($z_2 = z_1^3$); \rsq{} and \dci{}-D plateau near~$1.0$, correctly recognising ossless compression, then decline as informative factors are removed. \mcc{}-P/S remain at $1.0$ throughout both panels and cannot distinguish the two.}
\label{fig:dropping_d1_d3}
\end{figure*}

\cref{fig:dropping_d1_d3} reveals a split between metric families. \mcc{}-P/S perform optimal one-to-one matching and score only matched pairs, yielding $1.0$ for any $m \ge 1$ regardless of whether omitted factors are redundant or informative. \rsq{} and \dci{}-D train a probe to predict \emph{all} $d$ factors from the representation. Under \dcrl{1} (left), unrepresented factors are unpredictable and \rsq $\approx m/d$. Under \dbeyond{3} (middle), the redundant factor $z_2 = z_1^3$ is predictable from the retained~$z_1$, so \rsq{} and \dci{}-D stay near~$1.0$ at $m = d_{\mathrm{eff}}$, thus correctly satisfying \cref{prop:deff-sensitivity}. Under \dcrl{2} (correlated factors), \rsq $> m/d$ because the probe partially predicts dropped factors from correlated retained ones; we defer this to \cref{fig:apx-exp06-all}.


Under \dbeyond{4}, the redundant factor $z_k = g(z_i, z_j)$ depends jointly on two other factors. Although $d_{\mathrm{eff}} = 9$ (same as \dbeyond{3}), the nonlinear probe fails to detect the relationship
\cref{fig:dropping_d1_d3} shows a false negative: a lossless encoder is penalised as though it were lossy.


\begin{tcolorbox}[
  colback=gray!34,
  colframe=gray!34, 
  boxrule=0pt,
  arc=0pt,
  left=2pt,right=2pt,top=2pt,bottom=2pt
]
\textbf{Takeaway.} Regression-based metrics (\rsq{}, \dci{}-D) detect single-factor redundancy (\dbeyond{3}), but no current metric detects multi-factor redundancy (\dbeyond{4}).
\end{tcolorbox}

\subsection{Metrics cannot compare overparametrised encoders}\label{subsec:overcomplete}
When $m > d$, the encoder outputs more codes than there are factors. We  first formalise the desired property we want a metric to exhibit.

\begin{property}[Invariance to overcompleteness.]
\label{prop:oc-invariance}
Let $f$ be an encoder with $m = d$ that identifies factors up to equivalence class $\cG$ (\cref{def:identifiability-general}), and let $f'$ be an
overcomplete encoder ($m > d$) that identifies the same factors up to the same~$\cG$. A metric $\M$ is invariant to the overcomplete dimension if $|\M(f') - \M(f)| \leq \epsilon(n)$, where
$\epsilon(n) \to 0$ as $n \to \infty$.
\end{property}

Violation of \cref{prop:oc-invariance} implies that the metric either spuriously rewards extra codes that add no
per-factor information, or that it penalises an encoder that
has not lost any factors but merely represents them using multiple codes. In either case, this would represent a metric conflating
dimensionality with identifiability.

\parless{Setup.}
We compare four overcomplete geometries
(\ebeyond{5}{mo}--\ebeyond{8}{mo}) against the
matched-dimension entangled baseline \ecrl{3}{om}, under
\dcrl{1}.
We first fix $m/d = 2$ ($d{=}20$, $n{=}1600$) and then sweep
$m/d \in \{1, 1.5, 2, 3\}$ ($d{=}5$, $n{=}1000$) to test
whether the results are stable as overcompleteness increases. At moderate overcompleteness ($m/d{=}2$), all metrics correctly separate entangled from disentangled encoders
(\cref{fig:apx-exp09-bars}).
\cref{fig:overcomplete_sweep} tests whether this holds as $m/d$ increases.

\cref{fig:overcomplete_sweep} shows that increasing $m/d$ does not uniformly increase or decrease scores.  Instead, it
amplifies the mismatch between each metric's implicit equivalence class and the encoder's geometry. Two cases are particularly informative.

\parless{\mcc{} cannot be used for distributed codes (\ebeyond{8}{mo}).}
Each factor is encoded as $k$ codes
(e.g., $\sin z_j, \cos z_j$ for $k{=}2$); no single code suffices to recover the factor. \mcc{} pairs each factor with exactly one code, so the best match (say $\sin z_j$) has correlation strictly less than $1$ with $z_j$. As $k$ grows, per-code information thins and \mcc{}-P drops
from ${\sim}\,0.85$ at $m/d{=}2$ to ${\sim}\,0.65$ at $m/d{=}10$, even though the factors are fully recoverable from
their code subsets. This is a structural failure and it worsens monotonically with
$m/d$. \dci{}-D does not exhibit this failure as the nonlinear probe can fit all $m$ codes, selecting the $k$ codes in each disjoint subset. Since each selected code predicts only one factor, $D_i = 1$ and \dci{}-D stays near $1.0$ at all tested $m/d$.

\parless{Linear entanglement at high $m/d$ increase \dci{}-D (\ebeyond{7}{mo}).}
However, \dci{}-D increases substantially even for the linearly entangled encoder, from ${\sim}\,0.42$ at $m/d{=}1.5$ to ${\sim}\,0.80$ at $m/d{=}10$. This produces a false positive, that could mislead model comparison. 

Only \ebeyond{5}{mo} (elementwise linear duplication) satisfies
\cref{prop:oc-invariance} across all metrics and all tested
$m/d$ values.

\begin{figure*}[t]
    \centering
    \includegraphics[width=0.7\linewidth]{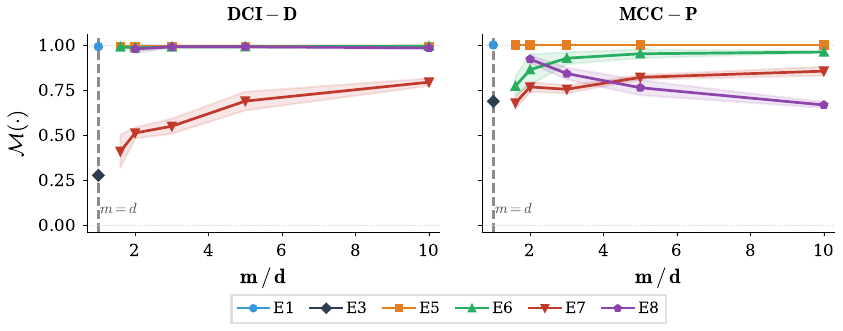}
    \caption{\textbf{Sweeping $m/d$ reveals encoder-specific
    violations of \cref{prop:oc-invariance}.}
    \ebeyond{5}{mo} (elementwise linear duplication) is the only
    encoder for which all metrics remain stable.
    \dci{}-D increases for entangled \ebeyond{7}{mo} as $m/d$ grows. \rsq{} collapses for nonlinear \ebeyond{6}{mo}.
    \mcc{}-P decreases for disjoint \ebeyond{8}{mo}.
    $d{=}5$, $n{=}1000$.}
    \label{fig:overcomplete_sweep}
\end{figure*}

\begin{tcolorbox}[
  colback=gray!34,
  colframe=gray!34,
  boxrule=0pt,
  arc=0pt,
  left=2pt,right=2pt,top=2pt,bottom=2pt
]
\textbf{Takeaway.} No metric satisfies \cref{prop:oc-invariance} across all encoder types; overcomplete representations require multi-metric evaluation or matched-dimension controls.
\end{tcolorbox}

\subsection{High representation-to-sample ratio increases risk of false positives}
\label{subsec:false-positive}

A metric should assign $\approx 0$ to a random encoder that carries no information about~$\rvz$.
Unlike the population-level misspecification studied in \cref{subsec:correlation,subsec:undercomplete,subsec:overcomplete}, the false-positive inflation in this section is a finite-sample phenomenon: the bias vanishes as $n \to \infty$.
We include it because the sample regimes encountered in practice---particularly in mechanistic interpretability, where $m/n$ routinely exceeds $1$---are far from this asymptotic limit, making the finite-sample floor operationally indistinguishable from structural misspecification.

\begin{property}[Insensitivity to uninformative encoders]
\label{prop:false-positive}
For any encoder $f$ independent of~$\rvz$, a metric $\M$ should satisfy $\M \approx 0$ regardless of the dimensionality ratio $m/d$ and sample size~$n$.
\end{property}

\parless{Setup.}
We construct a null encoder \rand{9} ($\hat{\rvz} \sim \mathrm{Uniform}([0,1]^m)$ and sweep  over both $m/d$ and $m/n$.

\begin{figure}[t]
    \centering
    \includegraphics[width=\linewidth]{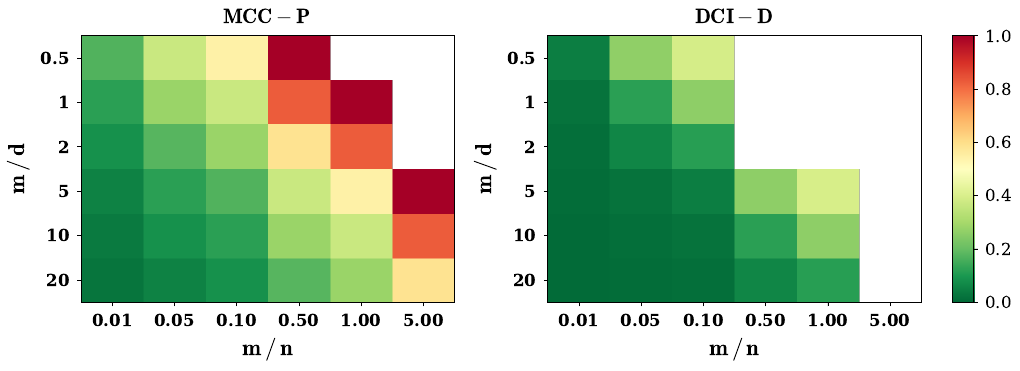}
    \caption{\textbf{Null-encoder scores reveal that $m/n$ (columns), not $m/d$ (rows), governs false-positive inflation.}
Each cell shows the metric score of a random encoder \rand{9} that carries no information about~$\rvz$; any score above~$0$ is a false positive.
See \cref{apx:mcc-false-positive} for the theoretical analysis and \cref{fig:apx-exp15-gauss} for the Gaussian null (nearly identical).}
    \label{fig:exp15-null}
\end{figure}

\mcc{} violates \cref{prop:false-positive} whenever $m/n \geq 0.1$. Reading along any row of \cref{fig:exp15-null} (fixed $m/d$, varying $m/n$), scores increase steadily; reading along any column (fixed $m/n$, varying $m/d$), scores are approximately constant.
The false-positive rate is therefore governed by~$m/n$, not~$m/d$.
At $m/n = 0.5$ and $m/d = 1$, \mcc{}-P reports $0.83$ for a representation that is pure noise.
\dci{}-D satisfies \cref{prop:false-positive} in the large-sample regime, but shows moderate
inflation at higher estimation ratios, particularly when $m/d$ is small. \rsq{} satisfies the property across the entire $(m/d,\, m/n)$ grid.

\parless{Theoretical analysis.}
We derive this behaviour in \cref{apx:mcc-false-positive}. Under the null, each entry of the $m \times d$ sample
correlation matrix has mean zero and standard deviation ${\approx}\, 1/\sqrt{n}$ by the Central Limit Theorem (CLT). Hungarian matching picks the best one-to-one assignment from $m$ candidates per column; the expected maximum of $m$ draws from $\mathcal{N}(0, 1/n)$ scales as
$\sqrt{2 \log m / n}$ \citep{cai2011limiting}, giving
$\mathbb{E}[\mcc\text{-P}] \gtrsim \sqrt{2\log m / n}$ (up to a constant).
This depends on $m$ and $n$ but not on $d$, explaining the column-varying, but constant across rows pattern in \cref{fig:exp15-null}.

\parless{Practical implications.}
The $m/n \gtrsim 0.1$ threshold is routinely exceeded: evaluating a pretrained LLM such as Llama-3.2-8B ($m = 4096$) with a few hundred samples gives $m/n \in [0.5, 10]$; even standard disentanglement benchmarks with $m = 64$ and $n = 500$ labelled samples yield $m/n > 0.1$.
\dci{}-D requires more samples and exhibits the same inflation.
\rsq{} is the most robust to false positives, but requires $n \gtrsim 500$ under nonlinear encoders (\cref{fig:apx-exp10-main}).

\begin{tcolorbox}[
  colback=gray!34,
  colframe=gray!34, 
  boxrule=0pt,
  arc=0pt,
  left=2pt,right=2pt,top=2pt,bottom=2pt
]
\textbf{Takeaway.} \mcc{} is unreliable whenever $m/n \gtrsim 0.1$. Always verify $m \ll n$ and report null-encoder baselines alongside metric scores.
\end{tcolorbox}


\section{Conclusion}

All existing identifiability metrics can be deceptive (\cref{fig:four-panel-summary}).
We provide a taxonomy (\cref{sec:taxonomy-main}) and theoretical and empirical analyses to characterise these failure modes, then propose four properties (\cref{prop:rho-invariance,prop:deff-sensitivity,prop:oc-invariance,prop:false-positive}) for future metric design.
We distil our findings into a practitioner checklist (\cref{apx:practitioner-checklist}) and a metric selection lookup table (\cref{tab:metric-property-summary}). Our results have direct consequences for any pipeline that uses identifiabilility metrics to make downstream predictions. 

\parless{Limitations.}
Our analysis uses synthetic encoders by design, to isolate metric misspecification from optimisation artefacts.
The taxonomy does not cover stochastic encoders or discrete factors, all of which arise in practice. Lastly, a systematic study of how metric failures manifest across different families of learned encoders (rather than constructed ones) would be a complementary direction.

\clearpage
\bibliography{refs}
\clearpage
\onecolumn

\tableofcontents
\clearpage

\addtocontents{toc}{\protect\setcounter{tocdepth}{2}}
\appendix

\section{Practitioner Checklist}
\label{apx:practitioner-checklist}

A metric score is interpretable only if two conditions hold: (1) the $(\text{DGP}, \text{encoder})$ pair lies in a structurally valid region for that metric, and (2) the sample size $n$ is large enough relative to the relevant dimension to ensure estimation stability.
Before reporting scores, verify the following conditions.

\parless{Before evaluation.}
\begin{enumerate}[leftmargin=*,itemsep=2pt]
  \item \textbf{Check the overparametrisation ratio $m/n$.}
    If $m/n > 0.1$, \mcc{} scores are unreliable: the expected score under a null encoder exceeds $\sqrt{2\log m / n}$ (\cref{subsec:false-positive}).
    Increase $n$ or reduce $m$ before interpreting results.
  \item \textbf{Report a null-encoder baseline.}
    Compute every metric on a random or constant encoder with the same $(m, n, d)$.
    Without this baseline, false positives are indistinguishable from genuine identifiability (\cref{subsec:false-positive}).
  \item \textbf{Know your DGP assumptions.}
    Determine whether latent factors are independent (\dcrl{1}) or correlated (\dcrl{2}), and whether the representation is matched ($m = d$), overcomplete ($m > d$), or undercomplete ($m < d$).
\end{enumerate}

\parless{Choosing a metric.}
\begin{enumerate}[leftmargin=*,itemsep=2pt,resume]
  \item \textbf{Matched dimension, independent factors} ($m = d$, \dcrl{1}): all three metrics (\mcc{}, \dci{}-D, \rsq{}) are reliable.
  \item \textbf{Correlated factors} (\dcrl{2}): prefer \rsq{}.
    \mcc{} conflates correlation with identifiability (\cref{prop:mcc-correlation}); \dci{}-D collapses under moderate entanglement (\cref{subsec:correlation}).
  \item \textbf{Overcomplete representations} ($m > d$): no single metric is reliable across all encoder geometries.
    Use multiple metrics and compare against matched-dimension controls (\cref{subsec:overcomplete}).
  \item Consult \cref{tab:metric-property-summary} for a full lookup table.
\end{enumerate}

\parless{Interpreting scores.}
\begin{enumerate}[leftmargin=*,itemsep=2pt,resume]
  \item \textbf{A high \mcc{} does not imply identifiability} when $m/n$ is large or factors are correlated.
  \item \textbf{A high \dci{}-D does not imply disentanglement} when the encoder is overcomplete and linearly entangled.
  \item \textbf{No pairwise metric detects multi-factor redundancy} (\dbeyond{4}); higher-order statistics are needed (\cref{subsec:undercomplete}).
\end{enumerate}

\section{Related Work}
\label{sec:related-work}
\parless{Identifiability theory}. Nonlinear ICA \citep{comon1994independent, hyvarinen1999nonlinear}establishes sufficient conditions under which latent factors can be recovered up to well-defined equivalence classes. Identifiability guarantees often leverage auxiliary variables \citep{hyvarinen2019nonlinear, khemakhem2020variational}, temporal structure \citep{hyvarinen2016unsupervised}, mechanism sparsity \citep{lachapelle2022disentanglement}, and restricted model classes \citep{khemakhem2020ice, marconato2024all}. Causal representation learning extends these results by additionally requiring that identified factors admit causal semantics with predictable behaviour under interventions \citep{scholkopf2021toward}. These works establish when identifiability holds in theory. We study whether the metrics used to verify these guarantees empirically are faithful to the equivalence classes the theorems provide. Our results indicate that even with correlations between latent factors, reliability on metrics drops \cref{subsec:correlation}.

\parless{Identifiability and disentanglement metrics.} A substantial body of work has proposed metrics for evaluating learned  representations against ground-truth factors, including \dci{} \citep{eastwood2018framework}, MIG \citep{chen2018isolating}, \mcc{} \citep{khemakhem2020ice}, InfoMEC \citep{hsu2023disentanglement}, and T-MEX \citep{yao2025third}. \citet{sepliarskaia2019not} showed that several metrics disagree on comparing methods and cautioned against relying on a single score. \citet{carbonneau2022measuring} surveyed metrics and noted the lack of a unified framework connecting metric assumptions to evaluation validity. Our work differs from both. Instead of comparing metric rankings aross methods, we identify the structural conditions on the DGP and encoder geometry under which each metric's score is interpretable, and show that the resulting failure modes are misspecification, not optimisation failures.

\parless{Overcomplete representations and mechanistic interpretability.} Recent work in mechanistic interpretability uses sparse autoencoders to extract interpretable features from pretrained models \citep{elhage2022toy}, and identifiability of these features is increasingly recognised as necessary for reliable interpretation \citep{song2025position, joshi2025identifiablesteeringsparseautoencoding, mueller2025isolation}. These settings are inherently overcomplete $m>>d$ and sample-constrained $m>>n$. We show that current metrics are not reliable under overcompleteness: \mcc{} does not work for overcomplete distributed codes, \dci{}-D may spuriousy reward a linearly entangled representation (\cref{subsec:overcomplete}), and the high $m/n$ ratios typical of these evaluations
may push the metrics into the regime where they can score high even with a random representation (\cref{subsec:false-positive}).

\parless{Relationship to prior evaluation studies.}
\citet{locatello2019challenging} demonstrated that unsupervised disentanglement learning requires inductive biases, studying how \emph{learning algorithms} behave under different model and data assumptions.
Our work is complementary: we study how \emph{evaluation metrics} behave under different structural regimes, holding the encoder fixed.
Their finding that unsupervised disentanglement is impossible without inductive biases is orthogonal to our finding that even supervised metrics are structurally misspecified under conditions the underlying identifiability theorems explicitly permit.
\citet{eastwood2023dciesextendeddisentanglementframework} extended \dci{} to handle dimension mismatch; our \cref{def:identifiability-general} generalises this to arbitrary equivalence classes and connects it to the full DGP taxonomy, revealing failure modes beyond what dimension-mismatch alone predicts.

\section{Metric Usage Review}
\label{apx:review}

We conducted a systematic review of evaluation metrics used in causal representation learning (CRL) and nonlinear independent component analysis (ICA). Using the Semantic Scholar API, we retrieved papers published between 2020 and 2025 at major ML conferences (NeurIPS, ICLR, ICML, AISTATS, UAI, AAAI, CLeaR, JMLR) based on the terms 'causal representation learning' and 'nonlinear ICA'. \textbf{Among the 62 papers identified, most relied on \mcc{} (25), followed by \rsq{} (9) and \dci{} (2).} None employed more recent metrics such as MIG or T-MEX. Finally, several papers did not use standard metrics at all, instead reporting performance in terms of objective optimization or relying on qualitative assessments.

\textbf{Also, nonlinear ICA papers use \mcc{} ($61\%$) more often than CRL ($29\%$).}

\section{Taxonomy}
\label{apx:taxonomy}


\subsection{Data Classes}
\label{apx:dgp-taxonomy}

Let $\rvz = (Z_1,\dots,Z_d)^\top \in \sR^d$ denote the ground-truth
latent factors with joint density~$p(\z)$.
We classify the factor distribution along two axes:
\emph{statistical dependence} (mutual information) and
\emph{functional dependence} (deterministic constraints).
Classes~\dcrl{1}--\dcrl{2} operate within the standard CRL setting
($d_{\mathrm{eff}} = d$); Classes~\dbeyond{3}--\dbeyond{4} extend it to
settings where functional constraints reduce the effective
dimensionality ($d_{\mathrm{eff}} < d$; cf.\
\cref{def:effective-dim}).

\medskip

\paragraph{\dcrl{1} --- Independent factors.}
The factors are mutually independent and non-redundant:
\[
  p(z_1,\dots,z_d)
  = \prod_{j=1}^{d} p(z_j),
  \qquad
  I(Z_i;\, Z_j) = 0
  \quad \forall\; i \neq j.
\]
No statistical, functional, or structural dependence exists among
factors.  In particular, $\operatorname{Corr}(Z_i, Z_j) = 0$ for all
$i \neq j$, and $d_{\mathrm{eff}} = d$.

\emph{Canonical example:}\;
$Z_1, Z_2 \stackrel{\mathrm{i.i.d.}}{\sim}
\operatorname{Uniform}(0,1)$.

\medskip

\paragraph{\dcrl{2} --- Correlated (statistically dependent) factors.}
The factors share information but each retains a unique degree of
freedom; no factor is a deterministic function of any subset of the
others:
\[
  \exists\; i \neq j:\;
  I(Z_i;\, Z_j) > 0,
  \qquad
  H(Z_j \mid Z_{\setminus j}) > 0
  \quad \forall\; j,
\]
where $Z_{\setminus j} := (Z_1,\dots,Z_{j-1},Z_{j+1},\dots,Z_d)$.
The first condition asserts statistical dependence; the second asserts
non-redundancy: no factor is determined by the rest.
Hence $d_{\mathrm{eff}} = d$.

\emph{Canonical example:}\;
$(Z_1, Z_2, Z_3)^\top \sim \mathcal{N}\!\bigl(0, \Sigma\bigr)$,\;
$\Sigma =
\begin{psmallmatrix}
  1 & 0 & 0 \\
  0 & 1 & \rho \\
  0 & \rho & 1
\end{psmallmatrix}$
with $\rho \in (0,1)$.

\emph{Remark.}\;  \dcrl{2} subsumes both causal dependence
($Z_j = f(Z_i) + \varepsilon$, $\varepsilon \not\equiv 0$) and
confounded dependence (shared latent common cause), as well as
nonlinear dependence invisible to linear measures.
For instance, $Z_1 \sim \cN(0,1)$,
$Z_2 = Z_1^2 + \varepsilon$, $\varepsilon \sim \cN(0, \sigma^2)$
satisfies $\operatorname{Corr}(Z_1, Z_2) = 0$ yet
$I(Z_1; Z_2) > 0$: the dependence is real but purely nonlinear.
As long as $H(\varepsilon) > 0$, the relationship is non-deterministic and falls under \dcrl{2}.
\textbf{In this paper, we'll focus on linear non-deterministic dependence only. }

\medskip

\paragraph{\dbeyond{3} --- Single-factor functional constraint.}
At least one factor is a deterministic function of exactly one other
factor, reducing the effective dimensionality:
\[
  \exists\; i \neq j,\; \exists\; f:\sR\to\sR:\quad
  Z_j = f(Z_i)
  \quad\text{a.s.},
  \qquad
  H(Z_j \mid Z_i) = 0.
\]
Two structurally distinct subcases arise:

\begin{description}
  \item[\dbeyond{3}A --- Invertible (information-preserving).]
  The map $f$ is injective, so $f^{-1}$ exists.  Then
  $H(Z_j) = H(Z_i)$ and $I(Z_i;\, Z_j) = H(Z_i)$.
  The intrinsic dimension of $(Z_i, Z_j)$ is~$1$, but no information
  is lost.
  \emph{Canonical example:}\; $Z_2 = Z_1^3$.

  \item[\dbeyond{3}B --- Non-invertible (collapsed).]
  The map $f$ is many-to-one, so $f^{-1}$ does not exist.
  Then $H(Z_j) < H(Z_i)$ and $I(Z_i;\, Z_j) = H(Z_j) < H(Z_i)$:
  information is destroyed.
  \emph{Canonical example:}\; $Z_2 = \operatorname{sign}(Z_1)$.
\end{description}

\noindent
In both subcases, $d_{\mathrm{eff}} \le d - 1$ (one constraint removes one
degree of freedom).
\circuit{}$R = R_0(1 + \alpha\, T)$: resistance is an invertible function of temperature (\dbeyond{3}A).

In this paper, \dbeyond{3}A will be of interest to us.

\paragraph{\dbeyond{4} --- Multi-factor functional constraint
(synergistic).}
At least one factor is a deterministic function of two or more other
factors, but not of any single one:
\[
  \exists\; k,\; \exists\; S \subset \{1,\dots,d\}\text{ with }
  |S| \ge 2:\quad
  Z_k = g(Z_S) \quad\text{a.s.},
\]
where $Z_S := (Z_j)_{j \in S}$, and no function of a strict subset of
$Z_S$ determines~$Z_k$.  Formally:
\[
  H(Z_k \mid Z_S) = 0,
  \qquad
  H(Z_k \mid Z_T) > 0
  \quad \forall\; T \subsetneq S.
\]
The constraint cannot be decomposed into single-variable contributions:
dependence is deterministic but \emph{synergistic}.  All pairwise
linear correlations may vanish ($\operatorname{Corr}(Z_j, Z_k) = 0$
for each $j \in S$) even though $(Z_S)$ jointly determines~$Z_k$.

\emph{Canonical example:}\;
$Z_1 = Z_2\, Z_3$, \; $Z_2 \perp Z_3$.
\circuit{$V = I R$: voltage is jointly determined by current and
resistance, so $(T, R, I, V)$ has $d_{\mathrm{eff}} = 2$.}

\bigskip
\noindent
\emph{Summary.}\;
Under \dcrl{1}--\dcrl{2}, $d_{\mathrm{eff}} = d$ (no functional
constraints).
Under \dbeyond{3}--\dbeyond{4}, $d_{\mathrm{eff}} < d$
(deterministic constraints reduce the number of free degrees of
freedom; cf.\ \cref{def:effective-dim}).
The horizontal axis of the validity-domain map (\cref{tab:metric-property-summary}) captures this
progression.

\subsection{Encoder Taxonomy}
\label{apx:encoder-taxonomy}

Let $f : \sR^n \to \sR^m$ denote the learned encoder, producing
$\hat{\rvz} := f(\rvx) = f(g(\rvz)) \in \sR^m$.
We classify encoders by (i)~the equivalence class~$\cG$ up to which
factors are identified (\cref{def:identifiability}), and (ii)~the
dimension ratio $m / d$.
Throughout, $S_d$ denotes the symmetric group on $\{1,\dots,d\}$.


\parless{Matched dimension ($m = d$).}

\paragraph{\ecrl{1}{oo} --- Elementwise linear (permutation \&
rescaling).}
The encoder identifies each factor up to $\cG_{\mathrm{perm}}$:
\[
  \exists\; \pi \in S_d,\;
  \exists\; a_j \neq 0:\quad
  \hat Z_j = a_j\, Z_{\pi(j)},
  \qquad j = 1,\dots,d.
\]
No cross-factor mixing or nonlinear reparameterisation is present
beyond scaling and permutation.
This is the strongest form of identifiability and every metric
should score~$1$.

\emph{Canonical example:}\;
$\hat Z_1 = 2 Z_3$,\;
$\hat Z_2 = -Z_1$,\;
$\hat Z_3 = 0.5\, Z_2$
\;(for $d = 3$).

\medskip

\paragraph{\ecrl{2}{oo} --- Elementwise nonlinear (invertible
componentwise).}
The encoder identifies each factor up to $\cG_{\mathrm{nl}}$:
\[
  \exists\; \pi \in S_d:\quad
  \hat Z_j = g_j(Z_{\pi(j)}),
  \qquad j = 1,\dots,d,
\]
where each $g_j : \sR \to \sR$ is a smooth, invertible scalar
function.
Information is preserved factor-wise, but linear correlation between
$\hat Z_j$ and $Z_{\pi(j)}$ may be misleading.
The parameter~$\alpha$ (\cref{tab:experiment-parameters})
controls the degree of nonlinearity; $\alpha = 0$ reduces to
\ecrl{1}{oo}.

\emph{Canonical example:}\; $\hat Z_j = Z_{\pi(j)}^3$.

\medskip

\paragraph{\ecrl{3}{om} --- Linearly entangled.}
The encoder identifies factors up to $\cG_{\mathrm{aff}}$:
\[
  \hat{\rvz} = \rmA\, \rvz,
  \qquad \rmA \in \sR^{d \times d},\;
  \det(\rmA) \neq 0,
\]
with $\rmA$ not a signed permutation matrix (i.e., at least one row
has two or more nonzero entries).
All factor information is preserved globally, but individual factors
are distributed across coordinates.
The condition number $\kappa(\rmA)$ controls the degree of
entanglement; $\kappa = 1$ reduces to \ecrl{1}{oo}.

\emph{Canonical example:}\;
$\hat Z_2 = a\, Z_2 + b\, Z_3$ with $ab \neq 0$.



\parless{Dimension mismatch ($m \neq d$).}

The standard definition of identifiability
(\cref{def:identifiability}) assumes $m = d$.  When $m \neq d$,
we use the generalised notion of \cref{def:identifiability-general}:
the encoder identifies a subset $S \subseteq \{1,\dots,d\}$ of
factors via a readout $r : \sR^m \to \sR^{|S|}$.

\medskip

\paragraph{\ebeyond{4}{oo} --- Undercomplete ($m < d$).}
The encoder outputs fewer dimensions than there are ground-truth
factors, so $|S| < d$: some factors are unrecoverable regardless of
the readout~$r$.  Each retained factor is encoded elementwise:
\[
  \hat Z_j = a_j\, Z_{i(j)},
  \qquad j = 1,\dots,m,
\]
with all $i(j)$ distinct and $a_j \neq 0$.

Under \dbeyond{3}--\dbeyond{4}, this need not be lossy in the
information-theoretic sense: if the encoder recovers all
$d_{\mathrm{eff}}$ independently varying factors, it captures the full
information of~$\rvz$ (\cref{def:effective-dim}).  No current metric
distinguishes omission of a redundant factor from omission of an
informative one.
\circuit{$(T,\; I)$: $|S| = 2 = d_{\mathrm{eff}}$.}

\parless{Overcomplete encoders ($m > d$).}

We now define four overcomplete encoder types ($m > d$), each corresponding to a distinct code--factor geometry.

\medskip

\paragraph{\ebeyond{5}{mo} --- Overcomplete elementwise linear.}
Each output coordinate is a scaled copy of exactly one ground-truth factor. Let $\sigma \colon \{1,\dots,m\} \to \{1,\dots,d\}$ be a surjective assignment (every factor is represented at least once; some are
duplicated), and let $a_j \neq 0$.  Then:
\[
  \hat Z_j = a_j\, Z_{\sigma(j)},
  \qquad j = 1,\dots,m, \qquad m > d.
\]
The surjectivity of~$\sigma$ ensures no factor is lost; factors
assigned to multiple indices appear as independently scaled copies.
The readout~$r$ must aggregate
(\textcolor{amethyst}{many-to-one}) across codes that share a source
factor.

\emph{Example ($d{=}2$, $m{=}4$):}\;
$\hat Z_1 = 1.3\, Z_1$,\;
$\hat Z_2 = -0.7\, Z_2$,\;
$\hat Z_3 = 0.9\, Z_1$,\;
$\hat Z_4 = -1.5\, Z_2$.

\medskip

\paragraph{\ebeyond{6}{mo} --- Overcomplete, multiple codes per factor.}
The first $d$ output coordinates are elementwise nonlinear
transforms of individual factors (one per factor); the remaining
$m - d$ coordinates are nonlinear functions that may depend on
multiple factors simultaneously:
\[
  \hat Z_j =
  \begin{cases}
    g_j\!\bigl(Z_{\pi(j)}\bigr), & j = 1,\dots,d, \\[4pt]
    \phi_j(Z_1,\dots,Z_d),       & j = d{+}1,\dots,m,
  \end{cases}
\]
where each $g_j : \sR \to \sR$ is an invertible scalar function,
$\pi \in S_d$ is a permutation, and each
$\phi_j : \sR^d \to \sR$ is a (possibly non-invertible) nonlinear
map.
The first $d$ coordinates preserve factor-wise information up to
$\cG_{\mathrm{nl}}$; the additional $m - d$ coordinates introduce
cross-factor codes that carry redundant or mixed information.
Recovery requires a \textcolor{amethyst}{many-to-one} readout that
can select or aggregate across both single-factor and multi-factor
codes.

\emph{Example ($d{=}2$, $m{=}3$):}\;
$\hat Z_1 = \tanh(Z_1)$,\;
$\hat Z_2 = Z_2^3$,\;
$\hat Z_3 = Z_1 \cdot Z_2$.

\medskip

\paragraph{\ebeyond{7}{om} --- Overcomplete, linearly entangled.}
The encoder is a dense linear map with $m > d$:
\[
  \hat{\rvz} = \rmA\, \rvz,
  \qquad
  \rmA \in \sR^{m \times d},\;
  \operatorname{rank}(\rmA) = d,
\]
where at least one row of $\rmA$ has two or more nonzero entries, so
each coordinate of $\hat{\rvz}$ mixes several factors
(\textcolor{burgundy}{one-to-many}).
The matrix $\rmA$ is constructed via its singular value decomposition
$\rmA = \rmU \, \mathrm{diag}(s_1,\dots,s_d) \, \rmV^\top$ with
$\rmU \in \sR^{m \times m}$, $\rmV \in \sR^{d \times d}$ orthogonal,
and $s_1 \ge \cdots \ge s_d > 0$.
The condition number $\kappa(\rmA) = s_1 / s_d$ controls the degree
of entanglement.
Since $\operatorname{rank}(\rmA) = d$, the factor information is
globally preserved; recovery requires $r$ to unmix the linear
superposition.

\emph{Example ($d{=}2$, $m{=}4$):}\;
every $\hat Z_j$ is a distinct linear combination of $Z_1$ and $Z_2$.

\medskip

\paragraph{\ebeyond{8}{mo} --- Overcomplete, nonlinear disjoint
subsets.}
Let $k \ge 2$ be an integer and set $m = k \cdot d$.
There exist pairwise-disjoint index sets
$S_1,\dots,S_d \subset \{1,\dots,m\}$ with $|S_i| = k$,
$\bigsqcup_{i=1}^{d} S_i = \{1,\dots,m\}$, such that each
ground-truth factor~$Z_i$ is encoded \emph{only} in the coordinates
indexed by~$S_i$ (no cross-factor mixing):
\[
  \forall\; j \in S_i:\quad
  \hat Z_j = h_j(Z_{\pi(i)}),
\]
where $\pi \in S_d$ is a permutation and each $h_j : \sR \to \sR$ is
a scalar nonlinear function (not necessarily invertible individually).
The factor is recoverable from its own subset via a decoder
$f_i : \sR^k \to \sR$:
\[
  Z_{\pi(i)} = f_i\!\bigl(\hat Z_j : j \in S_i\bigr).
\]
For $k = 2$, the canonical implementation uses
$\hat Z_{2i} = \sin(Z_{\pi(i)})$,
$\hat Z_{2i+1} = \cos(Z_{\pi(i)})$,
with perfect reconstruction via
$Z_{\pi(i)} = \operatorname{atan2}(\hat Z_{2i},\, \hat Z_{2i+1})$.
For $k > 2$, an interval-based encoding partitions the range of each
factor into $k$ bins; exactly one code per factor is active for each
sample.
Coordinate-wise metrics fail because the readout~$r$ must aggregate
(\textcolor{amethyst}{many-to-one}) across the $k$ codes in
each~$S_i$; no single $\hat Z_j$ suffices to recover~$Z_i$.

\emph{Example ($d{=}2$, $k{=}2$, $m{=}4$):}\;
$(\sin Z_1,\; \cos Z_1,\; \sin Z_2,\; \cos Z_2)$.

\medskip


\parless{Control baselines.}

\paragraph{\rand{9} --- Random (independent of data).}
$\hat{\rvz} \sim \operatorname{Uniform}([0,1]^m)$, independent
of~$\rvx$.  Every metric should return $\approx 0$; any nonzero
score is a false positive.

\bigskip
\noindent
\emph{Summary of code--factor geometry.}
\begin{itemize}[leftmargin=*, itemsep=2pt, topsep=4pt]
  \item \emph{\textcolor{teal}{One-to-one}}: $r$ selects one code per
    factor; each code represents exactly one factor.
    Applies to \ecrl{1}{oo}, \ecrl{2}{oo}.
  \item \emph{\textcolor{amethyst}{Many-to-one}}: multiple codes carry
    information about the same factor; $r$ must aggregate.
    Applies to \ebeyond{5}{mo}, \ebeyond{6}{mo}, \ebeyond{8}{mo}.
  \item \emph{\textcolor{burgundy}{One-to-many}}: each code entangles
    multiple factors; $r$ must unmix.
    Applies to \ecrl{3}{om}, \ebeyond{7}{om}.
\end{itemize}
\begin{conceptbox}{Examples of encoders.}
\label{box:circuit-encoders}
\begin{center}
\setlength{\tabcolsep}{3pt}
\small
\begin{tabular}{@{}ll@{}}
  \ecrl{1}{oo}  & $(2T, -R, 3V, 0.5I)$
                 \\
  \ecrl{3}{oo}  & $\rmA\,(T,R,I,V)^\top + \rvb$
                 \\
  \ebeyond{4}{oo} & $(T,\; I)$
                 \\
  \ebeyond{8}{mo} & $(a_1, a_2, T, I, V)$,\; $R\!=\!\sqrt{a_1^2+a_2^2}$
                 \\
  \rand{9}       & $\hat{\z} \sim \mathcal{N}(0, I_5)$
                 \\
\end{tabular}
\end{center}
\end{conceptbox}
\noindent
\ecrl{1}{oo} and \ecrl{2}{oo} preserve factor-wise information up to
invertible reparameterisations;
\ecrl{3}{om} preserves global information but mixes factors;
\ebeyond{4}{oo} loses information;
\ebeyond{5}{mo} duplicates information via elementwise linear copies;
\ebeyond{6}{mo} combines elementwise nonlinear codes with cross-factor codes;
\ebeyond{7}{om} mixes all factors linearly in an overcomplete space;
\ebeyond{8}{mo} distributes each factor across disjoint nonlinear codes.

\section{Metrics}
\label{apx:metrics}
See \cref{apx:derivations} for a more detailed description of each metric.
    
        \begin{table}[ht]
        \centering
        \setlength{\tabcolsep}{8pt}
        \renewcommand{\arraystretch}{1.25}
        \begin{tabular}{p{0.18\linewidth} p{0.72\linewidth}}
        \hline
        \textbf{Metric} & \textbf{Short description} \\
        \hline
        DCI & Measures \emph{Disentanglement} (D), \emph{Completeness} (C), and \emph{Informativeness} (I) by assessing how well latent dimensions predict ground-truth factors using supervised regressors \citep{eastwood2018framework}. \\
        MCC & Evaluates alignment between learned and ground-truth latent variables via an optimal one-to-one matching that maximizes pairwise correlations (P=pearson, S=spearman, RDC=Randomized Dependence Coefficient \citep{lopez2013randomized}). \\
        $R^2$ & Quantifies the proportion of variance in ground-truth factors explained by the learned representation through linear regression. \\
        T-MEX & Assesses disentanglement by measuring how selectively latent variables respond to interventions on ground-truth factors \citep{yao2025third}. \\
        MIG & Computes the gap between the top two mutual information scores between a factor and latent variables [CITE]. \\
        InfoMEC & Measures equivalence classes of representations by evaluating how much information about the ground-truth factors is preserved under invertible transformations \citep{hsu2023disentanglement} (M=modularity, E=explicitness, C=compactness). \\
        \hline
        \end{tabular}
        \caption{Common evaluation metrics for causal representation learning with ground-truth factors.}
        \label{tab:metrics_desc}
        \end{table}

\begin{table}[t]
\centering
\small
\setlength{\tabcolsep}{4pt}
\caption{\textbf{No single metric satisfies all four properties.}
Each cell shows whether a metric satisfies (\cmark), partially satisfies ($\sim$), or violates (\xmark) the corresponding property.
Superscripts reference the relevant subsection. See \cref{apx:expts} for MI-based metrics.}
\label{tab:metric-property-summary}

\begin{tabular}{lcccc}
\toprule
Metric
& P1
& P2
& P3
& P4 \\
\midrule
\mcc{}-P  & \xmark & \xmark & \xmark & \xmark \\
\mcc{}-S  & \xmark & \xmark & \xmark & \xmark \\
\rsq{}    & \cmark & $\sim$ & \xmark & \cmark \\
\dci{}-D  & $\sim$ & $\sim$ & \xmark & $\sim$ \\
MIG       & \xmark & \xmark & \xmark & $\sim$ \\
T-MEX     & $\sim$ & \xmark & \xmark & $\sim$ \\
\bottomrule
\end{tabular}

\vspace{2pt}
\raggedright
{\footnotesize
P1: $\rho$-invariance$^{\ref*{subsec:correlation}}$ \quad
P2: $d_{\mathrm{eff}}$-sensitivity$^{\ref*{subsec:undercomplete}}$ \\
P3: OC-invariance$^{\ref*{subsec:overcomplete}}$ \quad
P4: Uninformative-sensitivity$^{\ref*{subsec:false-positive}}$
}
\end{table}

\section{Expected Metrics' Behaviour: Theory and Derivations}
\label{apx:derivations}

\subsection{MCC: Correlated latent factors and linear entanglement}
\label{apx:theory-correlation-e3}

We consider three ground-truth latent variables
\[
\z = (Z_1, Z_2, Z_3)^\top,
\]
with the following second-order structure:
\begin{align}
\mathbb{E}[Z_i] &= 0 \quad \text{for all } i, \\
\Var(Z_1) &= 1, \\
\Var(Z_2) = \Var(Z_3) &= 1, \\
\Corr(Z_2, Z_3) &= \rho, \quad |\rho| < 1,
\end{align}
and $Z_1$ uncorrelated with $(Z_2, Z_3)$.

Consider a learned representation with linear mixing of the form:
\begin{align}
\hat Z_1 &= s Z_1, \\
\hat Z_2 &= a Z_2 + b Z_3, \\
\hat Z_3 &= c Z_2 + d Z_3,
\end{align}
with $s,a,b,c,d \neq 0$.
Our goal is to compute the Mean Correlation Coefficient (MCC) between
$\Z$ and $\hat \Z := (\hat Z_1, \hat Z_2, \hat Z_3)$ as a function of the latent correlation $\rho$.

We start by computing the covariances between the true latents and the learned coordinates.
For $Z_1$ we immediately have
\begin{align}
\Cov(Z_1, \hat Z_1) &= s \Var(Z_1) = s, \\
\Cov(Z_1, \hat Z_2) &= 0, \\
\Cov(Z_1, \hat Z_3) &= 0,
\end{align}
since $Z_1$ is uncorrelated with $Z_2$ and $Z_3$, and $\Var(Z_1)=1$.

For $Z_2$ and $\hat Z_2 = a Z_2 + b Z_3$,
\begin{align}
\Cov(Z_2, \hat Z_2)
&= \Cov\bigl(Z_2, a Z_2 + b Z_3\bigr) \\
&= a \Cov(Z_2, Z_2) + b \Cov(Z_2, Z_3) \\
&= a \Var(Z_2) + b \rho \sqrt{\Var(Z_2)\Var(Z_3)} \\
&= a + b \rho,
\end{align}
using $\Var(Z_2) = \Var(Z_3) = 1$ and $\Cov(Z_2,Z_3) = \rho$.

Similarly, the variance of $\hat Z_2$ is
\begin{align}
\Var(\hat Z_2)
&= \Var(aZ_2 + bZ_3) \\
&= a^2 \Var(Z_2) + b^2 \Var(Z_3) + 2ab \Cov(Z_2,Z_3) \\
&= a^2 + b^2 + 2ab\rho.
\end{align}

Therefore
\begin{equation}
\Corr(Z_2, \hat Z_2)
= \frac{\Cov(Z_2,\hat Z_2)}{\sqrt{\Var(Z_2)\Var(\hat Z_2)}}
= \frac{a + b\rho}{\sqrt{a^2 + b^2 + 2ab\rho}}.
\label{eq:r22}
\end{equation}

Analogously,
\begin{align}
\Cov(Z_3, \hat Z_2)
&= a \Cov(Z_3,Z_2) + b \Cov(Z_3,Z_3)
= a\rho + b, \\
\Corr(Z_3, \hat Z_2)
&= \frac{b + a\rho}{\sqrt{a^2 + b^2 + 2ab\rho}}.
\label{eq:r32}
\end{align}

Repeating the same computation for $\hat Z_3 = c Z_2 + d Z_3$ yields
\begin{align}
\Corr(Z_2, \hat Z_3)
&= \frac{c + d\rho}{\sqrt{c^2 + d^2 + 2cd\rho}},
\label{eq:r23} \\
\Corr(Z_3, \hat Z_3)
&= \frac{d + c\rho}{\sqrt{c^2 + d^2 + 2cd\rho}}.
\label{eq:r33}
\end{align}

Finally, for $Z_1$ we have
\begin{equation}
\Corr(Z_1, \hat Z_1) = \sgn(s), \qquad
\Corr(Z_1, \hat Z_2) = \Corr(Z_1, \hat Z_3) = 0.
\end{equation}

\subsubsection{Correlation matrix and MCC}

Collecting the correlations into a matrix $C(\rho)$, we obtain
\begin{flalign}
C(\rho) =
\begin{pmatrix}
\Corr(Z_1,\hat Z_1) & \Corr(Z_1,\hat Z_2) & \Corr(Z_1,\hat Z_3) \\
\Corr(Z_2,\hat Z_1) & \Corr(Z_2,\hat Z_2) & \Corr(Z_2,\hat Z_3) \\
\Corr(Z_3,\hat Z_1) & \Corr(Z_3,\hat Z_2) & \Corr(Z_3,\hat Z_3)
\end{pmatrix} \\
=
\begin{pmatrix}
\pm 1 & 0 & 0 \\
0 & r_{22}(\rho) & r_{23}(\rho) \\
0 & r_{32}(\rho) & r_{33}(\rho)
\end{pmatrix},
\end{flalign}
where $r_{22},r_{23},r_{32},r_{33}$ are given by
\eqref{eq:r22}--\eqref{eq:r33}.

The Mean Correlation Coefficient (MCC) between $\z$ and $\hat \z$ for the
$3\times 3$ case is defined as
\begin{equation}
\mathrm{MCC}(\rho)
= \frac{1}{3} \max_{\pi \in S_3}
\sum_{i=1}^3 \bigl|\Corr(Z_i, \hat Z_{\pi(i)})\bigr|,
\label{eq:mcc-def}
\end{equation}
where $S_3$ is the set of all permutations of $\{1,2,3\}$.

Since $\Corr(Z_1,\hat Z_1) = \pm 1$ and $\Corr(Z_1,\hat Z_2) = \Corr(Z_1,\hat Z_3) = 0$,
the optimal permutation always pairs $Z_1$ with $\hat Z_1$, contributing $1$ to the sum.
The remaining degrees of freedom are in how we pair $(Z_2,Z_3)$ with
$(\hat Z_2,\hat Z_3)$.

There are two relevant pairings:
\begin{enumerate}
\item ``Diagonal'' pairing:
      $Z_2 \leftrightarrow \hat Z_2$ and $Z_3 \leftrightarrow \hat Z_3$,
      giving a sum
      \[
      S_{\mathrm{diag}}(\rho)
      = |r_{22}(\rho)| + |r_{33}(\rho)|.
      \]
\item ``Swapped'' pairing:
      $Z_2 \leftrightarrow \hat Z_3$ and $Z_3 \leftrightarrow \hat Z_2$,
      giving
      \[
      S_{\mathrm{swap}}(\rho)
      = |r_{23}(\rho)| + |r_{32}(\rho)|.
      \]
\end{enumerate}

Thus
\begin{equation}
\mathrm{MCC}(\rho)
= \frac{1}{3} \Bigl(
1 + \max\{ S_{\mathrm{diag}}(\rho),\; S_{\mathrm{swap}}(\rho) \}
\Bigr).
\label{eq:mcc-rho-general}
\end{equation}
The dependence of MCC on the latent correlation $\rho$ is therefore entirely
through the functions $r_{22}(\rho),\ldots,r_{33}(\rho)$.

\paragraph{Effect of the sign of \texorpdfstring{$\rho$}{rho} in a symmetric example.}

To see how the sign of $\rho$ affects
$\mathrm{MCC}(\rho)$, consider a symmetric mixing:
\begin{align}
\hat Z_2 &= Z_2 + \varepsilon Z_3, \\
\hat Z_3 &= \varepsilon Z_2 + Z_3,
\end{align}
with $0 < \varepsilon < 1$. In this case
\[
a = d = 1, \qquad b = c = \varepsilon,
\]
and substituting into \eqref{eq:r22}--\eqref{eq:r33} yields
\begin{align}
r_{22}(\rho)
&= \frac{1 + \varepsilon \rho}{\sqrt{1 + \varepsilon^2 + 2\varepsilon\rho}},
\label{eq:r22-sym}\\
r_{33}(\rho)
&= \frac{1 + \varepsilon\rho}{\sqrt{1 + \varepsilon^2 + 2\varepsilon\rho}}
= r_{22}(\rho),
\label{eq:r33-sym}
\end{align}
and the off-diagonal entries are
\begin{align}
r_{32}(\rho)
&= \frac{\varepsilon + \rho}{\sqrt{1 + \varepsilon^2 + 2\varepsilon\rho}},
\label{eq:r32-sym}\\
r_{23}(\rho)
&= \frac{\varepsilon + \rho}{\sqrt{\varepsilon^2 + 1 + 2\varepsilon\rho}}
= r_{32}(\rho).
\label{eq:r23-sym}
\end{align}
To verify: from \eqref{eq:r33} with $d=1$, $c=\varepsilon$, the numerator is
$d + c\rho = 1 + \varepsilon\rho$, matching $r_{22}$.
From \eqref{eq:r32} with $a=1$, $b=\varepsilon$, the numerator is
$b + a\rho = \varepsilon + \rho$.
From \eqref{eq:r23} with $c=\varepsilon$, $d=1$, the numerator is
$c + d\rho = \varepsilon + \rho$, matching $r_{32}$.
All four denominators equal $\sqrt{1 + \varepsilon^2 + 2\varepsilon\rho}$.

The two pairings therefore give
\begin{align}
S_{\mathrm{diag}}(\rho)
&= |r_{22}(\rho)| + |r_{33}(\rho)|
= 2\,|r_{22}(\rho)|, \label{eq:sdiag-sym}\\
S_{\mathrm{swap}}(\rho)
&= |r_{23}(\rho)| + |r_{32}(\rho)|
= 2\,|r_{32}(\rho)|. \label{eq:sswap-sym}
\end{align}
These are \emph{not} equal in general.
We now show that the diagonal pairing always dominates.
Consider the difference of the numerators:
\begin{align}
(1 + \varepsilon\rho) - (\varepsilon + \rho)
&= (1 - \varepsilon)(1 - \rho) \ge 0,
\label{eq:num-diff-pos}\\
(1 + \varepsilon\rho) + (\varepsilon + \rho)
&= (1 + \varepsilon)(1 + \rho) > 0,
\label{eq:num-sum-pos}
\end{align}
for all $\rho \in (-1,1)$ and $\varepsilon \in (0,1)$.
Since both share the same (positive) denominator, we have
$|r_{22}(\rho)| \ge |r_{32}(\rho)|$ with equality only at the boundary
$\rho = 1$ or $\varepsilon = 1$.
Hence $S_{\mathrm{diag}}(\rho) \ge S_{\mathrm{swap}}(\rho)$, and
\begin{equation}
\mathrm{MCC}(\rho)
= \frac{1}{3}\Bigl(1 + 2\,|r_{22}(\rho)|\Bigr)
= \frac{1}{3}\left(1 +
  \frac{2\,|1 + \varepsilon\rho|}{\sqrt{1 + \varepsilon^2 + 2\varepsilon\rho}}\right).
\label{eq:mcc-sym}
\end{equation}
Since $1 + \varepsilon\rho \ge 1 - \varepsilon > 0$ for all $|\rho| < 1$,
the absolute value is redundant and we may write
\begin{equation}
\mathrm{MCC}(\rho)
= \frac{1}{3}\left(1 +
  \frac{2(1 + \varepsilon\rho)}{\sqrt{1 + \varepsilon^2 + 2\varepsilon\rho}}\right).
\label{eq:mcc-sym-clean}
\end{equation}

\paragraph{Derivative of $r_{22}(\rho)$.}
We now compute the derivative of $r_{22}(\rho)$ with respect to $\rho$.
Write $N(\rho) := 1 + \varepsilon\rho$ and
$D(\rho) := 1 + \varepsilon^2 + 2\varepsilon\rho$, so that
$r_{22} = N / \sqrt{D}$.  Then
\begin{align}
\frac{\partial r_{22}}{\partial \rho}
&= \frac{N'\,\sqrt{D} - N \cdot \frac{D'}{2\sqrt{D}}}{D}
= \frac{N'\,D - N \cdot \frac{D'}{2}}{D^{3/2}}.
\label{eq:dr22-setup}
\end{align}
We have $N' = \varepsilon$ and $D' = 2\varepsilon$, so the numerator is
\begin{align}
N'\,D - N \cdot \tfrac{D'}{2}
&= \varepsilon\,(1 + \varepsilon^2 + 2\varepsilon\rho)
 - (1 + \varepsilon\rho)\cdot\varepsilon \notag\\
&= \varepsilon\bigl[(1 + \varepsilon^2 + 2\varepsilon\rho)
                    - (1 + \varepsilon\rho)\bigr] \notag\\
&= \varepsilon\bigl(\varepsilon^2 + \varepsilon\rho\bigr)
 = \varepsilon^2(\varepsilon + \rho).
\label{eq:dr22-numerator}
\end{align}
Therefore
\begin{equation}
\frac{\partial r_{22}}{\partial \rho}
= \frac{\varepsilon^2(\varepsilon + \rho)}
       {\bigl(1 + \varepsilon^2 + 2\varepsilon\rho\bigr)^{3/2}}.
\label{eq:dr22-drho}
\end{equation}

The denominator in \eqref{eq:dr22-drho} is strictly positive for all
$\rho \in (-1,1)$ and $0<\varepsilon<1$, since
\[
1 + \varepsilon^2 + 2\varepsilon\rho
\ge 1 + \varepsilon^2 - 2\varepsilon
= (1 - \varepsilon)^2 > 0.
\]

\paragraph{Non-monotonicity of $r_{22}(\rho)$.}
The numerator in \eqref{eq:dr22-drho} changes sign at
$\rho^* = -\varepsilon$:
\begin{equation}
\frac{\partial r_{22}}{\partial \rho}
\begin{cases}
< 0 & \text{if } \rho < -\varepsilon, \\
= 0 & \text{if } \rho = -\varepsilon, \\
> 0 & \text{if } \rho > -\varepsilon.
\end{cases}
\label{eq:monotonicity-r22}
\end{equation}
Thus $r_{22}(\rho)$ is \emph{not} monotonically increasing on $(-1,1)$.
It attains its minimum at $\rho^* = -\varepsilon$, where
\begin{equation}
r_{22}(-\varepsilon)
= \frac{1 - \varepsilon^2}{\sqrt{1 + \varepsilon^2 - 2\varepsilon^2}}
= \frac{1 - \varepsilon^2}{\sqrt{1 - \varepsilon^2}}
= \sqrt{1 - \varepsilon^2}.
\label{eq:r22-min}
\end{equation}

\paragraph{Monotonicity of MCC.}
Since $1 + \varepsilon\rho > 0$ on $(-1,1)$, we have $|r_{22}| = r_{22}$, and
by \eqref{eq:mcc-sym-clean} the MCC inherits the same monotonicity
structure: decreasing on $(-1, -\varepsilon)$ and increasing on
$(-\varepsilon, 1)$, with minimum
\begin{equation}
\mathrm{MCC}(-\varepsilon)
= \frac{1}{3}\bigl(1 + 2\sqrt{1-\varepsilon^2}\,\bigr).
\label{eq:mcc-min}
\end{equation}
At the boundary values:
\begin{align}
\lim_{\rho \to 1}\, r_{22}(\rho)
&= \frac{1 + \varepsilon}{\sqrt{(1 + \varepsilon)^2}}
= 1,
\label{eq:r22-rho1}\\
\lim_{\rho \to -1^+}\, r_{22}(\rho)
&= \frac{1 - \varepsilon}{\sqrt{(1 - \varepsilon)^2}}
= 1.
\label{eq:r22-rho-1}
\end{align}
Hence $r_{22}(\rho) \to 1$ at \emph{both} extremes $\rho \to \pm 1$, and
$\mathrm{MCC}(\rho) \to 1$ in both limits.  The minimum of MCC is in the
interior, at $\rho = -\varepsilon$.

\begin{figure}
    \centering
    \includegraphics[width=\linewidth]{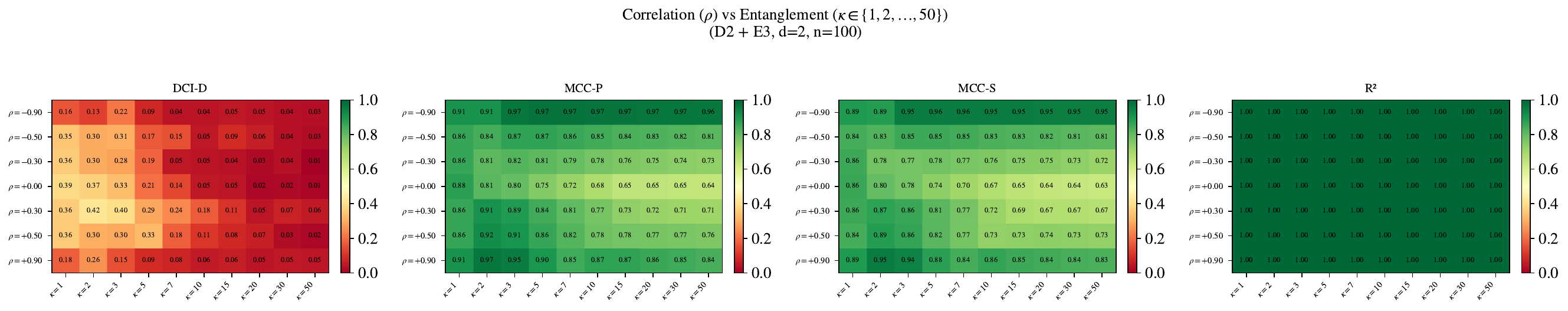}
    \caption{\mcc{} overestimates identifiability as $\rho$ increases.}
    \label{fig:apx-exp04-small}
\end{figure}

\paragraph{Implication.}
Even though $\hat Z_2$ and $\hat Z_3$ remain linearly entangled mixtures of
$Z_2$ and $Z_3$ for all $\varepsilon>0$, the MCC score varies with the
correlation $\rho$ between the ground-truth factors, despite the underlying
entanglement structure of the learned representation being unchanged.

For the practically relevant regime $\rho > 0$, positive correlations
inflate MCC monotonically: the MCC is strictly increasing on
$(0, 1)$ since $0 > -\varepsilon = \rho^*$.  For negative correlations,
the MCC first decreases (reaching its minimum at $\rho = -\varepsilon$)
and then increases again toward~$1$ as $\rho \to -1$.

The non-trivial dependence on $\rho$---including the fact that the minimum
is in the interior and that the MCC approaches~$1$ at \emph{both} boundary
values $\rho \to \pm 1$---demonstrates that MCC conflates representation
quality with the covariance structure of the ground-truth factors.

\bigskip

To illustrate this further, consider an even simpler (degenerate) example:

\begin{equation}
a = b = c = d = 1,
\end{equation}
so that
\begin{equation}
\hat Z_2 = Z_2 + Z_3, \qquad
\hat Z_3 = Z_2 + Z_3.
\end{equation}
In other words, $\hat Z_2$ and $\hat Z_3$ are identical, fully redundant, and both are symmetric mixtures of $Z_2$ and $Z_3$.

Using the same covariance structure as before, with
\[
\Var(Z_2) = \Var(Z_3) = 1,\qquad \Cov(Z_2,Z_3) = \rho,
\]
we compute
\begin{align*}
\Cov(Z_2,\hat Z_2)
&= \Cov\bigl(Z_2, Z_2 + Z_3\bigr)
= \Var(Z_2) + \Cov(Z_2,Z_3)
= 1 + \rho, \\
\Var(\hat Z_2)
&= \Var(Z_2 + Z_3)
= \Var(Z_2) + \Var(Z_3) + 2\Cov(Z_2,Z_3)
= 2 + 2\rho.
\end{align*}
Hence, for $\rho > -1$,
\begin{equation}
\Corr(Z_2,\hat Z_2)
= \frac{1+\rho}{\sqrt{2+2\rho}}
= \sqrt{\frac{1+\rho}{2}}
=: r(\rho).
\label{eq:r-degenerate}
\end{equation}
\emph{Verification via the general formula.}
Setting $\varepsilon = 1$ in \eqref{eq:r22-sym} gives
$r_{22} = (1+\rho)/\sqrt{2 + 2\rho}$, matching \eqref{eq:r-degenerate}.

By symmetry we also have
\begin{align*}
\Corr(Z_3,\hat Z_2)
= \Corr(Z_2,\hat Z_2)
= \Corr(Z_2,\hat Z_3)
= \Corr(Z_3,\hat Z_3)
= r(\rho),
\end{align*}
and $\Corr(Z_1,\hat Z_1)=\pm 1$, $\Corr(Z_1,\hat Z_2)=\Corr(Z_1,\hat Z_3)=0$ as before.

In this case, the two candidate pairings (diagonal and swapped) give the same sum, and the MCC simplifies to
\begin{equation}
\mathrm{MCC}(\rho)
= \frac{1}{3}\bigl(1 + 2\,r(\rho)\bigr)
= \frac{1}{3}\Bigl(1 + 2\sqrt{\tfrac{1+\rho}{2}}\Bigr),
\qquad -1 < \rho \le 1.
\label{eq:mcc-degenerate}
\end{equation}

Note that this degenerate case corresponds to $\varepsilon = 1$, where
the minimum of $r_{22}$ from \eqref{eq:monotonicity-r22} occurs at
$\rho^* = -1$ (the boundary), consistent with $r(\rho)$ being
monotonically increasing on $(-1,1)$.

This expression makes two important properties explicit.

\paragraph{(i) Asymmetry \texorpdfstring{$\mathrm{MCC}(\rho) \neq \mathrm{MCC}(-\rho)$}{MCC(rho) != MCC(-rho)}.}
From \eqref{eq:mcc-degenerate} we obtain
\begin{equation}
r(\rho) = \sqrt{\frac{1+\rho}{2}},
\quad -1 < \rho \le 1.
\label{eq:mcc-rho-def}
\end{equation}
Substituting $-\rho$ in place of $\rho$ yields
\begin{equation}
\mathrm{MCC}(-\rho)
= \frac{1}{3}\left(1 + 2\sqrt{\frac{1-\rho}{2}}\right).
\label{eq:mcc-neg-rho}
\end{equation}
Hence, for any $\rho \in (0,1)$,
\begin{equation*}
\mathrm{MCC}(\rho)
= \frac{1}{3}\left(1 + 2\sqrt{\frac{1+\rho}{2}}\right)
\neq
\frac{1}{3}\left(1 + 2\sqrt{\frac{1-\rho}{2}}\right)
= \mathrm{MCC}(-\rho).
\end{equation*}
Even though the entanglement structure for $(Z_2,Z_3)$ is symmetric under the sign flip $\rho \mapsto -\rho$, MCC values differ for positive and negative correlations.

\paragraph{(ii) Faithfulness issues at extreme correlations.}
The same formula reveals a qualitative difference between the limits $\rho \to 1$ and $\rho \to -1$:
\begin{align}
\lim_{\rho \to 1} r(\rho)
&= \sqrt{\frac{1+1}{2}} = 1, \\
\lim_{\rho \to -1^+} r(\rho)
&= \sqrt{\frac{1+(-1)}{2}} = 0.
\end{align}
When $\rho \to 1$, \cref{eq:mcc-degenerate} yields
$\mathrm{MCC}(\rho) \to (1+2)/3 = 1$.
In contrast, when $\rho \to -1$,
$\mathrm{MCC}(\rho) \to (1+0)/3 \approx 0.33$.

\paragraph{Contrast with the $\varepsilon < 1$ case.}
In the general symmetric mixing with $\varepsilon < 1$,
eqs.~\eqref{eq:r22-rho1}--\eqref{eq:r22-rho-1} show that
$r_{22}(\rho) \to 1$ at \emph{both} $\rho \to +1$ and $\rho \to -1$,
so $\mathrm{MCC}(\rho) \to 1$ in both limits.
The faithfulness collapse ($\mathrm{MCC} \to 0.33$) at $\rho \to -1$ is
specific to the degenerate case $\varepsilon = 1$, where
$\hat Z_2 = \hat Z_3 = Z_2 + Z_3 \approx 0$ when $Z_3 \approx -Z_2$.
For $\varepsilon < 1$, the representations $\hat Z_2 \neq \hat Z_3$ remain
distinct and their correlations with the true factors recover to~$1$ as
$\rho \to -1$.

\bigskip
\noindent
These examples show that MCC depends nontrivially on the covariance structure of the ground-truth factors, independently of the underlying entanglement of the learned representation.  In particular, for a fixed mixing matrix, MCC can be artificially inflated or deflated by the latent correlation $\rho$.

\subsection{Expected behaviour of DCI}
\label{apx:derivations-dci}

We derive properties of DCI that explain the metric's behaviour in
the main-text experiments.
We first recall the construction, then state four results organised
by failure mode.

\paragraph{Construction.}
A supervised probe is trained to predict each ground-truth factor
$z_j$ from the learned representation $\hat{\rvz}$, yielding a
nonnegative importance matrix
$\R \in \sR_{\ge 0}^{m \times d}$,
where $R_{i,j}$ quantifies the contribution of learned feature
$\hat z_i$ in predicting $z_j$.
Row $\R_{i,:}$ summarises which factors feature~$i$ encodes;
column $\R_{:,j}$ summarises which features encode factor~$j$.
The DCI scores are computed from $\R$ alone:

\emph{Disentanglement.}\;
Convert each row to a distribution
$p_{j \mid i} = R_{i,j} / \sum_{k} R_{i,k}$ and measure
concentration:
\[
  D_i = 1 - \frac{H(p_{\cdot \mid i})}{\log d},
  \qquad
  D_{\dci} = \sum_{i=1}^{m} w_i\, D_i,
  \qquad
  w_i = \frac{\textstyle\sum_{j} R_{i,j}}
             {\textstyle\sum_{i',j'} R_{i',j'}}.
\]

\emph{Completeness.}\;
Convert each column to a distribution
$\tilde p_{i \mid j} = R_{i,j} / \sum_{k} R_{k,j}$ and measure
concentration:
\[
  C_j = 1 - \frac{H(\tilde p_{\cdot \mid j})}{\log m},
  \qquad
  C_{\dci} = \sum_{j=1}^{d} v_j\, C_j,
  \qquad
  v_j = \frac{\textstyle\sum_{i} R_{i,j}}
             {\textstyle\sum_{i',j'} R_{i',j'}}.
\]

\emph{Informativeness.}\;
$I_{\dci} = \frac{1}{d} \sum_{j} (1 - L_j)$, where
$L_j$ is the normalised prediction loss (e.g., $1 - R^2_j$) of the
probe for factor~$j$.

\medskip\noindent
The weights $w_i$ and $v_j$ are proportional to total importance:
features or factors with negligible importance contribute negligibly
to the global scores.
This weighting is the source of the first failure mode.


\begin{proposition}[Dropped factors are invisible to \dci]
\label{prop:dci-missing}
Under \dcrl{1}~+~\ebeyond{4}{oo} with $|S| = m < d$ perfectly
identified factors, as $n \to \infty$:
$D_{\dci} \to 1$ and $C_{\dci} \to 1$.
\end{proposition}

\begin{proof}
For retained factors ($j \in S$), the encoder is elementwise, so the
probe importance concentrates on a single coordinate:
$p_{j \mid i}$ and $\tilde p_{i \mid j}$ are one-hot for the
matched pair, giving $D_i = 1$ and $C_j = 1$.

For discarded factors ($j \notin S$), no learned feature predicts
them: $R_{i,j} \approx 0$ for all~$i$.
Their weight $v_j \propto \sum_i R_{i,j} \approx 0$ vanishes from
$C_{\dci}$.
Likewise, the rows corresponding to codes that encode only retained
factors carry all the weight in $D_{\dci}$.

DCI does not verify that \emph{all} factors are represented; factors
that are never encoded produce zero importance, vanish from the
weighted averages, and do not penalise the score.
\end{proof}

\parless{Implication}: This explains the DCI-D false positive in \cref{fig:dropping_d1_d3} (left, \dcrl{1}): as factors are dropped, $D$ or $C$ do not penalise omission (only $I_{\dci}$ would drop).

\begin{proposition}[Functional dependence decreases $D$ under a perfect encoder]
\label{prop:dci-functional}
Under \dbeyond{3} with $z_2 = f(z_1)$ (deterministic) and a perfect elementwise encoder \ecrl{1}{oo} ($\hat z_j = a_j z_j$), a nonlinear probe (e.g., gradient boosted trees) yields $D_{\dci} < 1$.
\end{proposition}

\begin{proof}
Since $z_2 = f(z_1)$ exactly, $\hat z_1 = a_1 z_1$ perfectly
determines $z_2$ via $f$, so the nonlinear probe assigns $R_{1,2} > 0$
in addition to $R_{1,1} > 0$.
Symmetrically, when $f$ is invertible (e.g., $f(z) = z^3$),
$\hat z_2 = a_2 f(z_1)$ determines $z_1$ via $f^{-1}$, so
$R_{2,1} > 0$ and $D_2 < 1$.
The remaining $d - 2$ codes are independent and achieve $D_i = 1$,
but the deflated $D_1$ and $D_2$ pull down $D_{\dci}$ through their
nonzero weights $w_1, w_2 > 0$.
\end{proof}

\noindent
With a linear probe (e.g., Lasso), the result can differ.
For $z_1 \sim \cN(0,1)$ and $z_2 = z_1^3$, the population normal
equations yield zero cross-coefficients---$z_1$ and $z_1^3$ are
linearly orthogonal under Gaussian moments---so $\R$ is diagonal and
$D_{\dci} = 1$.
The deflation under \dbeyond{3} is therefore \emph{probe-dependent}:
it arises only when the probe is expressive enough to detect the
functional relationship $f$.

\parless{Implication.} This explains the \dci{}-D dip in the \dbeyond{3} panels of \cref{fig:dropping_d1_d3} (right): even at $m = d$ (dashed line), DCI-D is below~$1.0$ because the functional constraint between $z_1$ and $z_2$ spreads importance across the corresponding codes.

\subsection{MCC false-positive rate under null encoders}
\label{apx:mcc-false-positive}

We derive the expected behaviour of \mcc{}-P when the learned
representation is independent of the ground-truth factors,
explaining the inflation observed in \cref{fig:exp15-null}.

\paragraph{Setup.}
Let $\rvz \in \sR^d$ and $\hat{\rvz} \in \sR^m$ be independent
random vectors (null encoder), and let
$(z^{(1)}, \hat z^{(1)}), \dots, (z^{(n)}, \hat z^{(n)})$
be $n$ i.i.d.\ paired samples.
The sample Pearson correlation between $\hat z_i$ and $z_j$ is
\[
  \hat\rho_{ij}
  = \frac{\sum_{t=1}^{n}
    (\hat z_i^{(t)} - \bar{\hat z}_i)(z_j^{(t)} - \bar z_j)}
    {\sqrt{\sum_{t}(\hat z_i^{(t)} - \bar{\hat z}_i)^2}
     \;\sqrt{\sum_{t}(z_j^{(t)} - \bar z_j)^2}}.
\]
Since $\hat{\rvz} \perp \rvz$, the true correlation is
$\rho_{ij} = 0$ for all $i, j$.

\paragraph{Distribution of sample correlations under the null.}
For bivariate normal data with $\rho = 0$, the sample
correlation satisfies
\begin{equation}
  \frac{\hat\rho\,\sqrt{n-2}}{\sqrt{1 - \hat\rho^2}}
  \sim t_{n-2}
  \label{eq:fisher-t}
\end{equation}
exactly \citep{fisher1922mathematical}.
For non-Gaussian data, the exact $t$-distribution does not
hold, but the asymptotic result
$\sqrt{n}\,\hat\rho \xrightarrow{d} \cN(0, 1)$ follows from
the Central Limit Theorem (CLT) \citep{hoeffding1992class}.
In either case, for large~$n$,
\begin{equation}
  \hat\rho_{ij}
  \;\overset{\text{approx}}{\sim}\;
  \cN\!\left(0,\; \frac{1}{n}\right).
  \label{eq:rho-null}
\end{equation}

\paragraph{Maximum absolute correlation.}
\mcc{}-P computes the $m \times d$ matrix of absolute sample
correlations $|\hat\rho_{ij}|$ and applies Hungarian matching
to find the optimal one-to-one assignment.

Consider a single column~$j$.
The entries $\{|\hat\rho_{ij}|\}_{i=1}^{m}$ are approximately
half-normal with scale $1/\sqrt{n}$.
(They are not exactly independent---they share the $z_j^{(t)}$
samples---but the dependence is weak under the null since the
$\hat z_i$ are independent across rows; \citet{cai2011limiting}
handle this rigorously.)
The maximum of $m$ such entries satisfies, by standard extreme
value theory for Gaussian maxima,
\begin{equation}
  \mathbb{E}\!\left[\max_{i=1}^{m} |\hat\rho_{ij}|\right]
  \;\approx\;
  \sqrt{\frac{2 \log m}{n}}.
  \label{eq:max-corr}
\end{equation}

\paragraph{MCC-P under the null.}
To lower-bound the Hungarian matching, consider a greedy
assignment: assign column~$1$ its best row, remove that row,
assign column~$2$ its best among the remaining $m - 1$ rows,
and so on.
This produces a valid one-to-one assignment, and column~$j$
selects from $m - j + 1$ remaining candidates.
When $m \gg d$, every column still has ${\approx}\, m$
candidates, and the greedy score is close to the average
column-wise maximum.
Since the Hungarian matching is optimal over all one-to-one
assignments, it scores at least as high as the greedy, giving
\begin{equation}
  \mathbb{E}[\mcc\text{-P}]
  \;\gtrsim\;
  \frac{1}{d}\sum_{j=1}^{d}
  \sqrt{\frac{2\log(m - j + 1)}{n}}
  \;\approx\;
  \sqrt{\frac{2 \log m}{n}}
  \qquad\text{when } m \gg d.
  \label{eq:mcc-lb}
\end{equation}
This bound is non-negligible whenever $\log m / n$ is not
small.
In practice this inflation is substantial even at moderate
ratios:

\begin{center}
\small
\begin{tabular}{@{}lcc@{}}
\toprule
$m/n$ & $\sqrt{2\log m / n}$ (bound) &
\mcc{}-P (observed, $m/d{=}1$) \\
\midrule
$0.1$  & $0.21$ & $\sim 0.3$ \\
$0.2$  & $0.30$ & $\sim 0.5$ \\
$0.5$  & $0.48$ & $\sim 0.83$ \\
$1.0$  & $0.68$ & $\sim 0.95$ \\
\bottomrule
\end{tabular}
\end{center}

\noindent
The bound captures the correct scaling: it explains why $m/n$
governs the false-positive rate. But, it underestimates the
magnitude at practical sample sizes for two reasons:
(i)~the extreme value approximation is loose at small~$m$; and
(ii)~at small~$n$, the exact null distribution of $\hat\rho$
follows a scaled $t_{n-2}$ (\cref{eq:fisher-t}), which has
heavier tails than the Gaussian, pushing the maximum
correlation above the asymptotic prediction.

\paragraph{Extension to MCC-S.}
Spearman correlation is the Pearson correlation applied to
ranks.  Under independence, the sample Spearman correlation
also satisfies $\hat\rho^{S}_{ij} \approx \cN(0, 1/n)$
\citep{hotelling1936rank}, so the extreme value argument
above applies verbatim: the $\sqrt{2\log m / n}$ floor
governs both MCC-P and MCC-S. 

\paragraph{Why $m/n$ governs and $m/d$ does not.}
The bound \eqref{eq:max-corr} depends on $m$ (candidates per
column) and $n$ (sample size), but not on~$d$ (number of
columns).
Adding more ground-truth factors adds more columns to the
matching problem but does not change the distribution of each
column's maximum.
The MCC averaging divides by~$d$, but since each column
contributes approximately the same expected maximum, the
average is $\approx \sqrt{2\log m / n}$ regardless of~$d$.
This is consistent with the empirical observation in
\cref{fig:exp15-null}: reading along rows (fixed $m/d$,
varying $m/n$), scores increase; reading along columns (fixed
$m/n$, varying $m/d$), scores are approximately constant.

\paragraph{Comparison with \rsq{} and \dci{}-D.}
\rsq{} uses cross-validated nonlinear regression, which does
not exploit the maximum over candidates: it predicts each
factor independently and averages the explained variance.
Under the null, the cross-validated $R^2_j \approx 0$ for
each factor (overfitting is penalised by the held-out
evaluation), so $\rsq \approx 0$ regardless of $m/n$.

\dci{}-D trains a Lasso probe for each factor.
Under the null, the $\ell_1$ penalty shrinks most coefficients
to zero, but a few features can be spuriously selected---
particularly when $m$ is large relative to $n$
\citep{buhlmann2011statistics}.
The resulting importance matrix has most entries near zero;
the moderate inflation at high $m/n$ and low $m/d$ in
\cref{fig:exp15-null} is consistent with a small number of
spuriously selected features spreading enough importance mass
to inflate the disentanglement score.

\section{Experiments}
\label{apx:expts}

\subsection{Sanity Checks}

\begin{figure}
    \centering
    \includegraphics[width=\linewidth]{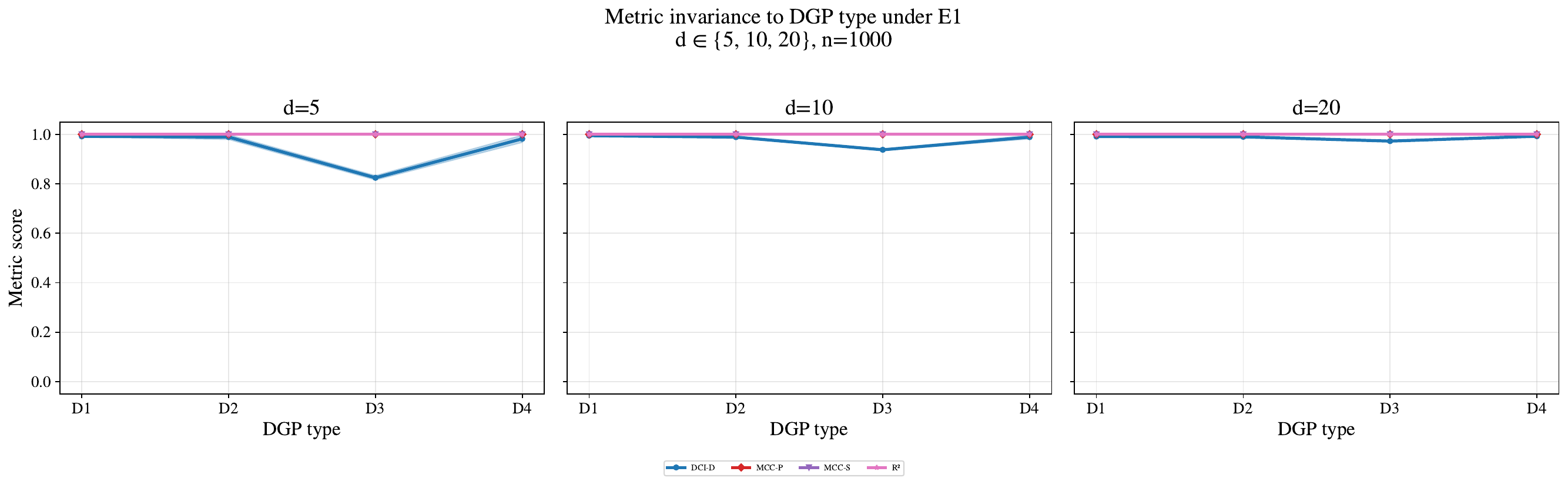}
    \caption{\dci{}-D is not stable for \dcrl{3}.}
    \label{fig:apx-exp01-main}
\end{figure}

\begin{figure}
    \centering
    \includegraphics[width=\linewidth]{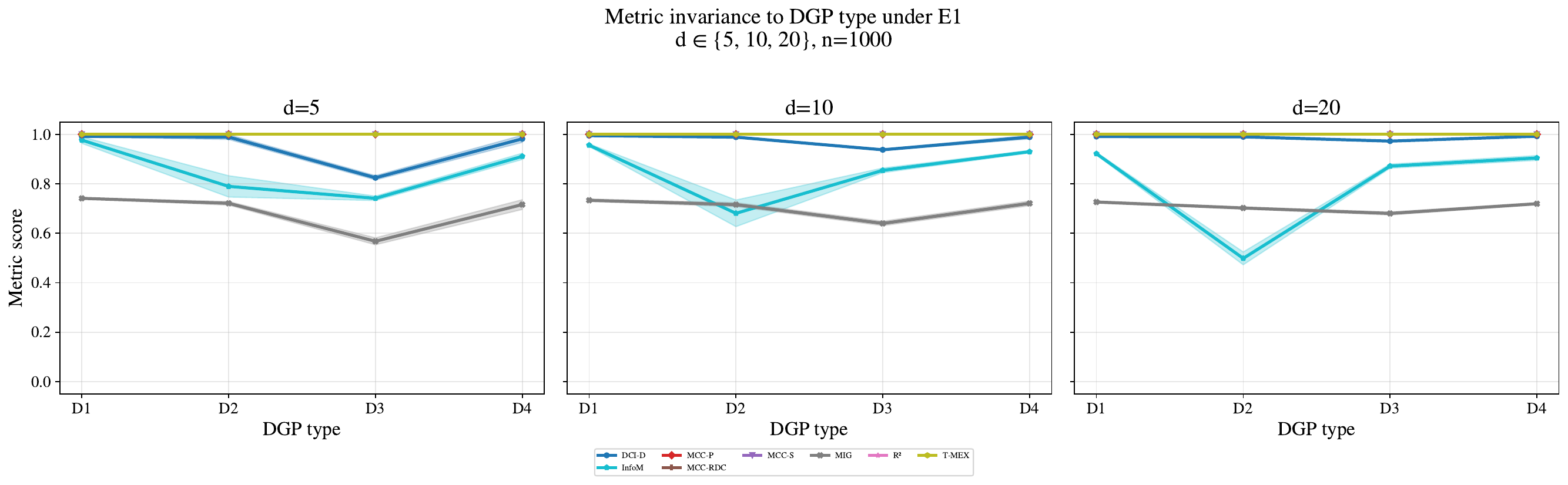}
    \caption{MIG reports $<1$ even for the ideal \dcrl{1}--\ecrl{1}{oo} case.}
    \label{fig:apx-exp01-apx}
\end{figure}

\begin{figure}
    \centering
    \includegraphics[width=0.6\linewidth]{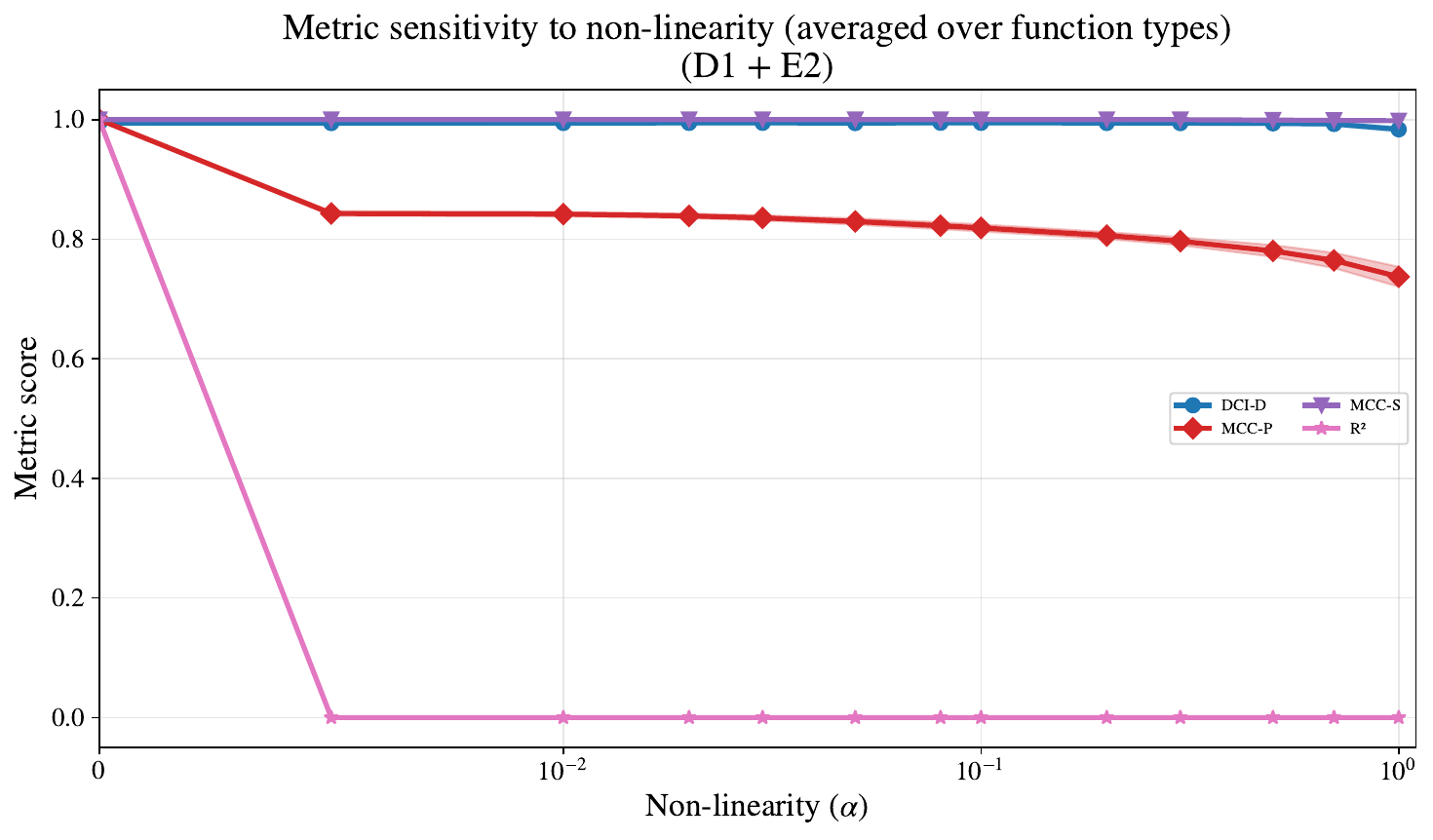}
    \caption{\dci{}-D and \mcc{}-S are quite stable against increasing non-linearity strength, hence reliable for evaluating \ecrl{2}{oo}.}
    \label{fig:apx-exp02-main}
\end{figure}

\begin{figure}
    \centering
    \includegraphics[width=0.6\linewidth]{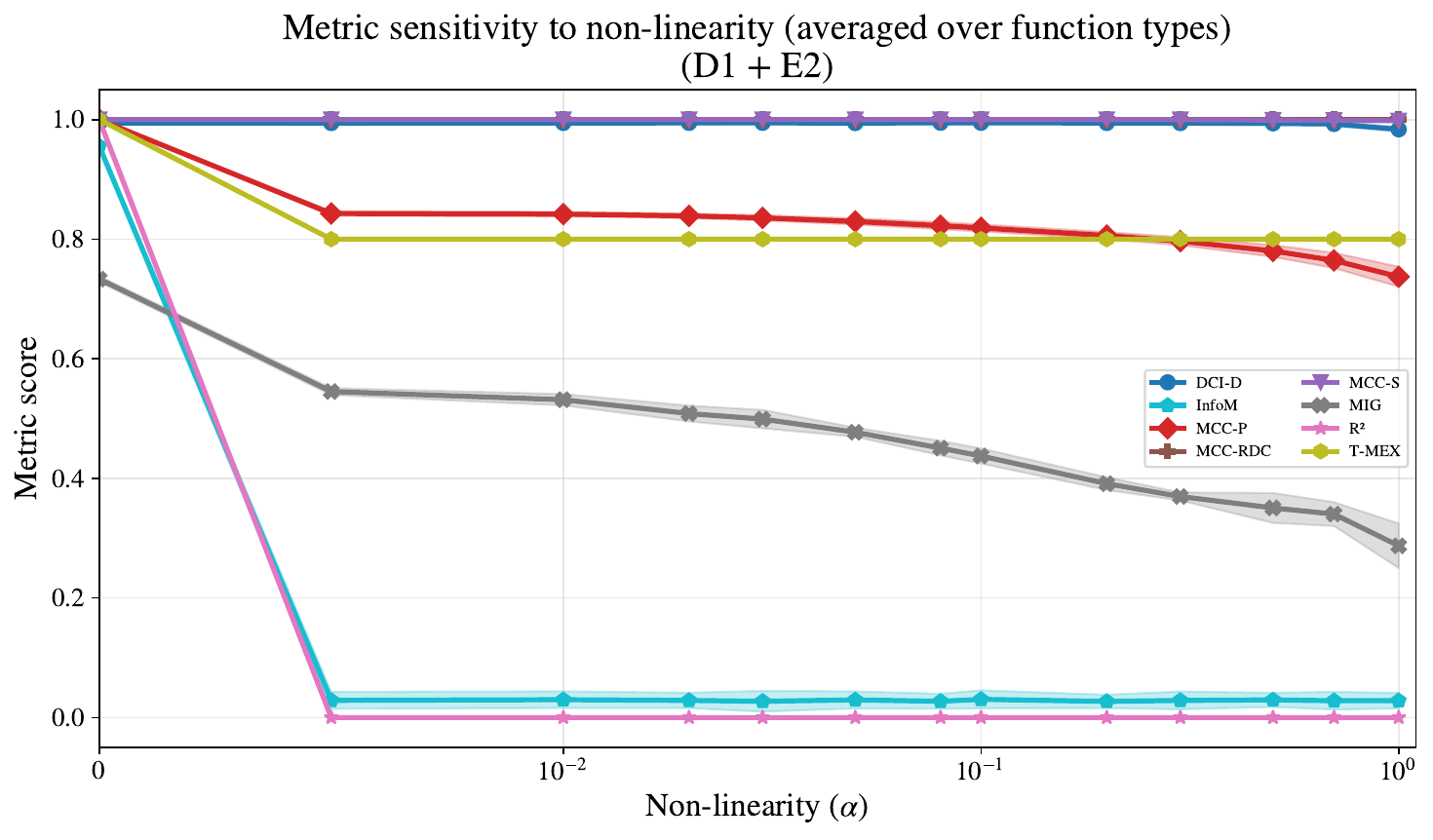}
    \caption{MI-based metrics struggle to evaluate \ecrl{2}{oo}.}
    \label{fig:apx-exp02-apx}
\end{figure}

\begin{figure}
    \centering
    \includegraphics[width=0.75\linewidth]{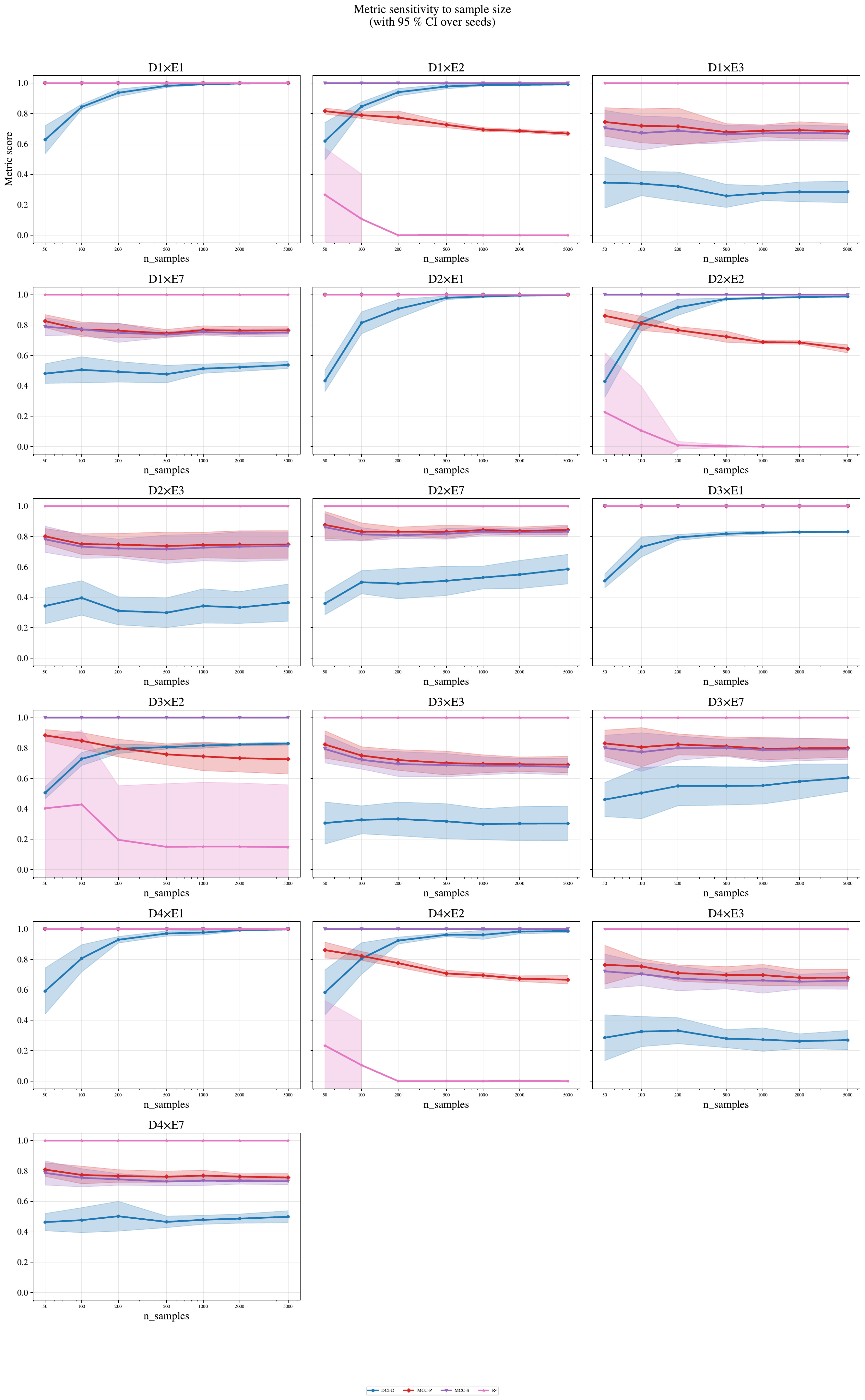}
    \caption{Sample sensitivity of four main metrics (DCI-D, MCC-P, MCC-S, R²) across all DGP ×
  encoder combinations ($n \in {50, \dots, 5000}$, $d{=}5$). MCC-P and MCC-S are
  sample-efficient: they stabilise by $n{=}100$ in most settings. DCI-D requires $n
  \gtrsim 500$ to converge, with wide confidence intervals at $n{<}200$, particularly
  under entangled encoders (E2, E3). The overcomplete encoder E7 ($m{=}2d{=}10$) inflates DCI-D variance at low $n$ because twice as many codes
  must be estimated from the same number of samples.}
    \label{fig:apx-exp10-main}
\end{figure}

\begin{figure}
    \centering
    \includegraphics[width=0.75\linewidth]{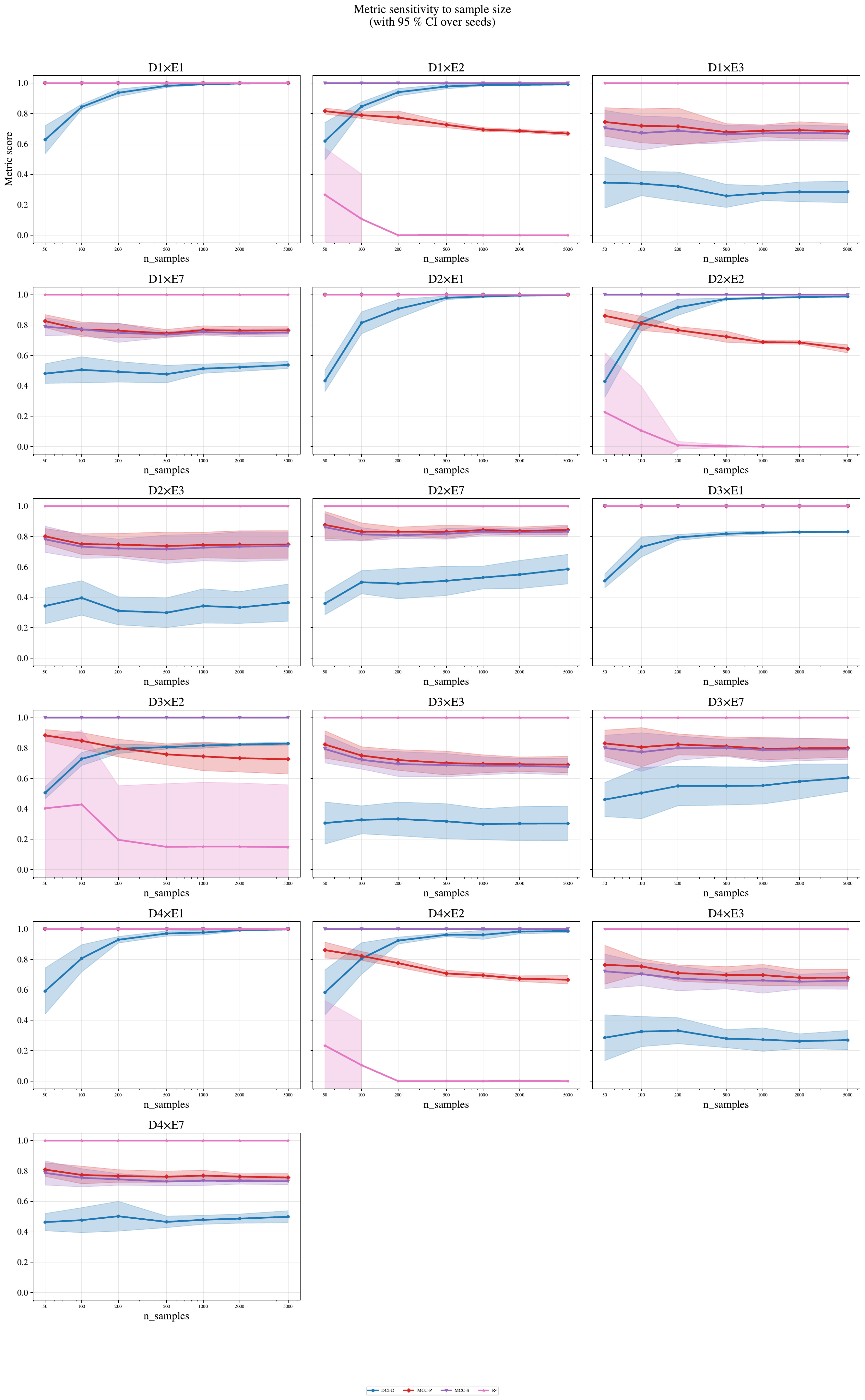}
    \caption{Full metric suite (adding InfoM, MIG, MCC-RDC, T-MEX to the main four). InfoM and MIG
   produce NaN at $n{=}50$ across all DGP × encoder combinations (visible as missing
  lines), making them unusable below $n \approx 100$. MCC-RDC converges slowly under
  nonlinear encoders (E2), lagging behind MCC-P and MCC-S until $n \gtrsim 1000$. T-MEX
   likewise returns NaN universally at small $n$. The variance heatmap confirms that R²
   under E2 is the single worst-case cell (std $\approx 0.35$ at $n{=}50$), while
  correlation-based metrics remain stable (std $< 0.01$). Overall, metrics split into
  two reliability tiers: correlation-based measures (MCC-P, MCC-S) are robust across
  sample sizes, while predictor-based (R², DCI-D) and information-theoretic (InfoM,
  MIG, T-MEX) metrics require $n \gtrsim 500$ for trustworthy estimates.}
    \label{fig:apx-exp10-apx}
\end{figure}

\subsection{Correlation among latent factors}

\begin{figure*}[t]
    \centering
    \includegraphics[width=\linewidth]{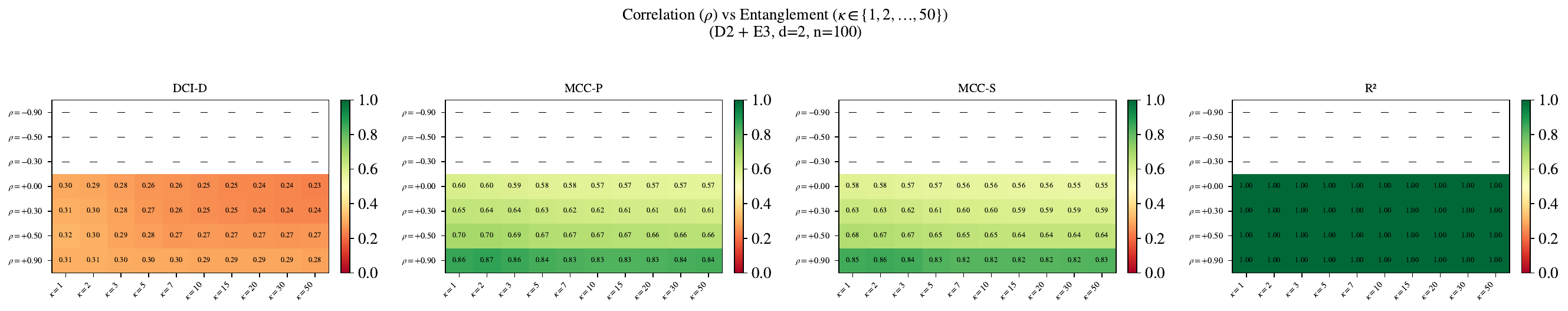}
    \caption{\textbf{Disentangling the effects of correlation~$\rho$
    and entanglement~$\kappa$ ($d{=}10$).}
    An ideal metric would vary only along columns (increasing $\kappa$, i.e.\ worse entanglement) and be constant along rows (changing $\rho$). \mcc{}-P and \mcc{}-S show clear row-wise gradients, confirming violation of \cref{prop:rho-invariance}.
    \dci{}-D is more stable but collapses to near-zero even under moderate entanglement (for all $\kappa > 1$).}
    \label{fig:exp04-rho-kappa}
\end{figure*}

\begin{figure}
    \centering
    \includegraphics[width=\linewidth]{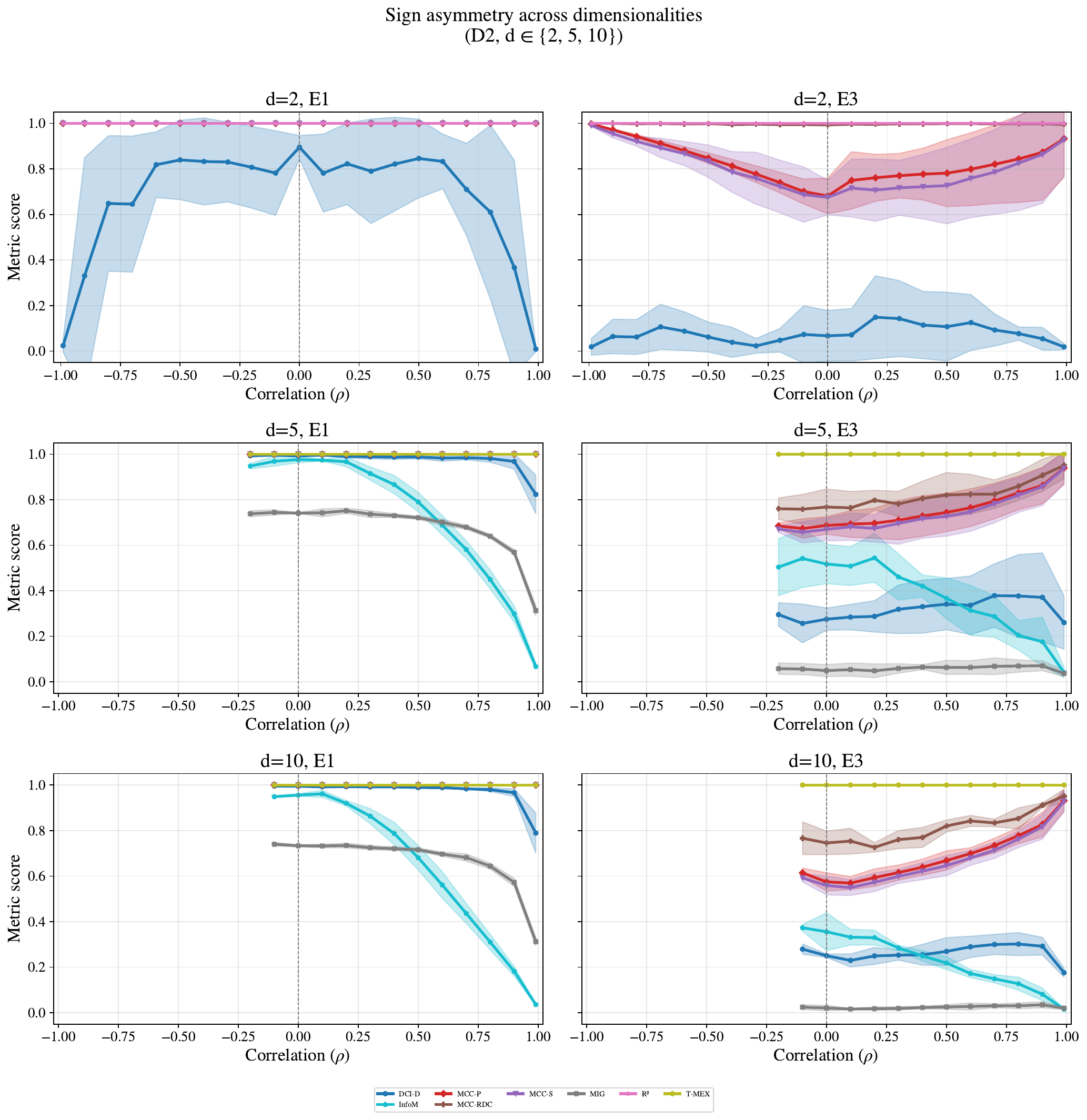}
    \caption{Full metric suite across dimensionalities. The additional metrics reveal two distinct
   behaviours. MIG (grey) and MCC-RDC (cyan) degrade sharply with increasing $|\rho|$
  under E1 at $d \geq 5$: MIG drops from ${\sim}0.8$ at $\rho{=}0$ to ${\sim}0.2$ at
  $\rho{=}0.99$, while MCC-RDC follows a similar decline, indicating that these metrics
   conflate inter-factor correlation with non-identifiability. InfoM and T-MEX are
  absent at $d{=}2$ (NaN) but appear at $d \geq 5$; T-MEX (yellow) remains flat near
  1.0 under E1 regardless of $\rho$, showing complete sign- and correlation-invariance.
   Under E3, the metric spread widens substantially: InfoM (dark blue) tracks DCI-D but
   at lower absolute values, while MCC-RDC drops more steeply than MCC-P, suggesting
  that the RDC kernel is sensitive to the interaction between entanglement and
  inter-factor correlation. The key takeaway is that correlation-based metrics (MCC-P,
  MCC-S) and T-MEX are robust to the sign of $\rho$, while tree-based (DCI-D) and
  MI-based (MIG) metrics are sensitive to it, particularly at low $d$.}
    \label{fig:apx-exp03-apx}
\end{figure}

\begin{figure}
    \centering
    \includegraphics[width=\linewidth]{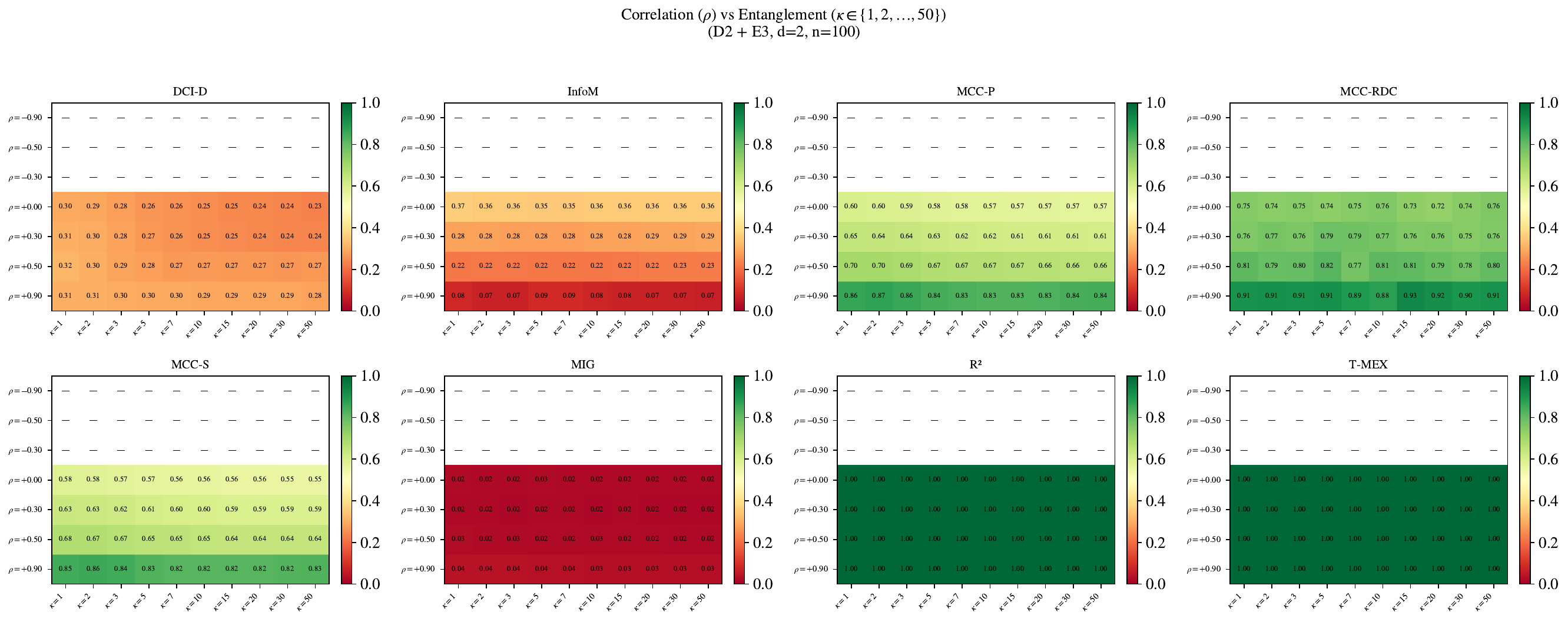}
    \caption{\textbf{Full metric suite: $\rho$--$\kappa$ heatmaps at $d{=}10$.} Extension of \cref{fig:exp04-rho-kappa} to all metrics. An ideal metric is constant along rows (varying~$\rho$ at fixed~$\kappa$); row-wise gradients indicate spurious sensitivity to inter-factor correlation. $d{=}10$, $n{=}1000$.}
    \label{fig:apx-exp04-apx}
\end{figure}

\begin{figure}
    \centering
    \includegraphics[width=\linewidth]{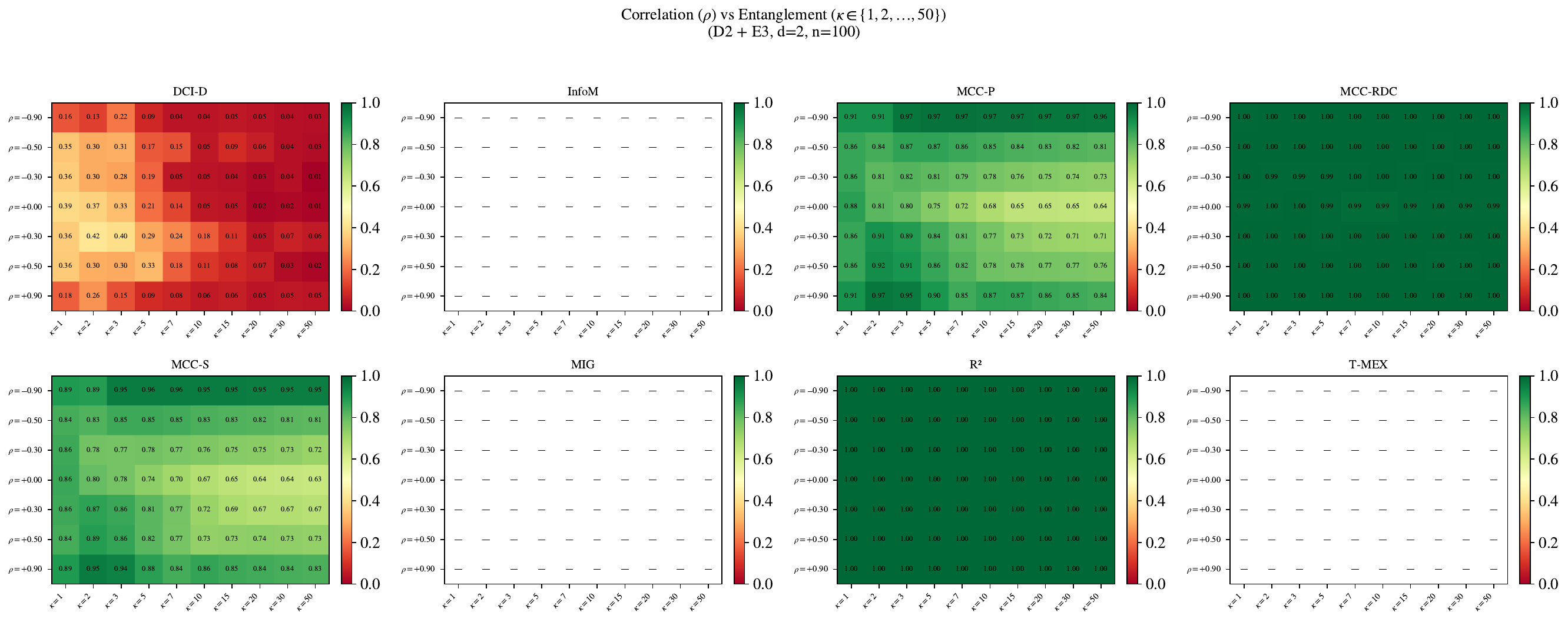}
    \caption{\textbf{Full metric suite: $\rho$--$\kappa$ heatmaps at $d{=}5$.} Extension of \cref{fig:apx-exp04-small} to all metrics. Same layout as \cref{fig:apx-exp04-apx}; qualitative conclusions carry over from $d{=}10$ to $d{=}5$. $n{=}1000$.}
    \label{fig:apx-exp04-small-apx}
\end{figure}

\begin{figure}
    \centering
    \includegraphics[width=\linewidth]{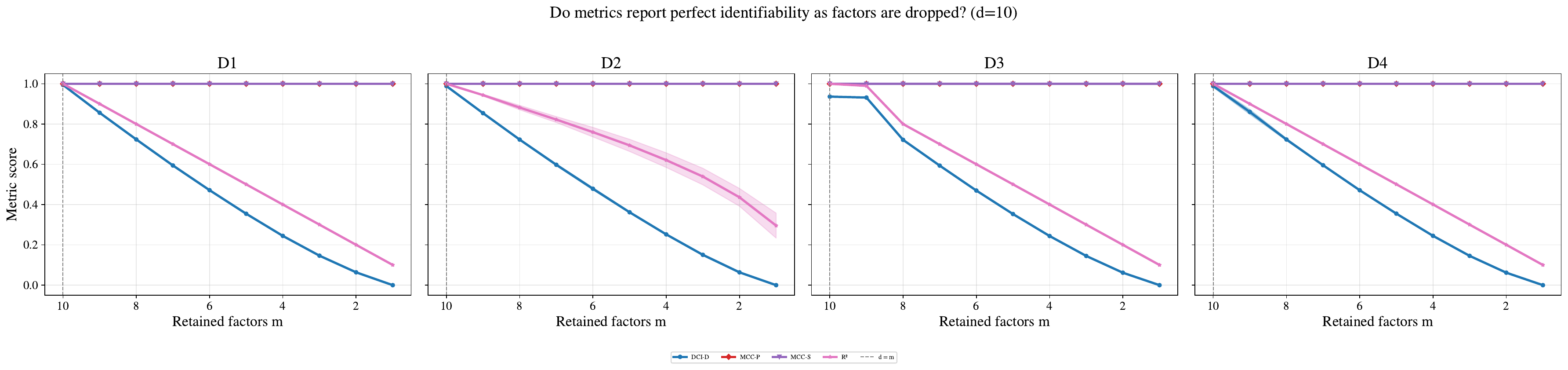}
    \caption{\textbf{Metric scores as a function of the number of retained factors across all DGP types.} Extension of \cref{fig:dropping_d1_d3} to \dcrl{2} and \dbeyond{4}. Under \dcrl{2}, \rsq{} exceeds $m/d$ because the probe partially predicts dropped factors from correlated retained ones. Under \dbeyond{4} ($z_k = g(z_i, z_j)$, $d_{\mathrm{eff}} = d{-}1$), metric behaviour is indistinguishable from \dcrl{1} despite the first omission being lossless: no metric detects the multi-factor redundancy. \mcc{}-P/S remain at $1.0$ across all panels. $d{=}10$, $n{=}1000$.}
    \label{fig:apx-exp06-all}
\end{figure}

\begin{figure}
    \centering
    \includegraphics[width=\linewidth]{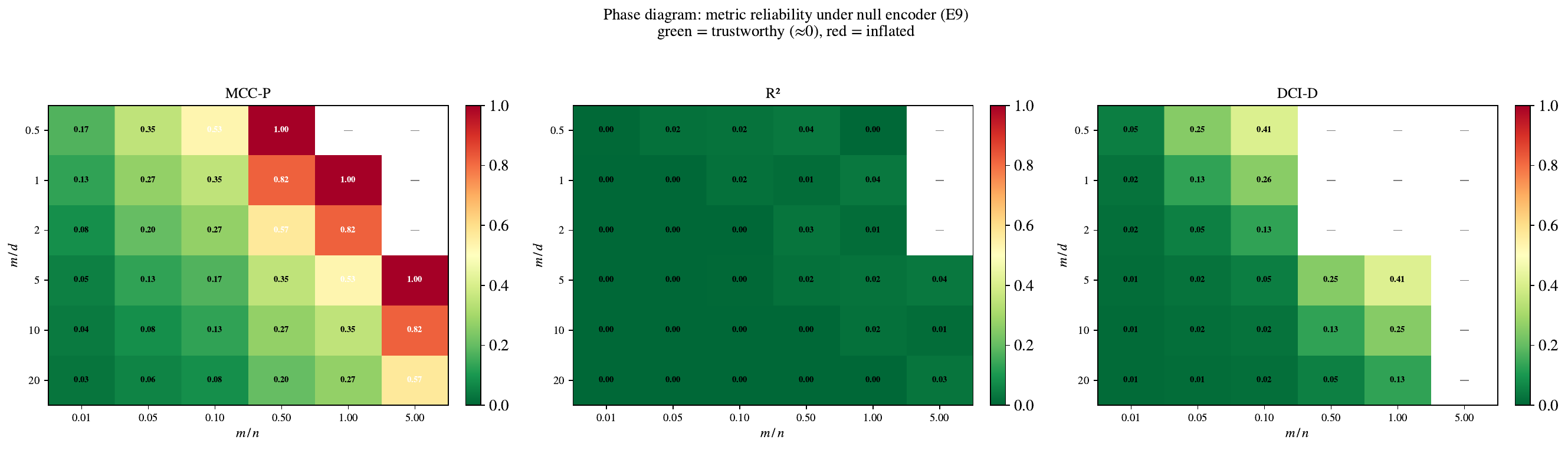}
    \caption{\textbf{False positives under a Gaussian null encoder.} Same layout as \cref{fig:exp15-null} (uniform null) but with $\hat{\rvz} \sim \cN(\mathbf{0}, \rmI_m)$. The false-positive pattern is nearly identical: \mcc{}-P/S scores are governed by $m/n$, not $m/d$, confirming that the inflation is independent of the null distribution.}
    \label{fig:apx-exp15-gauss}
\end{figure}

\begin{figure}[t]
    \centering
    \includegraphics[width=0.5\linewidth]{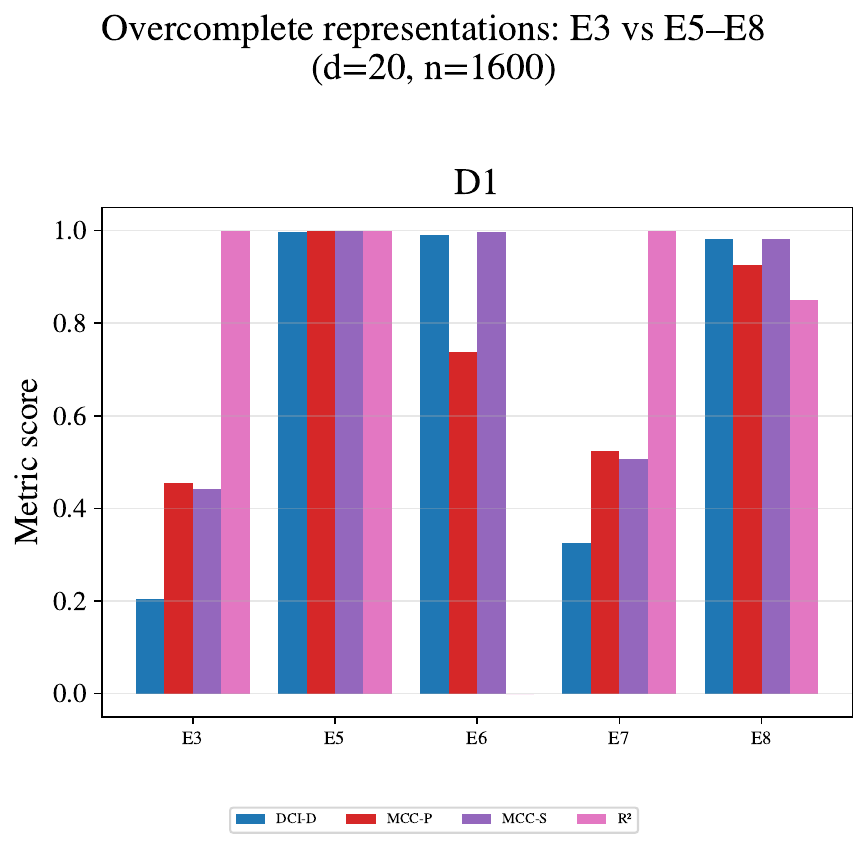}
    \caption{\textbf{At moderate overcompleteness ($m/d{=}2$),a
    metrics distinguish entanglement from redundancy.} Overcomplete disentangled encoders (\ebeyond{5}{mo}--\ebeyond{8}{mo}) score near $1.0$ on
    \dci{}-D, \mcc{}, and \rsq{}, whereas the entangled encoder \ebeyond{7}{mo} is correctly penalised by \dci{}-D and \mcc{}. $d{=}20$, $n{=}1600$.
    See \cref{fig:apx-exp09-d5} for all DGPs.}
    \label{fig:apx-exp09-bars}
\end{figure}

\begin{figure}
    \centering
    \includegraphics[width=\linewidth]{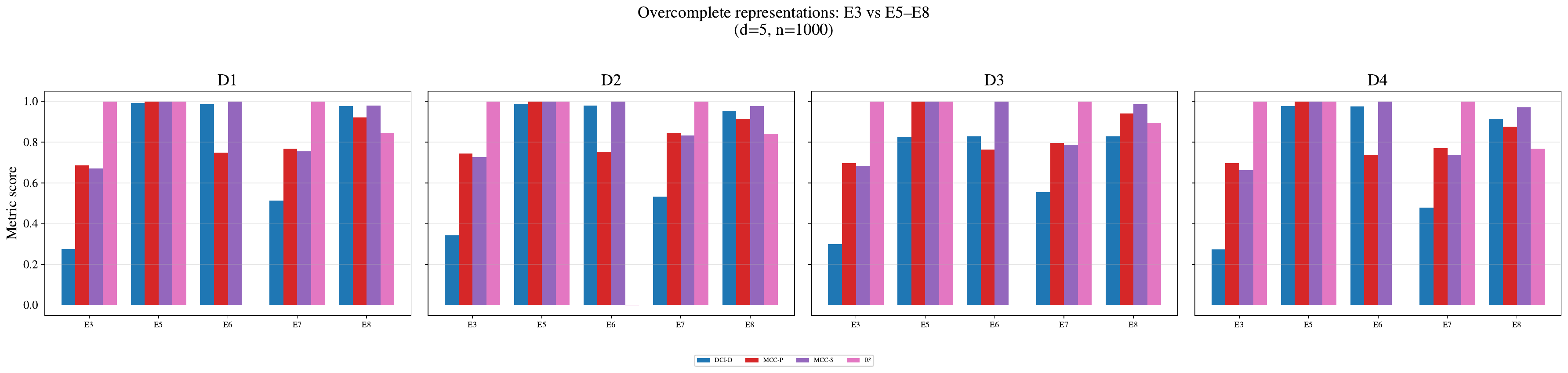}
    \caption{\textbf{Overcomplete encoders across all DGP types ($d{=}5$, $n{=}1000$).} Extension of \cref{fig:apx-exp09-bars} from $d{=}20$ to $d{=}5$. At this smaller~$d$, \dci{}-D and \mcc{} still separate disentangled overcomplete encoders (\ebeyond{5}{mo}--\ebeyond{8}{mo}) from the entangled baseline \ebeyond{7}{om}, though the gap is narrower than at $d{=}20$.}
    \label{fig:apx-exp09-d5}
\end{figure}


\begin{figure}
    \centering
    \includegraphics[width=\linewidth]{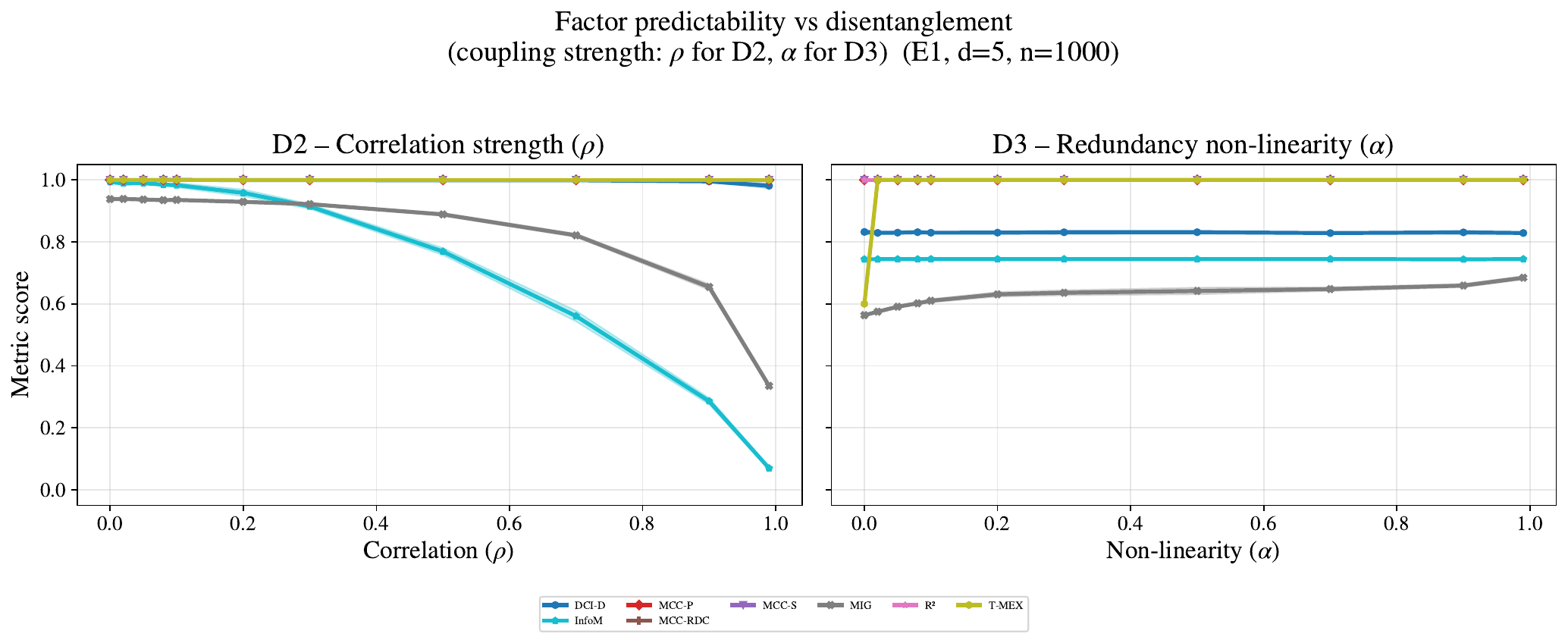}
    \caption{\textbf{Metrics conflate factor predictability with disentanglement under functional dependencies.} Full metric suite under \ecrl{1}{oo} comparing \dcrl{2} (varying~$\rho$) with \dbeyond{3} (deterministic constraint $z_2 = f(z_1)$). Under \dbeyond{3}, regression-based metrics (\rsq{}, \dci{}-D) penalise the encoder when the dependent factor is harder to predict from the retained code, despite perfect elementwise recovery. $d{=}5$, $n{=}1000$. See \cref{fig:apx-exp05-d10} for $d{=}10$.}
    \label{fig:apx-exp05}
\end{figure}
\begin{figure}
    \centering
    \includegraphics[width=\linewidth]{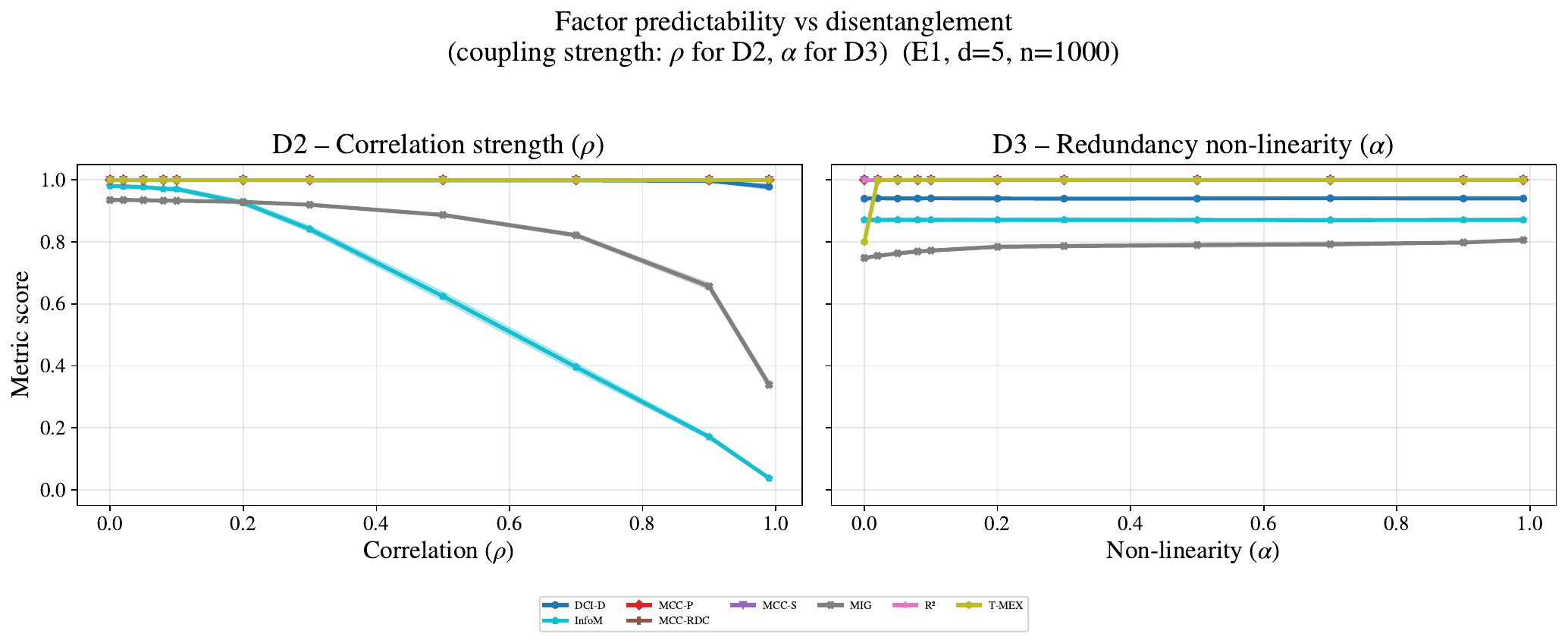}
    \caption{\textbf{Predictability vs.\ disentanglement at $d{=}10$.} Same setup as \cref{fig:apx-exp05} with $d{=}10$ factors. The conflation between factor predictability and measured disentanglement persists at higher dimensionality.}
    \label{fig:apx-exp05-d10}
\end{figure}


\begin{figure}
    \centering
    \includegraphics[width=\linewidth]{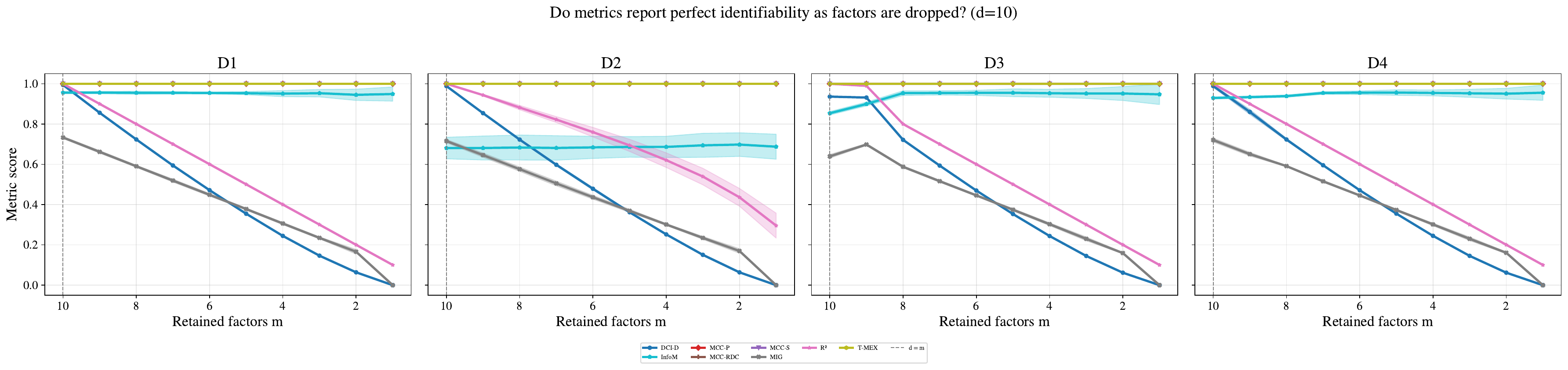}
    \caption{\textbf{Full metric suite for the dropped-factor experiment across all DGP types.} Extension of \cref{fig:apx-exp06-all} with additional metrics (MIG, InfoMEC, \mcc{}-RDC, T-MEX). MI-based metrics decline with fewer retained factors even under \dbeyond{3} at $m = d_{\mathrm{eff}}$, failing to recognise lossless compression of the redundant factor.}
    \label{fig:apx-exp06-drapped-all}
\end{figure}

\begin{figure}
    \centering
    \includegraphics[width=0.5\linewidth]{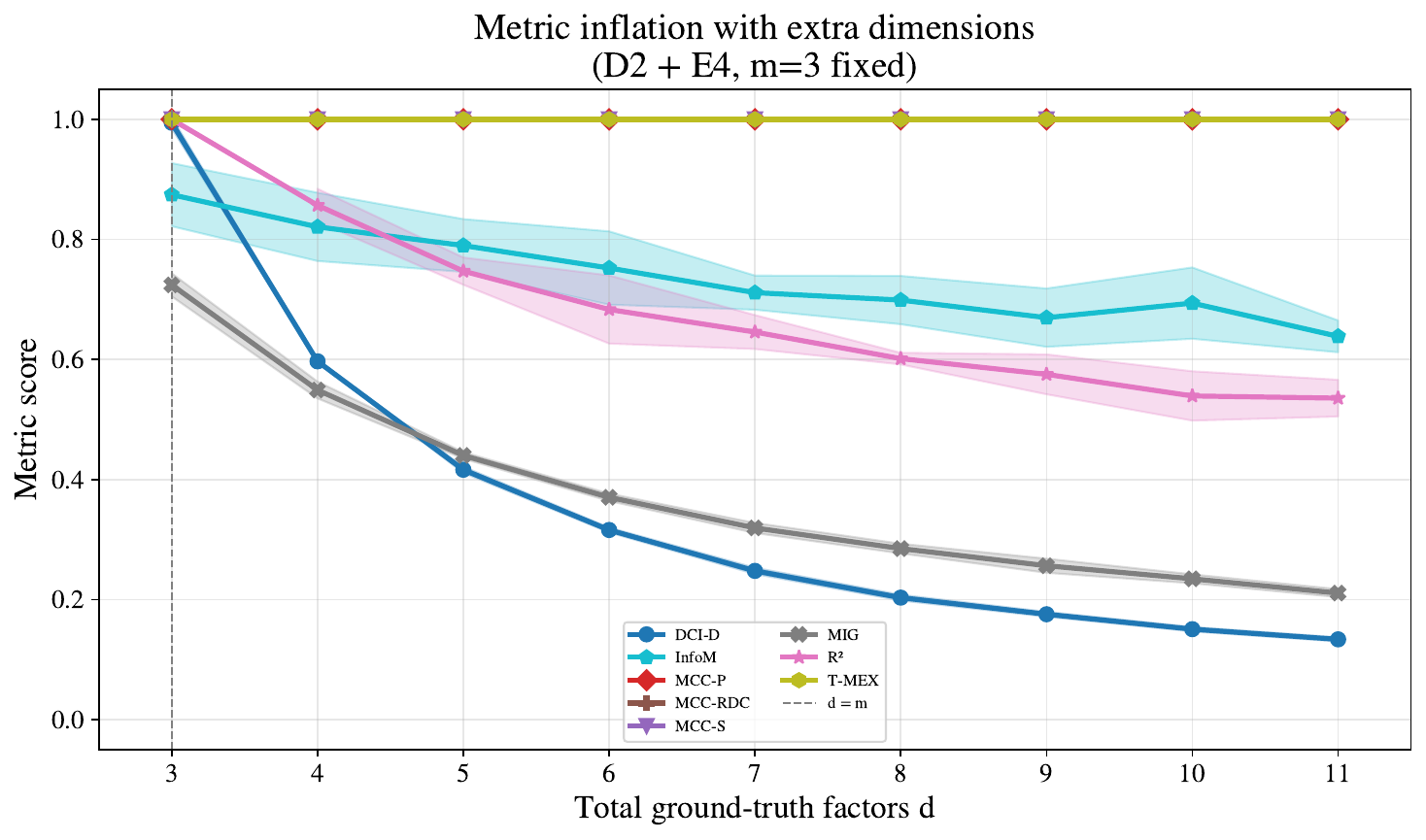}
    \caption{\textbf{Effect of inflating the number of ground-truth factors under \dcrl{2}.} The representation dimension is fixed at $m{=}3$ while $d$ increases by adding duplicated ground-truth factors. \mcc{}-P/S remain constant as they match only $m$ codes; \rsq{} and \dci{}-D decline because the probe must predict an increasing number of factors from the same $m$ codes.}
    \label{fig:apx-exp06-inflate-D2}
\end{figure}

\begin{figure}
    \centering
    \includegraphics[width=\linewidth]{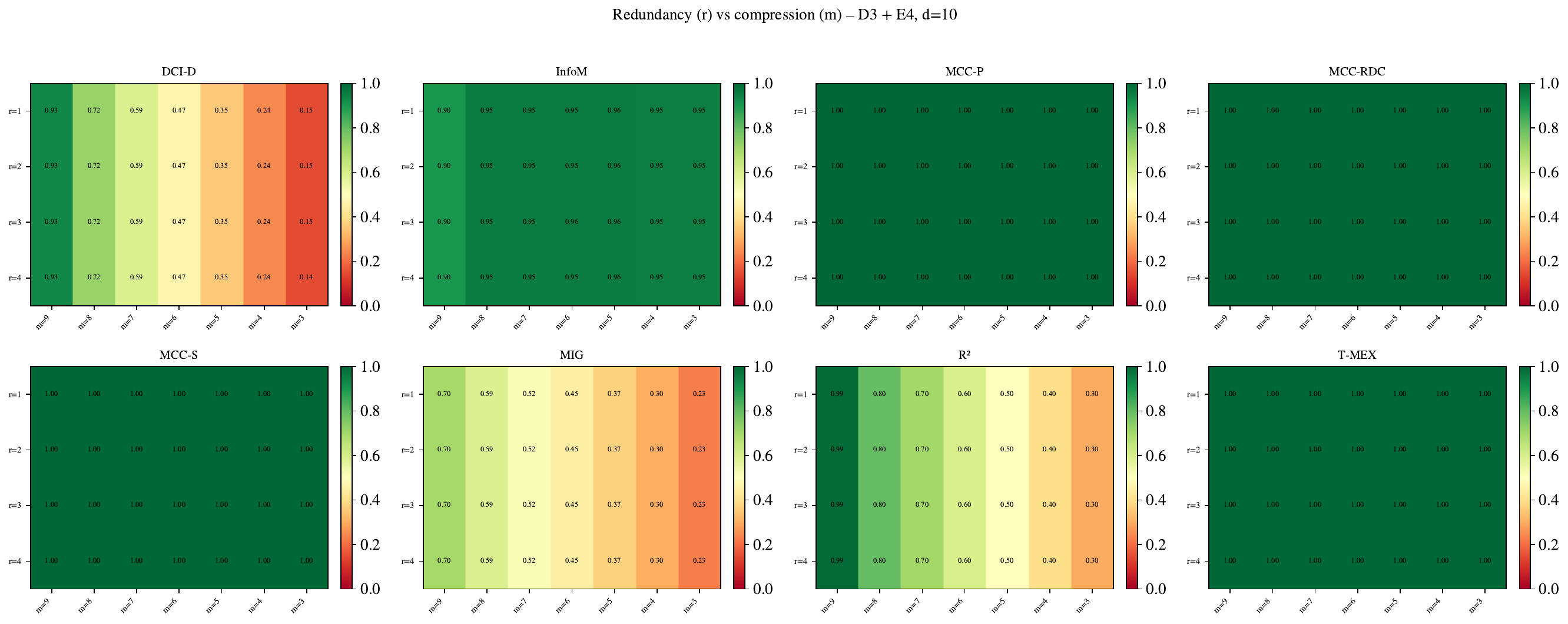}
    \caption{\textbf{Redundancy and compression under \dbeyond{3}.} The encoder compresses $d$ factors (with a single-factor constraint $z_2 = f(z_1)$, $d_{\mathrm{eff}} = d{-}1$) into $m \leq d$ codes. \rsq{} and \dci{}-D plateau near $1.0$ at $m = d_{\mathrm{eff}}$, correctly recognising that the omitted factor carries no independent information. \mcc{}-P/S report $1.0$ at all compression levels, unable to distinguish lossless from lossy omission.}
    \label{fig:apx-exp07-D3}
\end{figure}

\begin{figure}
    \centering
    \includegraphics[width=\linewidth]{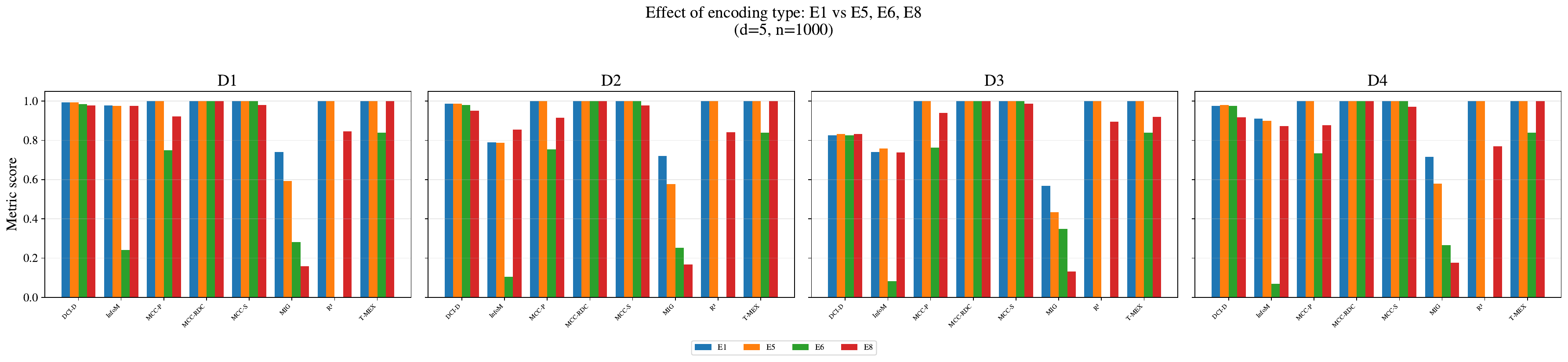}
    \caption{\textbf{Full metric suite: scores across encoding types and DGP types.} Each panel compares matched-dimension encoders (\ecrl{1}{oo}--\ecrl{3}{om}) with overcomplete encoders (\ebeyond{5}{mo}--\ebeyond{8}{mo}) under \dcrl{1}--\dbeyond{4}. MI-based metrics (MIG, InfoMEC) and T-MEX are included alongside the main metrics. $d{=}5$, $n{=}1000$.}
    \label{fig:apx-exp08}
\end{figure}

\begin{figure}
    \centering
    \includegraphics[width=\linewidth]{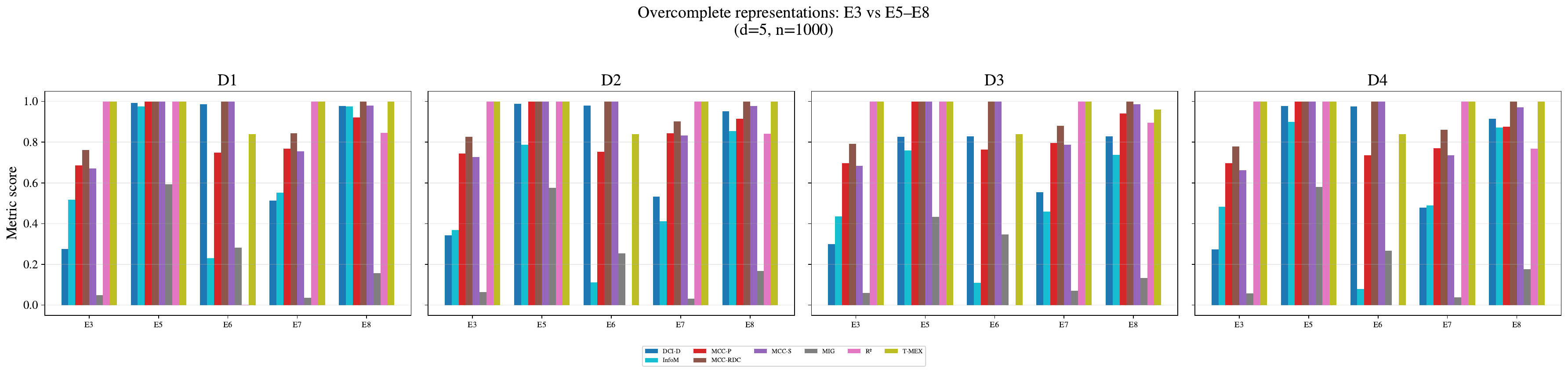}
    \caption{\textbf{Full metric suite for the overcomplete $m/d$ sweep.} Extension of \cref{fig:overcomplete_sweep} to all metrics (MIG, InfoMEC, \mcc{}-RDC, T-MEX). MI-based metrics decline for distributed codes (\ebeyond{8}{mo}) similarly to \mcc{}, while T-MEX is more robust to overcompleteness. $d{=}5$, $n{=}1000$.}
    \label{fig:apx-exp09-overcomplete-all}
\end{figure}
\begin{figure}
    \centering
    \includegraphics[width=\linewidth]{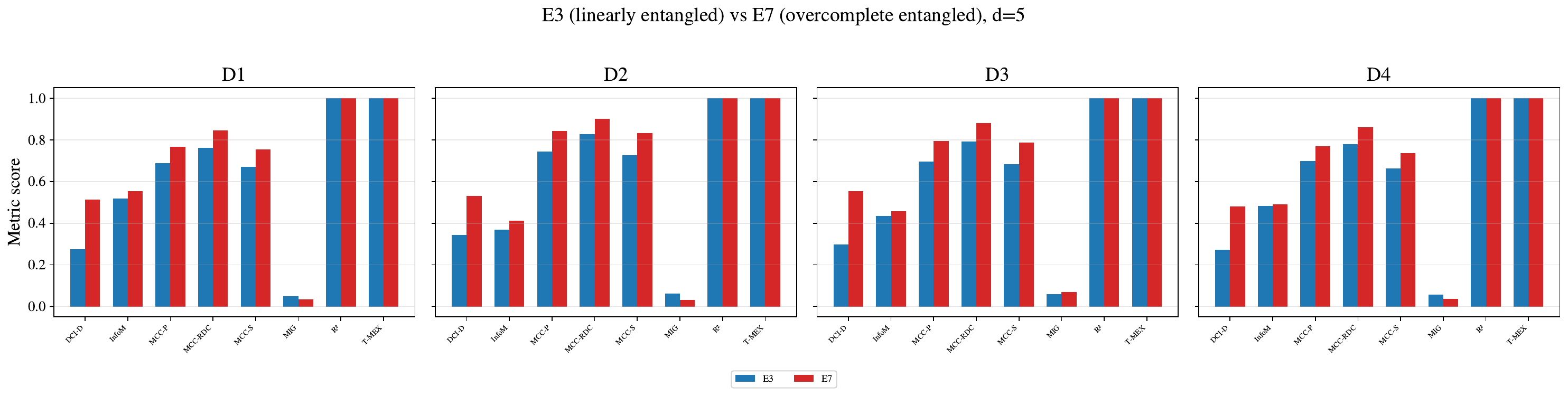}
    \caption{\textbf{Matched-dimension entangled (\ecrl{3}{om}) versus overcomplete entangled (\ebeyond{7}{om}).} Full metric suite comparing the two entangled geometries as $m/d$ increases. \dci{}-D inflates for \ebeyond{7}{om} at high $m/d$, scoring substantially above the matched-dimension baseline \ecrl{3}{om} despite equivalent identifiability status. $d{=}5$, $n{=}1000$.}
    \label{fig:apx-exp09-e3e7-all}
\end{figure}

\begin{figure}
    \centering
    \includegraphics[width=\linewidth]{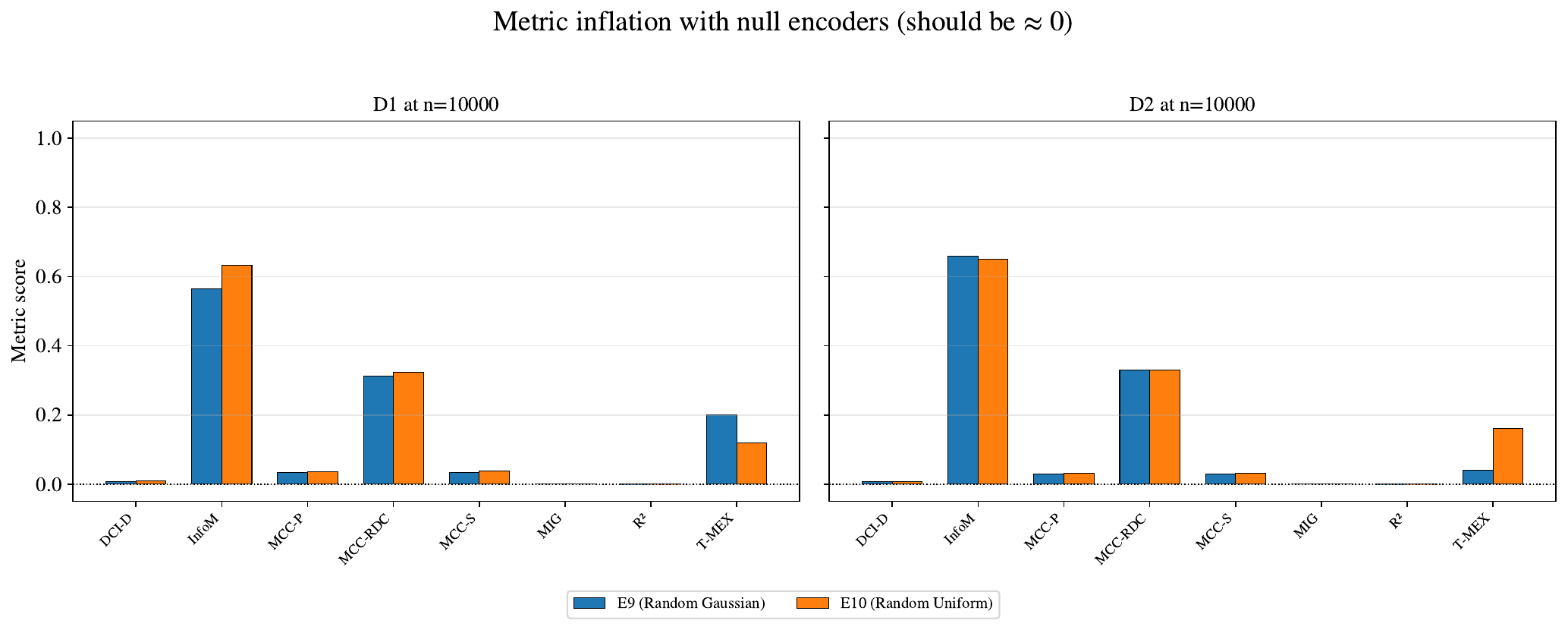}
    \caption{\textbf{Metric inflation under null encoders.} Full metric suite showing scores when the representation is independent of $\rvz$. \mcc{}-RDC exhibits persistent inflation that does not vanish with increasing~$n$. MI-based metrics (MIG, InfoMEC) also return non-trivial scores. \rsq{} remains closest to the expected value of~$0$. $d{=}5$, $n{=}1000$.}
    \label{fig:apx-exp11-inflation-bar}
\end{figure}
\begin{figure}
    \centering
    \includegraphics[width=0.75\linewidth]{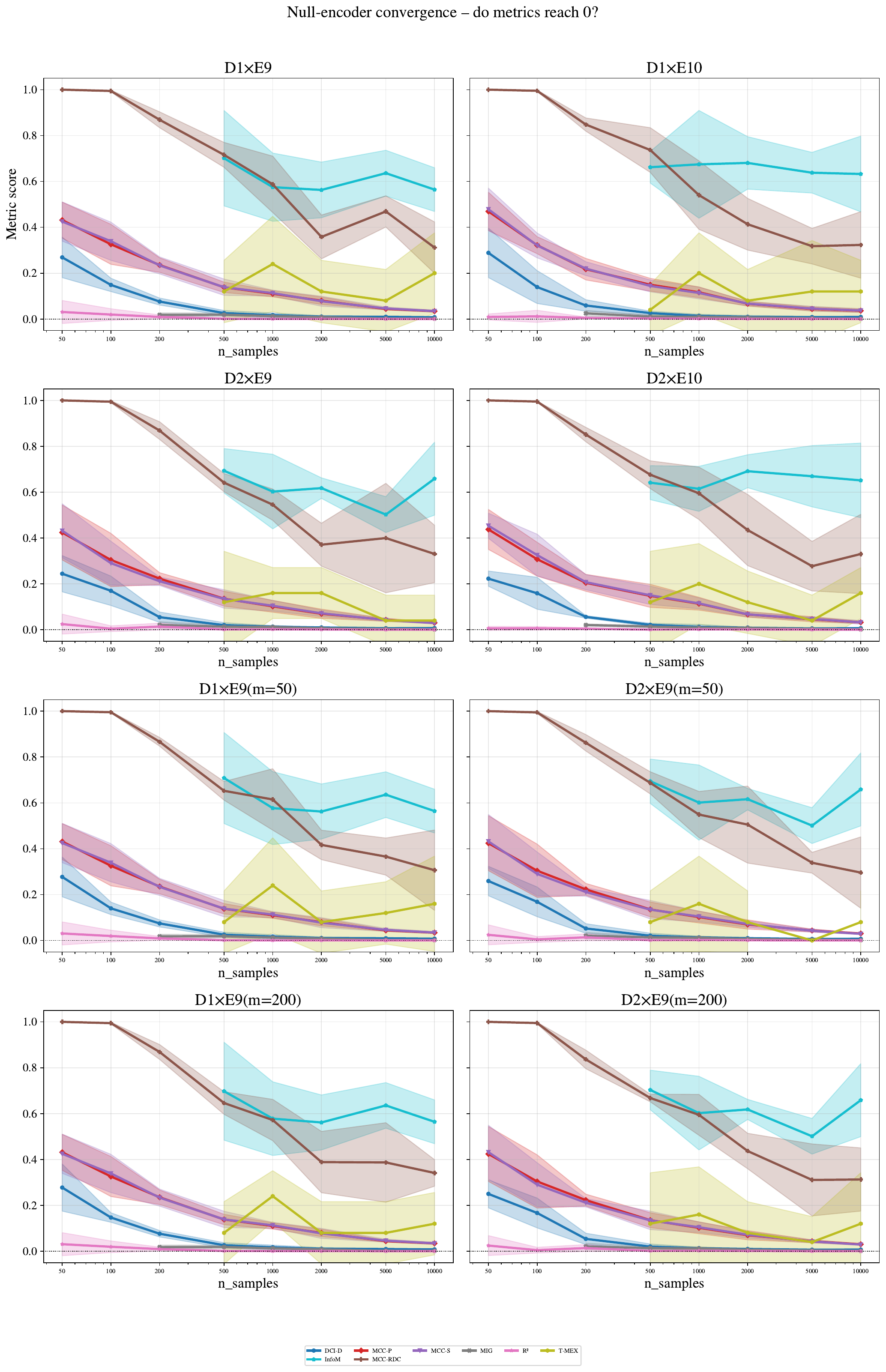}
    \caption{\textbf{Convergence of null-encoder scores with increasing $n$.} Each panel shows one metric under \rand{9}; scores should converge to~$0$ as $n$ grows. \rsq{} converges fastest. \mcc{}-P/S retain elevated scores at large~$n$ when $m$ is large, consistent with the $\sqrt{2\log m / n}$ floor (\cref{apx:mcc-false-positive}). $d{=}5$.}
    \label{fig:apx-exp11-convergence-grid}
\end{figure}

\begin{figure}
    \centering
    \includegraphics[width=\linewidth]{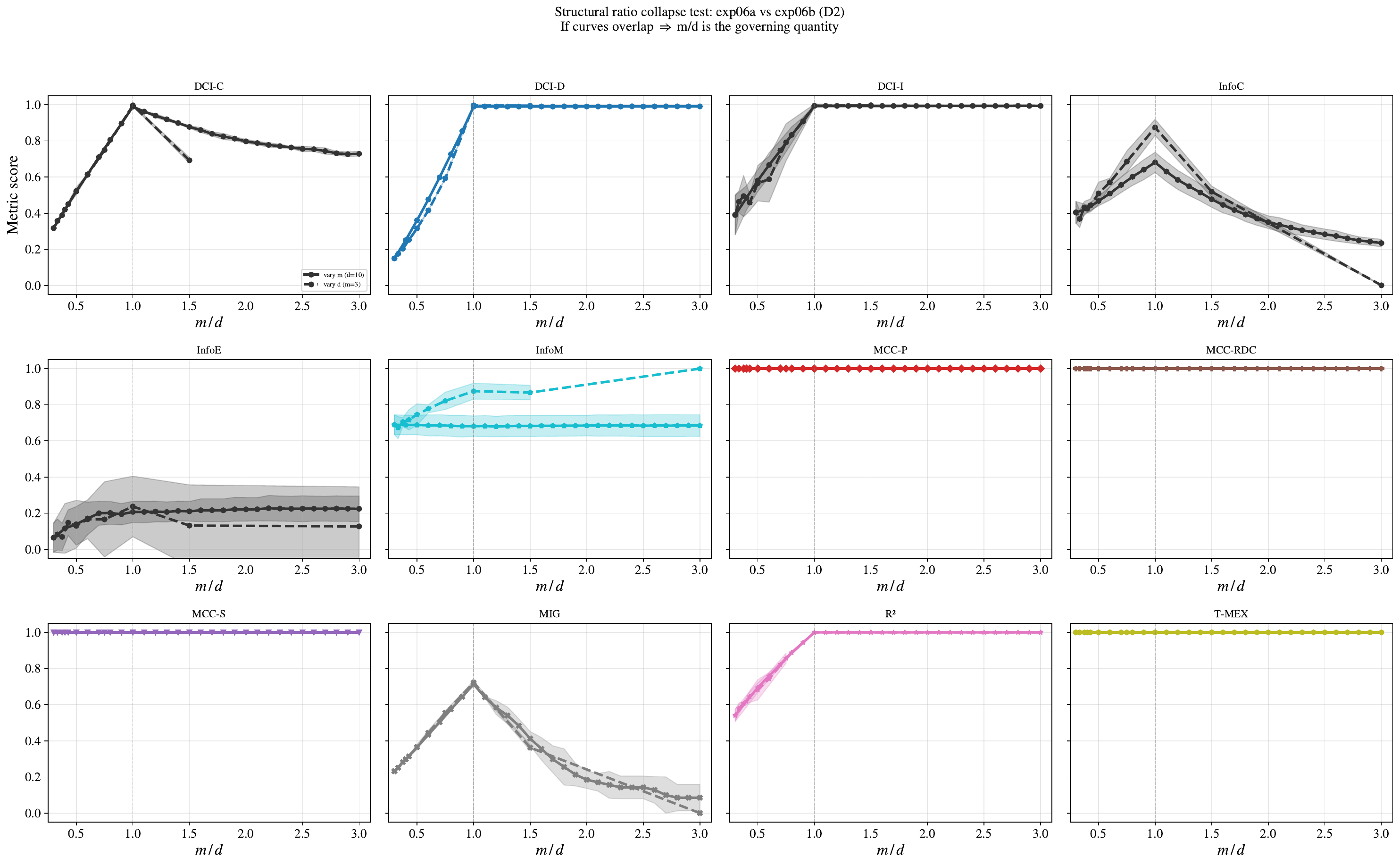}
    \caption{\textbf{Metric scores across encoder types under correlated factors (\dcrl{2}).} Overlay of all metrics for varying $\rho$ under \dcrl{2}. \mcc{}-P/S increase with $|\rho|$ under entangled encoders (\ecrl{3}{om}), while \rsq{} and \dci{}-D are less affected by the correlation structure. $d{=}5$, $n{=}1000$.}
    \label{fig:apx-exp12-overlay}
\end{figure}

\begin{figure}
    \centering
    \includegraphics[width=\linewidth]{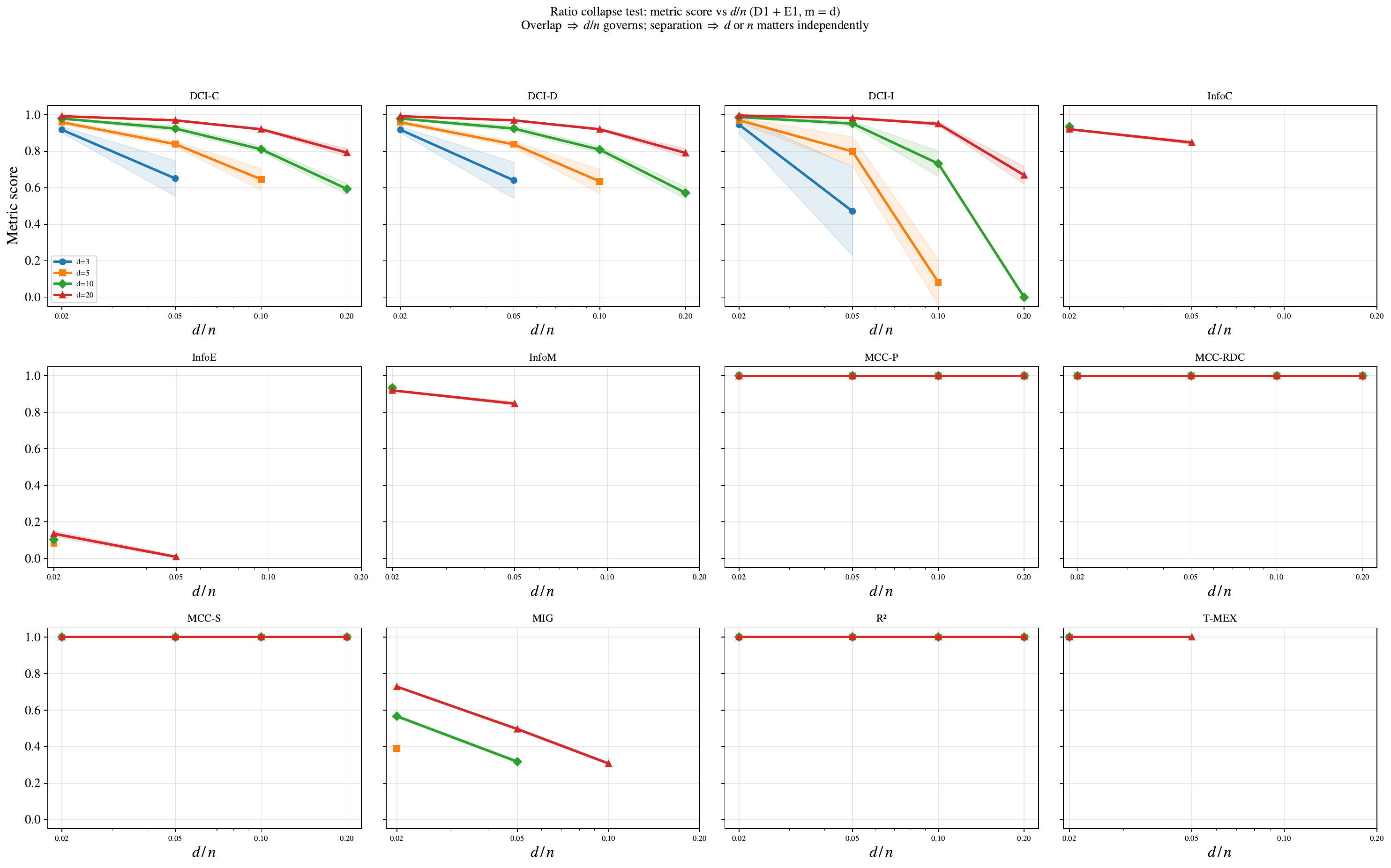}
    \caption{\textbf{Metric scores as a function of $m/d$ under different $(m, d)$ configurations.} Each panel shows one metric; overlapping curves from different $(m, d)$ pairs with the same ratio confirm that $m/d$, not $m$ or $d$ individually, governs metric behaviour in the undercomplete regime. \mcc{}-P/S report~$1.0$ regardless of $m/d$; \rsq{} and \dci{}-D increase approximately linearly. $n{=}1000$.}
    \label{fig:apx-exp13-ratio-collapse-grid}
\end{figure}
\begin{figure}
    \centering
    \includegraphics[width=\linewidth]{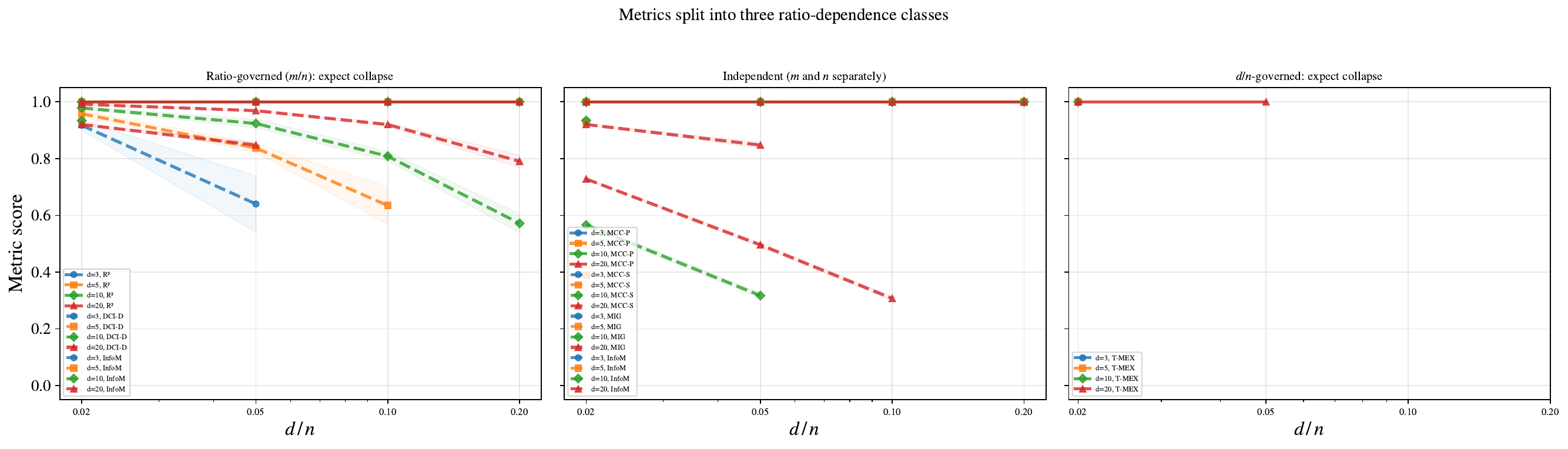}
    \caption{\textbf{Ratio-collapse analysis grouped by encoder type.} Same data as \cref{fig:apx-exp13-ratio-collapse-grid}, reorganised by encoder geometry. Curves from sweeps over $m$ (fixed $d$) and over $d$ (fixed $m$) overlap at matched $m/d$, confirming the ratio as the governing quantity across encoder types.}
    \label{fig:apx-exp13-ratio-collapse-groups}
\end{figure}

\begin{figure}
    \centering
    \includegraphics[width=\linewidth]{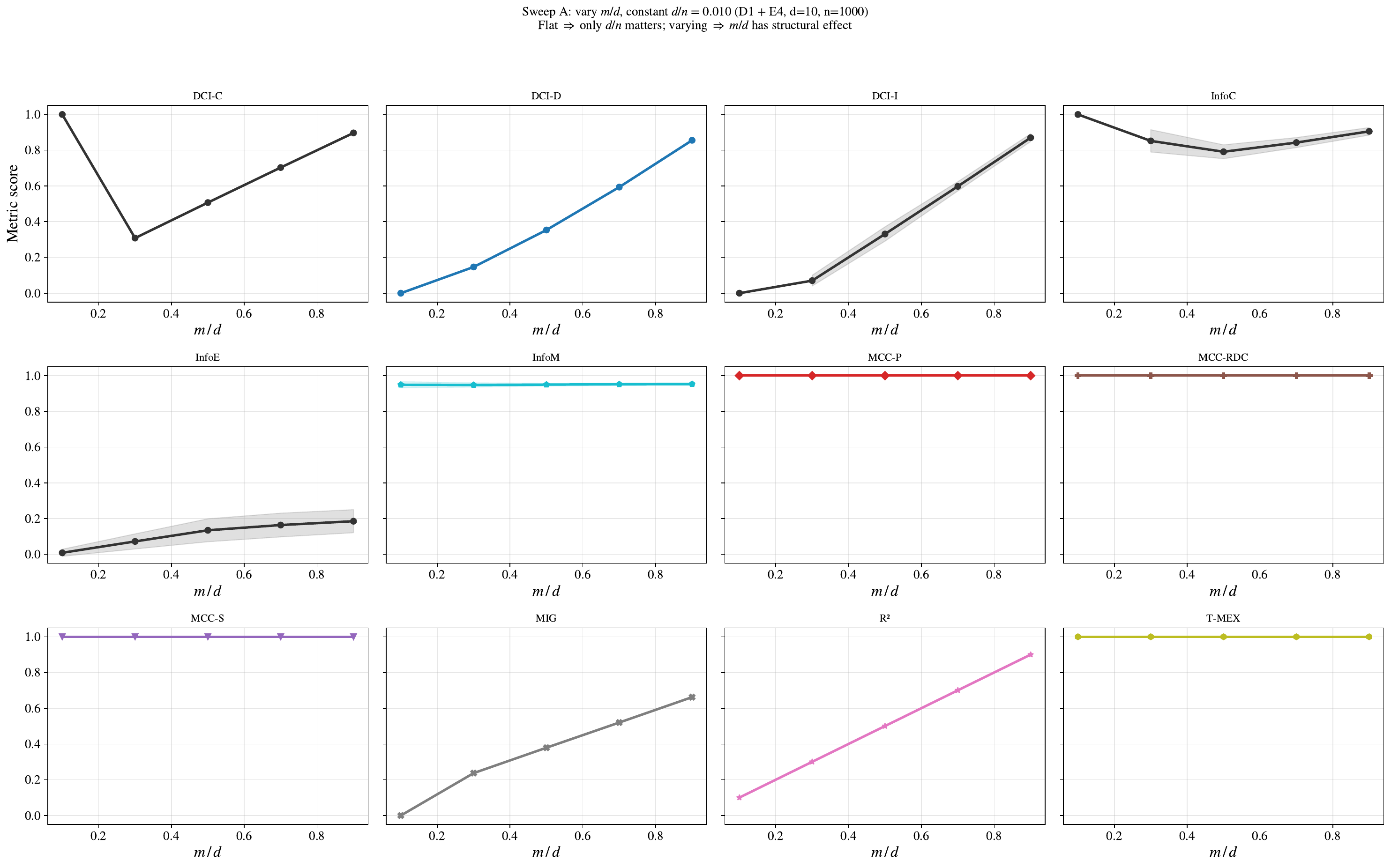}
    \caption{\textbf{Comprehensive parameter sweep (part~A).} Full metric suite across DGP types and encoder geometries, sweeping scaling and complexity parameters. Extends the targeted analyses of \cref{subsec:correlation,subsec:undercomplete,subsec:overcomplete,subsec:false-positive} to a broader parameter range. $d{=}5$, $n{=}1000$.}
    \label{fig:apx-exp14-sweep-a}
\end{figure}
\begin{figure}
    \centering
    \includegraphics[width=\linewidth]{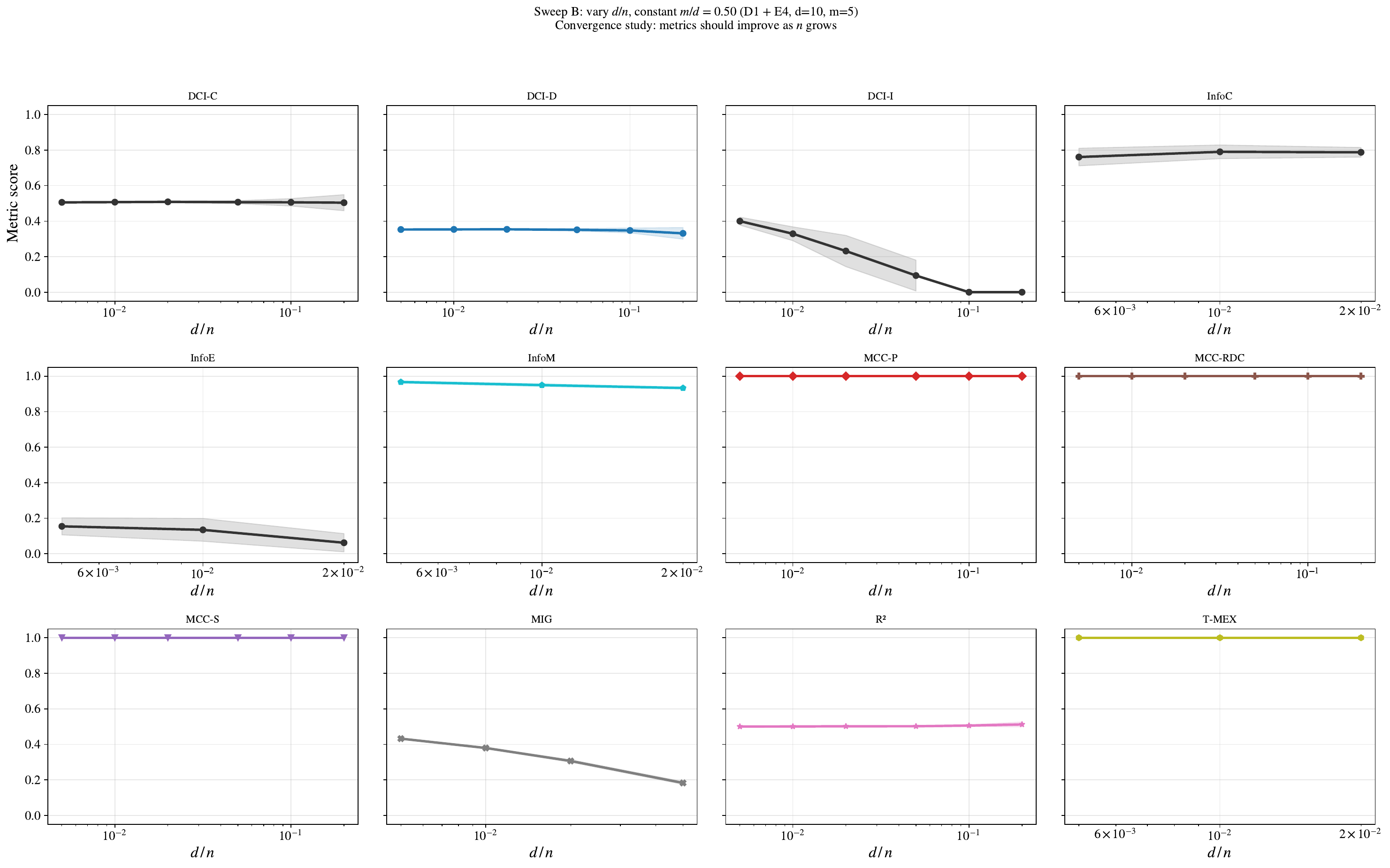}
    \caption{\textbf{Comprehensive parameter sweep (part~B).} Continuation of \cref{fig:apx-exp14-sweep-a} for additional parameter configurations and encoder--DGP combinations.}
    \label{fig:apx-exp14-sweep-b}
\end{figure}

\begin{figure}
    \centering
    \includegraphics[width=\linewidth]{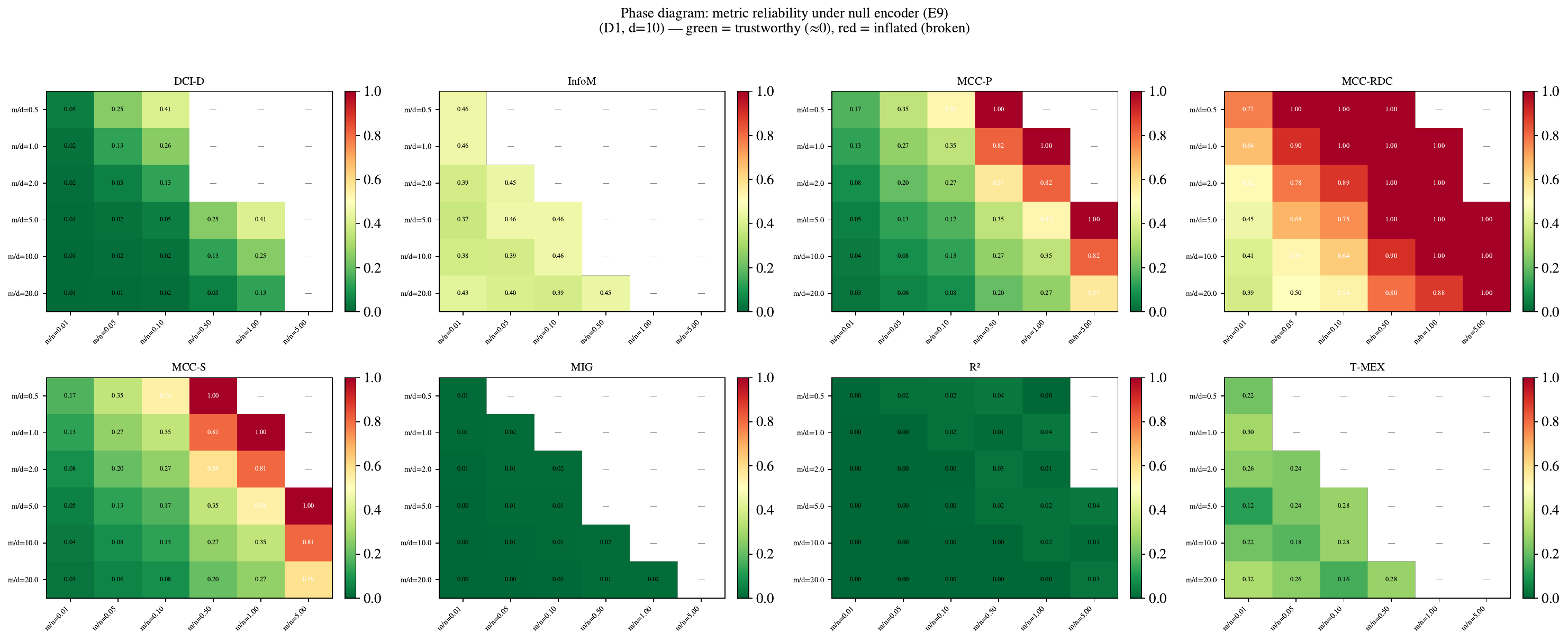}
    \caption{\textbf{Full metric suite: false-positive phase diagram under Gaussian null.} Extension of \cref{fig:apx-exp15-gauss} to all metrics. \mcc{}-RDC shows the highest inflation across the $(m/d,\, m/n)$ grid. MI-based metrics (MIG, InfoMEC) also inflate at moderate $m/n$.}
    \label{fig:apx-exp15-e9}
\end{figure}
\begin{figure}
    \centering
    \includegraphics[width=\linewidth]{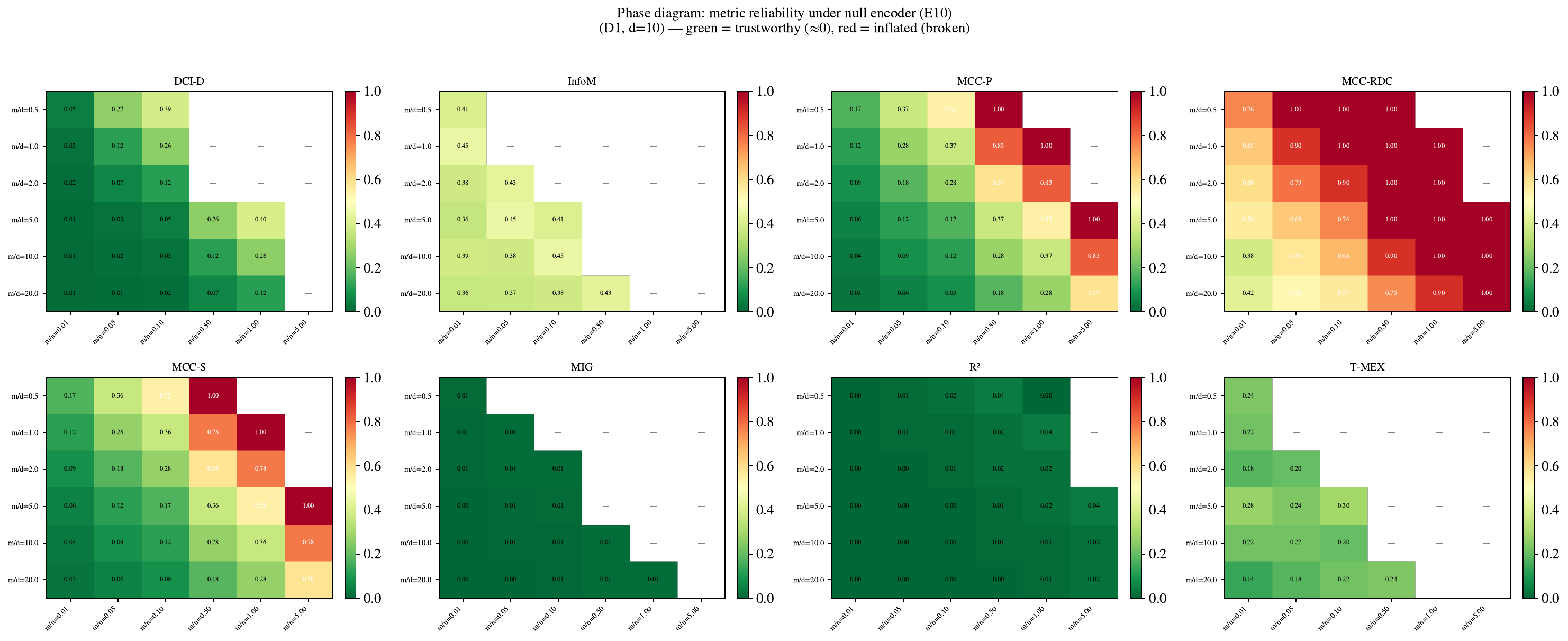}
    \caption{\textbf{Full metric suite: false-positive phase diagram under uniform null.} Extension of \cref{fig:exp15-null} to all metrics. The pattern closely mirrors the Gaussian null (\cref{fig:apx-exp15-e9}), confirming that the false-positive floor is distribution-agnostic and governed by $m/n$.}
    \label{fig:apx-exp15-e10}
\end{figure}

\end{document}